\documentclass[letterpaper]{article} % DO NOT CHANGE THIS
\usepackage{aaai24}  % DO NOT CHANGE THIS
\usepackage{times}  % DO NOT CHANGE THIS
\usepackage{helvet}  % DO NOT CHANGE THIS
\usepackage{courier}  % DO NOT CHANGE THIS
\usepackage[hyphens]{url}  % DO NOT CHANGE THIS
\usepackage{graphicx} % DO NOT CHANGE THIS
\usepackage{natbib}  % DO NOT CHANGE THIS AND DO NOT ADD ANY OPTIONS TO IT
\usepackage{caption} % DO NOT CHANGE THIS AND DO NOT ADD ANY OPTIONS TO IT
\usepackage{algorithm}
\usepackage{algorithmic}
\usepackage{amsmath}
\usepackage{amsthm}
\usepackage{booktabs}
\usepackage{multirow}
\usepackage{amssymb}

\usepackage{bm}
\usepackage[toc,page,header]{appendix}
\usepackage{minitoc}
\usepackage{xcolor}

\newtheorem{theorem}{Theorem}[]

\newtheorem{lemma}[]{Lemma}

\newtheorem{definition}[]{Definition}
\newtheorem{assumption}[]{Assumption}
\usepackage{newfloat}
\usepackage{listings}
\DeclareCaptionStyle{ruled}{labelfont=normalfont,labelsep=colon,strut=off} % DO NOT CHANGE THIS
\floatstyle{ruled}
\newfloat{listing}{tb}{lst}{}
\floatname{listing}{Listing}

\nocopyright
\title{FP3O: Enabling Proximal Policy Optimization in Multi-Agent Cooperation \\ with Parameter-Sharing Versatility}
\author{
    Lang Feng\textsuperscript{\rm 1},
    Dong Xing\textsuperscript{\rm 1,\rm 2},
    Junru Zhang\textsuperscript{\rm 1},
    Gang Pan\textsuperscript{\rm 1,\rm 2}\thanks{Corresponding author}
}
\affiliations{
    \textsuperscript{\rm 1}College of Computer Science and Technology, Zhejiang University, Hangzhou, China\\
    \textsuperscript{\rm 2}The State Key Lab of Brain-Machine Intelligence, Zhejiang University, Hangzhou, China\\
    \{langfeng, dongxing, junruzhang, gpan\}@zju.edu.cn
}

\usepackage{bibentry}
\begin{document}
\doparttoc % Tell to minitoc to generate a toc for the parts
\faketableofcontents % Run a fake tableofcontents command for the partocs

\maketitle

\begin{abstract}
  Existing multi-agent PPO algorithms lack compatibility with different types of parameter sharing when extending the theoretical guarantee of PPO to cooperative multi-agent reinforcement learning (MARL). In this paper, we propose a novel and versatile multi-agent PPO algorithm for cooperative MARL to overcome this limitation. Our approach is achieved upon the proposed full-pipeline paradigm, which establishes multiple parallel optimization pipelines by employing various equivalent decompositions of the advantage function. This procedure successfully formulates the interconnections among agents in a more general manner, i.e., the interconnections among pipelines, making it compatible with diverse types of parameter sharing. We provide a solid theoretical foundation for policy improvement and subsequently develop a practical algorithm called Full-Pipeline PPO (FP3O) by several approximations. Empirical evaluations on Multi-Agent MuJoCo and StarCraftII tasks demonstrate that FP3O outperforms other strong baselines and exhibits remarkable versatility across various parameter-sharing configurations.
\end{abstract}

\section{Introduction}
Proximal Policy Optimization (PPO) \citep{schulman2017ppo} has been deemed as one of the most effective algorithms in single-agent deep reinforcement learning (RL) with empirical successes in a wide range of challenging domains, such as game playing \citep{bellemare2013arcade}, robotic control \citep{duan2016benchmarking} and human-level task \citep{berner2019dota}.
Originally derived from Trust Region Policy Optimization (TRPO) \citep{schulman2015trpo}, PPO enjoys the theoretical monotonic improvement guarantee by restricting the Kullback-Leibler (KL) divergence between the updated policy and the old one to a trust region. The attractive theoretical support makes PPO's extension in cooperative multi-agent reinforcement learning (MARL) rapidly developed \citep{wu2021coordinated,kuba2022happo,wang2023order}.

Parameter sharing has emerged as a significant topic in MARL due to its substantial influence on training efficiency \cite{liu2020data}, knowledge transfer \cite{agarwal2020learning}, and overall algorithm performance \cite{fu2022revisiting}. Nevertheless, one practical yet essential issue that has received less attention in the scaling of PPO's theoretical guarantee to MARL is the assurance of algorithmic versatility across different parameter-sharing types (e.g., full, partial, and non-parameter sharing). Full parameter sharing, where all parameters are shared among agents, has been widely adopted in scenarios with homogeneous agents or a large number of agents. It has proven beneficial in enhancing learning efficiency and minimizing the number of parameters \cite{rashid2018qmix,wang2020qplex}. Conversely, non-parameter sharing allows each agent to possess an individual policy network, facilitating diverse decision-making \cite{lowe2017maddpg,kuba2022happo}. This is especially crucial in heterogeneous agent scenarios, where agents operate within distinct action spaces. Striking a balance between diversity and sharing, partial parameter sharing has been introduced \cite{chenghao2021celebrating,christianos2021scaling}. It is often employed in large-scale heterogeneous agent scenarios to prevent the explosion of the parameter space while enabling diverse decision-making. Therefore, ensuring the versatility of an algorithm across different parameter-sharing types is critical for its applicability to various multi-agent tasks.

However, existing PPO-based algorithms in cooperative MARL exhibit a deficiency in compatibility with various types of parameter sharing. Algorithms such as Independent PPO (IPPO) \citep{de2020ippo} and Multi-agent PPO (MAPPO) \citep{yu2021mappo} apply PPO straightforwardly to every agent's optimization with full parameter sharing. Despite their empirical successes, the absence of the theoretical guarantee could disturb their performances on other tasks or network types \citep{kuba2022happo}. Recent advancements like Coordinate PPO (CoPPO) \citep{wu2021coordinated}, Heterogeneous-agent PPO (HAPPO) \citep{kuba2022happo} and Agent-by-agent PPO (A2PO) \cite{wang2023order} succeed to scale the theoretical property of PPO to cooperative MARL, but they are predicated on constrained training procedures, which in turn limit their versatility. Specifically, the optimization objective of each agent in CoPPO incorporates the policy network parameters of other agents. However, when updating agents in a partial/non-parameter sharing manner, the process encounters impractical repetitive error backpropagation of the policy networks. Meanwhile, the sequential updates of HAPPO and A2PO are established upon the assumption that each agent's update does not impact the policies of others. This assumption, however, becomes untenable when dealing with full/partial parameter sharing cases, as each agent's update can induce changes in the policy network parameters of other agents, thereby breaking the monotonicity guarantee. 

We propose a general-purpose multi-agent PPO algorithm for cooperative MARL, designed to maintain versatility across different parameter-sharing types. We begin by presenting our single pipeline optimization as the basic pipeline, which offers critical insight into the various available ways of advantage function decomposition. By executing various decompositions, we establish multiple parallel optimization pipelines that are essentially equivalent, resulting in a novel \emph{full-pipeline paradigm}. This paradigm innovatively formulates all potential interconnections among agents, concerning different parameter-sharing types, as interconnections among pipelines, making it a general update scheme. We then prove in theory that full-pipeline optimization enables monotonic joint policy improvement. To implement this procedure, we propose the practical \emph{\textbf{F}ull-\textbf{P}ipeline \textbf{PPO} (FP3O)} algorithm by several approximations. We evaluate our algorithm on Multi-Agent MuJoCo \citep{de2020mamujoco} and StarCraftII Multi-Agent Challenge \citep{samvelyan2019starcraft} benchmarks against several strong baselines. The results demonstrate the superior performances of FP3O and its versatility across various parameter-sharing types.

\section{Preliminaries}
\paragraph{Problem Formulation.}
% We model the cooperative multi-agent task with a multi-agent version of discounted Markov decision process (MDP) \citep{littman1994markov}, which can be described as a tuple
We model our task as a dcentralized Markov decision process~\citep{bernstein2002complexity}, which can be described as a tuple $<\mathcal{S}, \mathcal{N}, \bm{\mathcal{A}}, P, r, \gamma, \rho_{0}>$. Here, $s \in \mathcal{S}$ is the state of the environment. $i \in \mathcal{N} \equiv \left\{ 1,...,n \right\} $ denotes the index of agent. $\bm{\mathcal{A}}$ denotes the joint action space, which is defined by $\prod_{i\in \mathcal{N}}{\mathcal{A}^i}$, the product of all agents' action spaces. $P:\mathcal{S}\times \bm{\mathcal{A}} \times \mathcal{S} \rightarrow [0, 1]$ is the transition probability function and $r: \mathcal{S}\times \bm{\mathcal{A}}\rightarrow \mathbb{R}$ is the shared reward. $\gamma \in [0, 1)$ is the discount factor. $\rho_{0}: \mathcal{S} \rightarrow \mathbb{R}$ is the distribution of the initial state of environment.
At each timestep $t \in \mathbb{N}$, agent $i$ will choose an action $a^i_t \in \mathcal{A}^i$ according to its policy $\pi^i(\cdot|s_t)$. The joint action $\bm{a}_t=\left\{a^1_t,...,a^n_t\right\}$ and the joint policy $\bm{\pi}(\bm{a}_t|s_t)$ can be formed by combining all agents' actions and policies. Then, the state of the environment is transferred to a new state according to the transition probability function $P(s_{t+1}|s_t,\bm{a}_t)$, and all agents will receive a shared reward $r(s_t, \bm{a}_t)$. We let $\rho_{\pi}(s)=\sum\nolimits_{t=0}^{\infty}\gamma^{t}P(s_t=s)$ denote the (unnormalized) state visitation frequencies. The optimization objective for all agents in the fully cooperative multi-agent task is maximizing the discounted cumulative reward: $\mathcal{J}(\bm{\pi})\triangleq\mathbb{E}_{s_{0:\infty}\sim\rho_{\pi},\bm{a}_{0:\infty}\sim\bm{\pi}}\left[\sum_{t=0}^{\infty}{\gamma^{t}r(s_t, \bm{a}_t)}\right]$. 
Then, we denote the joint state value function, the joint state-action function, and the joint advantage function as $V_{\bm{\pi}}(s)\triangleq \mathbb{E}_{s_{1:\infty}\sim\rho_{\pi},\bm{a}_{0:\infty}\sim\bm{\pi}} \left[\sum_{t=0}^{\infty}{\gamma^{t}r(s_t, \bm{a}_t)} \big| s_0=s\right]$, $Q_{\bm{\pi}}(s, \bm{a})\triangleq \mathbb{E}_{s_{1:\infty}\sim\rho_{\pi},\bm{a}_{1:\infty}\sim\bm{\pi}} \left[\sum_{t=0}^{\infty}{\gamma^{t}r(s_t, \bm{a}_t)} \big| s_0=s, \bm{a}_0=\bm{a}\right]$, and $A_{\bm{\pi}}(s,\bm{a})\triangleq Q_{\bm{\pi}}(s, \bm{a})-V_{\bm{\pi}}(s)$ respectively.
% Then, we denote the joint state value function as $V_{\bm{\pi}}(s)\triangleq \mathbb{E}_{s_{1:\infty}\sim\rho_{\pi},\bm{a}_{0:\infty}\sim\bm{\pi}} \left[\sum_{t=0}^{\infty}{\gamma^{t}r(s_t, \bm{a}_t)} \big| s_0=s\right]$, the joint state-action function as $Q_{\bm{\pi}}(s, \bm{a})\triangleq \mathbb{E}_{s_{1:\infty}\sim\rho_{\pi},\bm{a}_{1:\infty}\sim\bm{\pi}} \left[\sum_{t=0}^{\infty}{\gamma^{t}r(s_t, \bm{a}_t)} \big| s_0=s, \bm{a}_0=\bm{a}\right]$, and the joint advantage function as $A_{\bm{\pi}}(s,\bm{a})\triangleq Q_{\bm{\pi}}(s, \bm{a})-V_{\bm{\pi}}(s)$.

\paragraph{Multi-Agent Credit Assignment.}
Quantifying agents' contribution from the whole team's joint reward is commonly known as the multi-agent credit assignment \citep{chang2003all,sunehag2017value,foerster2018counterfactual,son2019qtran}.
% Currently, there are many efforts to address this issue from both value-based \citep{sunehag2017value,son2019qtran} and policy-based perspectives \citep{foerster2018counterfactual}. 
In this paper, we follow \citet{kuba2022happo} to deduce the contribution of different subsets of agents. We let $i_{1:m}$ be a subset $\left\{i_1,...,i_m\right\}$ of $\mathcal{N}$, and let $-i_{1:m}$ be the complement set of $i_{1:m}$. In particular, $i_{1:n}$ represents the complete set $\mathcal{N}$ of all agents. Then, the multi-agent state-action value function can be defined as
\begin{equation}
  \small
  \label{equation_state_action_value_function}
  Q_{\bm{\pi}}^{i_{1:m}}(s,\bm{a}^{i_{1:m}})\triangleq \mathbb{E}_{\bm{a}^{-i_{1:m}}\sim\bm{\pi}^{-i_{1:m}}}\left[Q_{\bm{\pi}}(s,\bm{a})\right].
\end{equation}
It quantifies the contribution of agents $i_{1:m}$ to the joint reward by calculating the average return if agents $i_{1:m}$ execute joint action $\bm{a}^{i_{1:m}}$. Then, for arbitrary disjoint sets $j_{1:k}$ and $i_{1:m}$, the multi-agent advantage function is defined as
{\small
\begin{align}
  \label{equation_advantage_function}
  &A_{\bm{\pi}}^{i_{1:m}}(s, \bm{a}^{j_{1:k}},\bm{a}^{i_{1:m}}) \nonumber \\
  &\qquad\qquad \triangleq Q_{\bm{\pi}}^{j_{1:k},i_{1:m}}(s,\bm{a}^{j_{1:k},i_{1:m}})-Q_{\bm{\pi}}^{j_{1:k}}(s,\bm{a}^{j_{1:k}}), 
\end{align}
}
which quantifies the additional contribution of agents $i_{1:m}$ executing joint action $\bm{a}^{i_{1:m}}$ if agents $j_{1:k}$ have executed joint action $\bm{a}^{j_{1:k}}$. Finally, given the complete set $i_{1:n}$, the joint advantage function $A_{\bm{\pi}}(s,\bm{a})$ can be decomposed as follow:
\begin{equation}
\small
  \label{equation_multi_agent_advantage_decomposition}
  A_{\bm{\pi}}(s,\bm{a}) = \sum\nolimits_{m=1}^n{A_{\bm{\pi}}^{i_m}(s,\bm{a}^{i_{1:m-1}},a^{i_m})}.  
\end{equation}
% It is a general setup without extra assumptions for the credit assignment problem in cooperative MARL.

\paragraph{Policy Improvement Lower Bound.}
In single-agent RL, trust region learning such as TRPO and PPO enjoys the monotonicity guarantee of discounted cumulative reward $\mathcal{J}(\pi)$. The key policy improvement lower bound \cite{schulman2015trpo} is given by
% \begin{theorem}[\citeauthor{schulman2015trpo}, \citeyear{schulman2015trpo}]
\begin{theorem}
  \label{theorem_bound}
  Let $D_{\rm KL}^{\max}(\pi, \tilde{\pi})=\max_s{D_{\rm KL}(\pi(\cdot |s), \tilde{\pi}(\cdot |s))}$, $C=\frac{4\gamma \max_{s,a}|A_{\pi}(s,a)|}{(1-\gamma)^2}$, and $\pi$ denote the old policy. For any policy $\tilde{\pi}$, the following policy improvement bound holds:
  \begin{equation}
  \small
     \mathcal{J}(\tilde{\pi})\geq \mathcal{J}(\pi)+\mathbb{E}_{s\sim\rho_{\pi}, a\sim\tilde{\pi}}[A_{\pi}(s,a)]-CD_{\rm KL}^{\max}(\pi,\tilde{\pi}). \nonumber
  \end{equation}
  % where $D_{\rm KL}^{\max}(\pi, \tilde{\pi})=\max_s{D_{\rm KL}(\pi(\cdot |s), \tilde{\pi}(\cdot |s))}$, and $C=\frac{4\gamma \max_{s,a}|A_{\pi}(s,a)|}{(1-\gamma)^2}$.
\end{theorem}
% The right-hand policy improvement lower bound described in \cref{theorem_bound} serves as a surrogate objective for optimizing the true objective $\mathcal{J}(\tilde{\pi})$. The maximization of this surrogate objective at each policy iteration can guarantee a monotonic improvement of the true objective.z
The lower bound in Theorem~\ref{theorem_bound} serves as a surrogate objective of the true objective $\mathcal{J}(\tilde{\pi})$. It guarantees $\mathcal{J}(\tilde{\pi})$ is non-decreasing when improving the right-hand lower bound at each policy iteration. To leverage Theorem~\ref{theorem_bound} to cooperative MARL, we inherit HAPPO's settings of centralized training with decentralized execution (CTDE) paradigm \citep{oliehoek2008optimal,kraemer2016multi,mahajan2019maven} and the below assumption.
\begin{assumption}
  \label{assumption_policy}
  In the CTDE paradigm, agents take actions in a decentralized-execution manner and the policies $\pi^1(a^1_t|s_t),\dots,\pi^n(a^n_t|s_t)$ are independent from one another.
\end{assumption}
Then, by substituting $A_{\bm{\pi}}(s, \bm{a})$ with the advantage decomposition given by Equation (\ref{equation_multi_agent_advantage_decomposition}), the following lower bound can be derived (see proofs in \cite{kuba2022happo}):
% To extend \cref{theorem_bound} to cooperative MARL, 
% However, in cooperative MARL, the application of \cref{theorem_bound} is less straightforward. $A_{\bm{\pi}}(s, \bm{a})$ can be replaced with joint advantage decomposition defined by \cref{equation_multi_agent_advantage_decomposition} and the following lower bound can be derived:
\begin{equation}
  \small
  \label{equation_happo_bound}
  \mathcal{J}(\tilde{\bm{\pi}})\geq \mathcal{J}(\bm{\pi})+\sum\nolimits_{m=1}^n{M_{\bm{\pi}}^{i_{m}}(\tilde{\bm{\pi}}^{i_{1:m-1}}, \tilde{\pi}^{i_m})},
\end{equation}
where $M_{\bm{\pi}}^{i_{m}}(\tilde{\bm{\pi}}^{i_{1:m-1}}, \tilde{\pi}^{i_m})$ satisfies the below definition.
\begin{definition}\label{definition_surrogate}
  Let $\bm{\pi}$ denote the old joint policy, $\tilde{\bm{\pi}}$ denote the any future joint policy, $\mathcal{C}=\frac{4\gamma \max_{s,\bm{a}}|A_{\bm{\pi}}(s,\bm{a})|}{(1-\gamma)^2}$. For arbitrary disjoint sets $j_{1:k}$ and $i_{1:m}$, we define
  \begin{small}
  \begin{align}
     &M_{\bm{\pi}}^{i_{1:m}}(\tilde{\bm{\pi}}^{j_{1:k}}, \tilde{\bm{\pi}}^{i_{1:m}})= \mathbb{E} _{s\sim \rho_{\bm{\pi}},\bm{a}\sim \tilde{\bm{\pi}}}\big[ \nonumber\\
     &\qquad\qquad A_{\bm{\pi}}^{i_{1:m}}(s, \bm{a}^{j_{1:k}}, \bm{a}^{i_{1:m}})\big]-\sum\nolimits_{i \in i_{1:m}} \mathcal{C}D_{\rm KL}^{\max}({\pi}^{i}, {\tilde{\pi}}^{i}). \nonumber
  \end{align}
  \end{small}
\end{definition}
Equation~(\ref{equation_happo_bound}) describes a sequential update scheme of HAPPO at $k$-th iteration in which the policy of each agent $i_m$ ($m=1,\ldots,n$) is updated in sequence by $\pi_{k+1}^{i_m} = \arg\max_{\pi^{i_m}}M_{\bm{\pi}_k}^{i_{1:m}}(\bm{\pi}_{k+1}^{i_{1:m-1}}, \pi^{i_m})$ based on the update results of the preceding agents $i_{1:m-1}$. The update of each agent $i_m$ can result in a non-negative change of $M_{\bm{\pi}_k}^{i_{1:m}}$, thus enhancing the overall summation. However, the improvement guarantee of the overall summation is restricted to cases where there is no sharing of parameters among agents, i.e., the update of one agent's policy will not affect the objectives $M_{\bm{\pi}}^{i_{1:m}}$ of other agents. This guarantee becomes invalid in scenarios where parameter sharing exists, as the update to one agent's policy network parameters directly affects the parameters of other agents. This not only undermines the improvement guarantee but also results in excessive KL divergence, negatively impacting the performance (see Appendix~\ref{appendix_sequential_update_scheme_bad}). Hence, it is desirable for a general update scheme that can be applied to varieties of parameter-sharing types.

\section{The FP3O Algorithm}
In this part, we propose our \textbf{F}ull-\textbf{P}ipeline \textbf{PPO} (FP3O) algorithm for cooperative MARL that is highly versatile across different parameter-sharing types. We first introduce the single pipeline optimization as the basic pipeline and then establish multiple parallel optimization pipelines to yield the full-pipeline paradigm. Then, we visit its theoretical guarantee and present the implementation of FP3O in practice. 

\subsection{Single Pipeline Optimization}\label{section_single_pipeline}
We first introduce single pipeline optimization, the basic optimization pipeline of full pipeline optimization. It enables an arbitrary agent to initiate an optimization pipeline, which is the foundation for establishing full pipelines. To this end, it is necessary to separate the contribution of a specific agent from the joint advantage function, independent of the other agents, as each agent has no chance to access the actions and optimization directions of other agents at the beginning of each policy iteration. Hence, for $\forall i_{p} \in i_{1:n}=\mathcal{N}$, we give the following bound:
\begin{equation}
\small
   \label{equation_multibound}
   \mathcal{J}(\tilde{\bm{\pi}})\geq \mathcal{J}(\bm{\pi})+M_{\bm{\pi}}^{i_p}(\tilde{\pi}^\varnothing,{\tilde{\pi}}^{i_p})+M_{\bm{\pi}}^{-i_p}(\tilde{\pi}^{i_p}, \tilde{\bm{\pi}}^{-i_p}),
\end{equation}
where $\varnothing$ denotes the empty set. The proof is in Appendix~\ref{appendix_proof_equation5}. Equation~(\ref{equation_multibound}) decomposes the lower bound into two surrogate objective components. The first part, $M_{\bm{\pi}}^{i_p}(\tilde{\pi}^\varnothing,{\tilde{\pi}}^{i_p})$ of agent $i_p$, corresponds to a completely independent update step without regard to the policies of other agents $-i_{p}$. The second part, $M_{\bm{\pi}}^{-i_p}(\tilde{\pi}^{i_p}, \tilde{\bm{\pi}}^{-i_p})$ of agents $-i_p$, corresponds to a dependent update step that is impacted by the updated outcome of agent $i_{p}$. 
To get the lower bound improved at $k$-th iteration, we first do update by $\pi_{k+1}^{i_p} = \arg\max\limits_{\pi^{i_p}}M_{\bm{\pi}_k}^{i_p}(\pi_{k+1}^\varnothing,{\pi}^{i_p})$, which with importance sampling is equivalent to
\begin{equation}
    \small
   \label{equation_impsampling1}
   \pi_{k+1}^{i_p} = \arg\max\limits_{\pi^{i_p}}\left[\mathbb{E}_{s,\bm{a}} \left[r(\pi^{i_p}) A_{\bm{\pi}_k}(s,\bm{a})\right]-\mathcal{C}D_{\rm KL}^{\max}({\pi}^{i_p}_{k}, {\pi}^{i_p})\right],
\end{equation}
where $s\sim \rho_{\bm{\pi}_k},\bm{a}\sim \bm{\pi}_k$, $r(\pi^{i_p})=\frac{\pi^{i_p}(a^{i_p}|s)}{\pi_k^{i_p}(a^{i_p}|s)}$. Generally, we define $r(\pi^{i})=\frac{\pi^{i}(a^{i}|s)}{\pi_k^{i}(a^{i}|s)}$ as the probability ratio between $\pi^{i}$ and its old one $\pi^{i}_k$ at $k$-th iteration. We then update the policies of agents $-i_p$ by $\bm{\pi}_{k+1}^{-i_p} = \arg\max\limits_{\bm{\pi}^{-i_p}} M_{\bm{\pi}_k}^{-i_p}(\pi_{k+1}^{i_p}, \bm{\pi}^{-i_p})$, the importance-sampling objective of which is
{
\small
\begin{align}
   \label{equation_impsampling2}
   &\bm{\pi}_{k+1}^{-i_p} = \arg\max\limits_{\bm{\pi}^{-i_p}}\Big[\mathbb{E}_{s,\bm{a}} \Big[\left(r(\bm{\pi}^{-i_p}) - 1\right)  \nonumber \\
   &\qquad \cdot r(\pi_{k+1}^{i_p}) \cdot A_{\bm{\pi}_k}(s,\bm{a})\Big] -\sum\nolimits_{i \in -i_p} \mathcal{C}D_{\rm KL}^{\max}({\pi}_k^{i}, \pi^{i})\Big],
\end{align}}
where $r(\bm{\pi}^{-i_p})=\frac{\bm{\pi}^{-i_p}(\bm{a}^{-i_p}|s)}{\bm{\pi}_k^{-i_p}(\bm{a}^{-i_p}|s)}$, $r(\pi_{k+1}^{i_p})=\frac{\pi_{k+1}^{i_p}(a^{i_p}|s)}{\pi_k^{i_p}(a^{i_p}|s)}$. The proof of importance-sampling objectives is in Appendix~\ref{appendix_proof_equation67}. Single pipeline optimization is not suitable for all types of parameter sharing since it requires the update of agent $i_p$ (or $-i_p$) not to affect the objectives of other agents $M_{\bm{\pi}}^{-i_p}$ (or $M_{\bm{\pi}}^{i_p}$). This limitation has also been observed in the sequential update scheme of HAPPO. 
Nevertheless, as an \emph{arbitrary} agent can be the separated agent $i_{p}$ in Equation~(\ref{equation_multibound}), a variety of decomposition ways are made available. This provides critical insight into the full-pipeline paradigm.

\subsection{Full-Pipeline Paradigm}\label{section_full_pipeline_optimization}
Instead of relying on one decomposition way, we can let every agent be the separated agent by setting $i_p=i_1,...,i_n$ and derive multiple, fully equivalent policy improvement lower bounds on $n$ pipelines: 
{\small
\begin{align}
  \label{equation_fullmultibound}
  \mathcal{J}(\bm{\pi})+M_{\bm{\pi}}^{i_1}(\tilde{\pi}^\varnothing,{\tilde{\pi}}^{i_1}&)+ M_{\bm{\pi}}^{-i_1}(\tilde{\pi}^{i_1}, \tilde{\bm{\pi}}^{-i_1}) \nonumber                           \\
                                                                 & \vdots                                                                                    \nonumber \\
  \mathcal{J}(\bm{\pi})+M_{\bm{\pi}}^{i_n}(\tilde{\pi}^\varnothing,{\tilde{\pi}}^{i_n}&)+M_{\bm{\pi}}^{-i_n}(\tilde{\pi}^{i_n}, \tilde{\bm{\pi}}^{-i_n}).
\end{align}}
The optimizations of these $n$ pipelines, which we refer to as the \emph{full-pipeline paradigm}, are performed in a parallel manner. The first update step involves maximizing $M_{\bm{\pi}}^{i_1}(\tilde{\pi}^\varnothing,{\tilde{\pi}}^{i_1}),...,M_{\bm{\pi}}^{i_n}(\tilde{\pi}^\varnothing,{\tilde{\pi}}^{i_n})$ for each pipeline from $1$ to $n$. These objectives are independent of the update policies of other agents, which allows us to initiate the parallel optimization of pipelines without any restrictions. As such, we call the optimization of $M_{\bm{\pi}}^{i_1}(\tilde{\pi}^\varnothing,{\tilde{\pi}}^{i_1}),...,M_{\bm{\pi}}^{i_n}(\tilde{\pi}^\varnothing,{\tilde{\pi}}^{i_n})$ as the \emph{independent step}.
The second update step involves maximizing $M_{\bm{\pi}}^{-i_1}(\tilde{\pi}^{i_1}, \tilde{\bm{\pi}}^{-i_1}),...,M_{\bm{\pi}}^{-i_n}(\tilde{\pi}^{i_n}, \tilde{\bm{\pi}}^{-i_n})$ for each pipeline from $1$ to $n$. The optimization of this step relies on the update results of the independent step, hence it is referred to as the \emph{dependent step}. Then, the full-pipeline optimization can be completed with the following two measures. 

\paragraph{Non-overlapping selection.} 
% When parallelizing the optimization of $\{\mathop{\rm maximize}\nolimits_{\tilde{\bm{\pi}}^{-i_p}}M_{\bm{\pi}}^{-i_p}(\tilde{\pi}^{i_p}, \tilde{\bm{\pi}}^{-i_p})\}_{p=1}^n$ on pipelines $p=1,...,n$, there are overlapping policies among $\tilde{\bm{\pi}}^{-i_1},...,\tilde{\bm{\pi}}^{-i_n}$. 
In the process of parallelizing $\{\mathop{\rm maximize}\nolimits_{\tilde{\bm{\pi}}^{-i_p}}M_{\bm{\pi}}^{-i_p}(\tilde{\pi}^{i_p}, \tilde{\bm{\pi}}^{-i_p})\}_{p=1}^n$ during the dependent step across pipelines $p=1,...,n$, we encounter overlapping policies among $\tilde{\bm{\pi}}^{-i_1},...,\tilde{\bm{\pi}}^{-i_n}$. This overlap implies that a single agent's policy (for instance, $\tilde{\pi}^1$) is subject to participate in the optimization of multiple pipelines (e.g. $p=2,...,n$) simultaneously, which is impractical. Thus, we propose a specific selection criterion that selects \emph{one} specific agent's policy from $\tilde{\bm{\pi}}^{-i_p}$ for each pipeline $p$ to optimize, ensuring there are no overlapping policies among pipelines. This non-overlapping selection criterion satisfies
\begin{equation}
\small
   \label{equation_mapping}
   f: i_{1:n}\rightarrow j_{1:n}, \quad \text{where}\,\, j_p \in -i_p , \,\, j_{1:n}=\mathcal{N}.             
\end{equation}
% It outlines a selection rule for each pipeline $p$: policy $\tilde{\pi}^{i_p}$ is optimized at the independent step, while only $\tilde{\pi}^{j_p}$ (where $j_p=f(i_p)$) is selected from $\tilde{\bm{\pi}}^{-i_p}$ to be optimized at the dependent step. 
This selection criterion specifies that, for each pipeline $p$, we optimize policy $\tilde{\pi}^{i_p}$ at the independent step and then select a candidate policy $\tilde{\pi}^{j_p}$ (where $j_p=f(i_p)$) from $\tilde{\bm{\pi}}^{-i_p}$ to be optimized at the dependent step.
By doing so, it ensures no agent's policy is included in the optimization of more than one pipeline at a time. 
To indicate that $\tilde{\pi}^{j_p}$ is the selected candidate policy, we rephrase $M_{\bm{\pi}}^{-i_p}(\tilde{\pi}^{i_p},\tilde{\bm{\pi}}^{-i_p})$ as $M_{\bm{\pi}}^{-i_p}(\tilde{\bm{\pi}}^{-j_p},\tilde{\pi}^{j_p})$. Thus, the dependent step executes $\{\mathop{\rm maximize}\nolimits_{\tilde{\pi}^{j_p}}M_{\bm{\pi}}^{-i_p}(\tilde{\bm{\pi}}^{-j_p},\tilde{\pi}^{j_p})\}_{p=1}^n$ instead. However, this does not mean that updates to the non-selected agents' policies are excluded. Thanks to the non-overlapping selection, they can be optimized by other pipelines as $j_{1:n}=\mathcal{N}$ ensures that $n$ agents can be exactly selected by $n$ pipelines for optimization without overlap.

\paragraph{The modified objectives.} 
% A natural question: is the improved lower bound from the single pipeline optimization still valid under the full-pipeline paradigm? 
In the full-pipeline paradigm, the pipelines are not independent but rather exhibit complex interconnections, e.g., the policy optimization of one pipeline will also affect the lower bounds of other pipelines. This could potentially compromise the inherent monotonicity of the single pipeline optimization. 
Therefore, to maintain the monotonic improvement of the lower bound in the full-pipeline paradigm, we next give the modified objectives of original $M_{\bm{\pi}}^{i_p}(\tilde{\pi}^\varnothing,{\tilde{\pi}}^{i_p})$ and $M_{\bm{\pi}}^{-i_p}(\tilde{\bm{\pi}}^{-j_p},\tilde{\pi}^{j_p})$. We split $A_{\bm{\pi}}(s, \bm{a})$ into $n$ arbitrary scalar values $A^1,...,A^n$, i.e., $\sum_{i=1}^n{A^{i}}=A_{\bm{\pi}}(s, \bm{a})$. At the \emph{independent step} of $k$-th iteration, we perform the parallel maximization of the modified objective $\widetilde{M}_{\bm{\pi}_k}^{i_p}(\pi_{k+1}^\varnothing,\pi^{i_p})$ w.r.t $\pi^{i_p}$ of all pipelines: 
\begin{equation}
    \small
   \label{equation_full_step1}
   \Big\{\pi_{k+0.5}^{i_p} = \arg\max\limits_{\pi^{i_p}}\Big[\mathbb{E}_{s,\bm{a}} \big[r(\pi^{i_p}) A^{i_p}\big]-\mathcal{C}D_{\rm KL}^{\max}({\pi}^{i_p}_{k}, {\pi}^{i_p})\Big]\Big\}_{p=1}^n,
\end{equation}
where $\{\pi_{k+0.5}^{i_p}\}_{p=1}^n$ denotes the intermediate policies of $k$-th iteration, specifically the update outcomes of the independent step. In the full-pipeline paradigm, the result $\pi_{k+0.5}^{i_p}$ generated by pipeline $p$ at the independent step will undergo additional optimization by one of the complementary pipelines $-p$ at the dependent step due to the non-overlapping selection. 
This mutual influence on the policies among pipelines is fundamentally different from the single pipeline optimization and HAPPO. 
Then, the \emph{dependent step} maximizes the modified objective $\widetilde{M}_{\bm{\pi}_k}^{-i_p}(\bm{\pi}_{k+1}^{-j_p}, \pi^{j_p})$ w.r.t $\pi^{j_p}$ of all pipelines:
\begin{small}
\begin{align}
    \label{equation_full_step2}
     \Big\{\pi_{k+1}^{j_p} &= \arg\max\limits_{\pi^{j_p}}\Big[\mathbb{E}_{s,\bm{a}} \Big[\Big(r(\pi^{j_p})\cdot r\big(\bm{\pi}_{k+1}^{-\{i_p,j_p\}}\big) - 1\Big)\nonumber\\
     &\qquad \ \ \cdot r(\pi_{k+1}^{i_p}) \cdot A^{j_p}\Big] - \mathcal{C}D_{\rm KL}^{\max}({\pi}_k^{j_p}, \pi^{j_p})\Big]\Big\}_{p=1}^n   \\ 
     \text{s.t.}\quad &\sum\nolimits_{i\in \mathcal{N}}{\mu(\bm{\pi}^{-i}_{k+1})} \geq \sum\nolimits_{i \in \mathcal{N}}{\mu(\bm{\pi}^{-i}_{k})}, \nonumber \\
     % &\mu(\bar{\bm{\pi}}^{-i}) = \mathbb{E}_{\substack{s\sim \rho_{\bm{\pi}_k} \\ a^{i}\sim \pi_{k+0.5}^{i} \\ \bm{a}^{-i}\sim \bar{\bm{\pi}}^{-i}}}\left[A_{\bm{\pi}_{k}}^{i}\right].\nonumber
     &\mu(\tilde{\bm{\pi}}^{-i}) = \mathbb{E}_{s\sim \rho_{\bm{\pi}_k, a^{i}\sim \pi_{k+0.5}^{i} , \bm{a}^{-i}\sim \tilde{\bm{\pi}}^{-i}}}\big[A^{i}\big].\nonumber
\end{align}
\end{small}
% Note that, in the cases of parameter sharing, Equation~(\ref{equation_full_step1}) is implemented by maximizing the expectation along pipelines: $\mathop{\rm maximize}\limits_{\pi^{i_1},...,\pi^{i_n}}{\frac{1}{n}\sum_{p=1}^n \widetilde{M}_{\bm{\pi}_k}^{i_p}(\pi_{k+1}^\varnothing,\pi^{i_p})}$ due to the interdependency among the policies of agents. The same applies to Equation~(\ref{equation_full_step2}).
% Besides, a constraint on the expectation of $\{A^i\}_{i\sim \mathcal{N}}$ is introduced at the dependent step to ensure the monotonically improved lower bounds of all pipelines. In the next section, we will prove how the modified objectives can guarantee non-decreasing lower bounds.
Note that, in the cases of parameter sharing, Equations~(\ref{equation_full_step1}) and (\ref{equation_full_step2}) are executed by maximizing the expectation along pipelines due to the interdependency among the policies of agents, e.g. $\mathop{\rm maximize}\nolimits_{(\pi^{i_1},...,\pi^{i_n})}{\frac{1}{n}\sum_{p=1}^n \widetilde{M}_{\bm{\pi}_k}^{i_p}(\pi_{k+1}^\varnothing,\pi^{i_p})}$.
Besides, a constraint on the expectation of $\{A^i\}_{i\sim \mathcal{N}}$ is introduced at the dependent step to ensure the monotonicity of all pipelines' lower bounds. In the next section, we will prove how the modified objectives can guarantee non-decreasing lower bounds.

\subsection{Theoretical Foundation}\label{section_theoretical_guarantee}
We now present the theoretical analysis for full-pipeline optimization. We begin by providing a lemma that establishes the relationship between the lower bounds of $n$ pipelines.
\begin{lemma}[Shared Lower Bound]
   \label{lemma_shared_lower_bound}
   Let $\mathcal{L}_{\bm{\pi}}(\tilde{\pi}^{i_p},\tilde{\pi}^{j_p})=\mathcal{J}(\bm{\pi})+M_{\bm{\pi}}^{i_p}(\tilde{\pi}^\varnothing,{\tilde{\pi}}^{i_p})+M_{\bm{\pi}}^{-i_p}(\tilde{\bm{\pi}}^{-j_p}, \tilde{\pi}^{j_p})$ be the lower bound of pipeline $p$. For all pipelines $p=1,...,n$, their lower bounds satisfies
   \begin{equation}
   \small
      \mathcal{L}_{\bm{\pi}}(\tilde{\pi}^{i_1},\tilde{\pi}^{j_1}) = ,\dots, = \mathcal{L}_{\bm{\pi}}(\tilde{\pi}^{i_n},\tilde{\pi}^{j_n}). \nonumber
   \end{equation}
\end{lemma}
The proof is in Appendix~\ref{appendix_shared_lower_bound}. 
The conclusion of Lemma~\ref{lemma_shared_lower_bound} suggests that, while $n$ pipelines employ distinct decomposition ways, \emph{a shared lower bound persists among them}, i.e., the lower bounds of all pipelines are equal. This means that any policy update has the equivalent influence on all pipelines' lower bounds. We introduce the notation $\mathcal{L}^{S}_{\bm{\pi}}(\tilde{\pi}^{1},\dots,\tilde{\pi}^{n})$ as the shared lower bound. Thus, instead of examining the lower bounds of all pipelines, we need only to prove that the shared lower bound is monotonically increasing in the full-pipeline optimization.

Given that the core process of the full-pipeline paradigm is the successive maximization of the modified objectives $\widetilde{M}_{\bm{\pi}_k}^{i_p}$ and $\widetilde{M}_{\bm{\pi}_k}^{-i_p}$ at all pipelines, the following lemma next provides the resulting inequality drawn from the independent (Equation~(\ref{equation_full_step1})) and dependent (Equation~(\ref{equation_full_step2})) steps.

\begin{lemma}
  \label{lemma_KL_expectation}
  At $k$-the policy iteration, the intermediate policies $\{\pi_{k+0.5}^{i_p}\}_{p=1}^n$ and the new policies $\{\pi_{k+1}^{j_p}\}_{p=1}^n$ are produced by Equations~(\ref{equation_full_step1}) and (\ref{equation_full_step2}) respectively. Regardless of whether parameter sharing is used or not, the following inequality holds:
  \begin{small}
  \begin{align}
   &\sum_{i=1}^n \mathcal{C}D_{\rm KL}^{\max}({\pi}^{i}_{k}, {\pi}_{k+1}^{i}) \leq \sum_{i=1}^n \mathbb{E}_{s,\bm{a}} \Big[\nonumber\\ &\qquad \Big(r(\bm{\pi}_{k+1})-r(\pi_{k+0.5}^{i})\cdot r(\bm{\pi}_{k+1}^{-i})+r(\pi_{k+0.5}^{i})-1\Big) A^{i} \Big].\nonumber
\end{align}
\end{small}
\end{lemma}
The proof is outlined in Appendix~\ref{appendix_KL_expectation}. It is worth noting that during the derivation of Lemma~\ref{lemma_KL_expectation}, these parallel pipelines are not treated as independent of each other, but rather as exhibiting complex interconnections. It is these interconnections that help to reach a \emph{unified} inequality between the KL divergence term and expectation term for different parameter-sharing types. It forms a fundamental foundation for the generality of the algorithm as well as our following main theorem.
\begin{theorem}
   \label{theorem_imporved_lower_bound_fp3o}
   Whether parameter sharing is used or not, given Assumption~\ref{assumption_policy}, the new policies $\{\pi_{k+1}^{j_p}\}_{p=1}^n$ of full-pipeline optimization enable the shared lower bound to be monotonically improved:
   $\mathcal{L}^{S}_{\bm{\pi}_k}(\pi_{k+1}^{1},...,\pi_{k+1}^{n}) \geq \mathcal{L}^{S}_{\bm{\pi}_k}(\pi_{k}^{1},...,\pi_{k}^{n}),$
   which guarantees $\mathcal{J}(\bm{\pi}_{k+1})\geq \mathcal{J}(\bm{\pi}_{k})$.
\end{theorem}
% \begin{proof}
%    The proof of the improved shared lower bound is unfolded by incorporating the constraint in Equation~(\ref{equation_full_step2}) with the inequality in Lemma~\ref{lemma_KL_expectation}. More details are outlined in Appendix~\ref{appendix_imporved_lower_bound_fp3o}.
% \end{proof}
The proof unfolds based on the inequality in Lemma~\ref{lemma_KL_expectation} and the detailed derivation is in Appendix~\ref{appendix_imporved_lower_bound_fp3o}. This theorem demonstrates that the shared lower bound consistently exhibits a non-decreasing trend for different parameter-sharing types as the policy iteration progresses. In other words, it guarantees that the lower bounds $\{\mathcal{L}_{\bm{\pi}_k}(\pi_{k+1}^{i_p},\pi_{k+1}^{j_p})\}_{p=1}^n$ of all pipelines, as defined by Lemma~\ref{lemma_shared_lower_bound}, are non-decreasing. This ensures that the true objective $\mathcal{J}$ can enjoy the monotonic improvement property in the full-pipeline paradigm. 

\subsection{Practical Implementation}\label{section_fp3o}
\begin{algorithm}[tb]
   \caption{Policy iteration of FP3O}
   \label{algorithm_policyiteration_fp3o}
 \begin{algorithmic}[1]
      \STATE Initialize the parameters of policy networks, $\bm{\theta}_{0}=\{\theta_{0}^1,\dots,\theta_{0}^n\}$.
      \FOR{$k=1,2,...$ until convergence}
      \STATE Estimate the advantage functions $\hat{A}_{\pi}$ and $\{\hat{A}^i\}_{i=1}^n$.
      % \STATE Solve the optimization of independent and dependent steps for all pipelines:
      \STATE Solve the independent step at all pipelines: \\
      % $\left\{
      %    \begin{array}{c}
      %       \theta_{k+0.5}^{i_1}=\arg\max\limits_{\theta^{i_1}}L^{\rm FP3O}(\theta^{i_1}|\bm{\theta}_{k}) \\
      %       \vdots                                                                                            \\
      %       \theta_{k+0.5}^{i_n}=\arg\max\limits_{\theta^{i_n}}L^{\rm FP3O}(\theta^{i_n}|\bm{\theta}_{k})    \\
      %    \end{array}
      % \right.$  
      $\Big\{\theta_{k+0.5}^{i_p}=\arg\max\limits_{\theta^{i_p}}L^{\rm FP3O}(\theta^{i_p}|\bm{\theta}_{k})\Big\}_{p=1}^{n}$
      \STATE Calculate the probabilities $\big\{\pi^i_{\theta_{k+0.5}^i}(a^i|s)\big\}_{i,a^i,s}$ with respect to all agents, actions and states.
      \STATE Select the candidate policies according to non-overlapping selection.
      \IF{$\sum_{i}\mu(\bm{\pi}^{-i}_{\bm{\theta}^{-i}_{k+0.5}})\geq \sum_{i}\mu(\bm{\pi}^{-i}_{\bm{\theta}^{-i}_{k}})$}
      \STATE Solve the dependent step at all pipelines:
      % $\left\{
      %    \begin{array}{c}
      %       \theta_{k+1}^{j_1}=\arg\max\limits_{\theta^{j_1}}L^{\rm FP3O}(\theta^{j_1}|\bm{\theta}_{k+0.5}) \\
      %       \vdots                                                                                            \\
      %       \theta_{k+1}^{j_n}=\arg\max\limits_{\theta^{j_n}}L^{\rm FP3O}(\theta^{j_n}|\bm{\theta}_{k+0.5})    \\
      %    \end{array}
      % \right.$ 
      \\
      $\Big\{\theta_{k+1}^{j_p}=\arg\max\limits_{\theta^{j_p}}L^{\rm FP3O}(\theta^{j_p}|\bm{\theta}_{k+0.5})\Big\}_{p=1}^{n}$
      \ENDIF
      \ENDFOR
 \end{algorithmic}
 \end{algorithm}

% Given that the full-pipeline optimization consists of two steps, we denote the old policy parameter set that we aim to improve as $\theta_{\rm old}$, the intermediate policy parameter set generated by the independent step as $\theta_{\rm int}$, and the new policy parameter set generated by the dependent step as $\theta_{\rm new}$.
Next, we leverage full-pipeline optimization into practice with the parameterized policies. At dependent step (Equation~(\ref{equation_full_step2})) of $k$-th iteration, the policies $\bm{\pi}^{-\{i_p,j_p\}}_{k+1}$ and $\pi_{k+1}^{i_p}$ are dynamically changed by other pipelines at each mini-batch update. One way is to directly use the dynamically changed policies like \citep{wu2021coordinated}. However, it entails a significant computational burden as it requires the calculation of the probabilities at every mini-batch update. Additionally, the dynamic probabilities also lead to non-stationary optimization objectives. Fortunately, the full-pipeline paradigm enables us to utilize the intermediate policies $\{\pi_{k+0.5}^1,...,\pi_{k+0.5}^n\}$ as a heuristic approximation, providing a more computationally efficient and stable alternative for the dynamic policies. This, consequently, simplifies the constraint in Equation~(\ref{equation_full_step2}) as a condition.

Then, we parameterize the agent's policy $\pi_{\theta^{i}}^{i}$ by parameter vector $\theta^{i}$, and $\bm{\pi}_{\bm{\theta}}$ is the joint parameterized policy with $\bm{\theta}=\{\theta^{1},...,\theta^{n}\}$. Here, these parameters can be in any form of sharing. Following \cite{schulman2017ppo}, we replace the KL divergence penalty with the clipping mechanism, incorporating an outer minimization to ensure that they are the lower bounds to the modified objectives. To simplify the notation, we integrate the modified objectives of the independent and dependent steps into a main objective and yield the FP3O algorithm, which is a programming-friendly implementation of full-pipeline optimization (see Appendix~\ref{appendix_reproducibility}). The main objective $L^{\rm FP3O}(\theta^i|\bar{\bm{\theta}})$ is formulated as
\begin{small}
\begin{align}
   % \label{equation_main_objective_fp3o}
   &\mathbb{E}_{s,\bm{a}}\bigg[\min\bigg(\Big(r(\pi_{\theta^{i}}^{i})\cdot r\big(\bm{\pi}^{-\{j,i\}}_{\bar{\bm{\theta}}^{-\{j,i\}}}\big)-1\Big) r(\pi_{\bar{\theta}^{j}}^{j})\cdot\hat{A}^{i}, \nonumber\\ 
   &\qquad \Big({\rm clip}\big(r(\pi_{\theta^{i}}^{i}),1\pm\epsilon\big)\cdot r\big(\bm{\pi}^{-\{j,i\}}_{\bar{\bm{\theta}}^{-\{j,i\}}}\big)-1\Big) r(\pi_{\bar{\theta}^{j}}^{j})\cdot \hat{A}^{i}\bigg)\bigg],  \nonumber
\end{align}
\end{small}
where $j=f^{-1}(i)$, $f^{-1}(\cdot)$ is the inverse function of Equation~(\ref{equation_mapping}), and $\hat{A}^{i}$ is generated by $\hat{A}_{\bm{\pi}}$ (e.g., average split), which is an estimator of the joint advantage function (e.g., Generalized Advantage Estimation \citep{schulman2015gae}). $\{L^{\rm FP3O}(\theta^{i_p}|\bm{\theta}_{k})\}_{p=1}^n$ and $\{L^{\rm FP3O}(\theta^{j_p}|\bm{\theta}_{k+0.5})\}_{p=1}^n$ serve the independent and dependent steps respectively (see Appendix~\ref{appendix_main_objective_fp3o}). We present its policy iteration in Algorithm~\ref{algorithm_policyiteration_fp3o}.

\section{Experiments}\label{section_experiment}
\paragraph{Environments.}
We evaluate our algorithm on two challenging cooperative benchmarks: Multi-Agent MuJoCo (MAMuJoCo)~\citep{de2020mamujoco} and StarCraftII Multi-Agent Challenge (SMAC)~\citep{samvelyan2019starcraft}. MAMuJoCo is a continuous, partially observable task that groups different joints of a robot in MuJoCo simulator~\citep{todorov2012mujoco} and models them as different agents. The agents are regarded as heterogeneous, given the diversification of body part control. SMAC is a discrete, partially observable task aimed at training a team of ally units to defeat an opponent team, with scenarios categorized into homogeneous or heterogeneous tasks based on the types of ally units. See Appendix~\ref{appendix_environments} for the details of environments.

\paragraph{Network types and baselines.} In assessing the versatility of our algorithm, we adopt three network types including full, partial, and non-parameter sharing. Full parameter sharing means that different agents' policy networks share the same set of parameters, which is widely used to enhance learning efficiency and reduce parameters. For non-parameter sharing, each agent has an individual set of parameters to make personalized decisions, particularly for heterogeneous agents. Partial parameter sharing combines the properties of full and non-parameter sharing. In our experiments, we let the last layers of different agents' policy networks have individual parameters, while other layers share the same set of parameters. 
We extend HAPPO \citep{kuba2022happo}, MAPPO \citep{yu2021mappo}, and IPPO \citep{de2020ippo} to all three parameter-sharing types, using them as baselines. We also include CoPPO \citep{wu2021coordinated} as a baseline with full parameter sharing in Appendix~\ref{appendix_CoPPO_Baseline} due to its limitation in adaptability to partial/non-parameter sharing. For fairness, the settings of all algorithms are identical including all hyperparameters and network inputs. More details are provided in Appendix~\ref{appendix_hyperparameters}. 

% \subsection{Results}
\paragraph{Analysis on the condition.}
\begin{table*}[t]
   \small
   \centering
   \begin{tabular}{ @{\,}c @{\,}| @{\,\,} c @{\,\,} | c c c c | @{\,\,}c@{\,\,}}
      \toprule
 
      Network & Task                            & FP3O                             & HAPPO              & MAPPO              & IPPO               & Agent \\ \midrule
      % \multirow{6}{*}{$\begin{array}{c}\rotatebox{90}{\text{Shared}} \end{array}$}
      \multirow{6}{*}{$\begin{array}{@{\,}c@{\,}}{\text{Full Param. Sharing}} \end{array}$}
        & 2-Agent Reacher [2$\times$1]     & \phantom{0}\textbf{-32.7}\tiny{$\pm$5.8} \phantom{0}  & \phantom{0}\underline{-33.2}\tiny{$\pm$2.2} \phantom{0}   & \phantom{0}-45.4\tiny{$\pm$11.3} \                        & \phantom{0}-51.5\tiny{$\pm$4.5} \phantom{0} & HE    \\
        & 2-Agent Ant [2$\times$4]         & \textbf{2617.4}\tiny{$\pm$190.9}                      & \underline{2296.2}\tiny{$\pm$181.3}                       & 1939.3\tiny{$\pm$52.1}\phantom{0}                         & 1584.8\tiny{$\pm$192.3} & HE    \\
        & 2-Agent Walker [2$\times$3]      & \textbf{2517.2}\tiny{$\pm$393.7}                      & \phantom{0}907.0\tiny{$\pm$280.8}                       & \underline{1684.3}\tiny{$\pm$858.7}                         & \phantom{0}523.4\tiny{$\pm$208.4} & HE    \\
        & 3-Agent Hopper [3$\times$1]      & \underline{3514.1}\tiny{$\pm$133.4}                   & \textbf{{3554.7}}\tiny{$\pm$118.3}                       & 3373.3\tiny{$\pm$238.3}                                    & 2203.7\tiny{$\pm$855.0} & HE    \\
        & 6-Agent HalfCheetah [6$\times$1] & \textbf{4128.3}\tiny{$\pm$366.9}                      & \underline{4058.8}\tiny{$\pm$129.7}                       & 3170.5\tiny{$\pm$116.6}                                   & 2683.5\tiny{$\pm$431.3} & HE    \\
        & Manyagent Swimmer [8$\times$2]   & \phantom{0}\textbf{419.1}\tiny{$\pm$33.8}\phantom{0}     & \phantom{0}404.9\tiny{$\pm$38.9}\phantom{0}                 & \phantom{0}\underline{409.2}\tiny{$\pm$48.3}\phantom{0} & \phantom{0}\phantom{0}63.6\tiny{$\pm$21.8}\phantom{0} & HE    \\
      %   & 17-Agent Humanoid [17$\times$1]  & xxx\tiny{$\pm$xxx}               & xxx\tiny{$\pm$xxx} & xxx\tiny{$\pm$xxx} & xxx\tiny{$\pm$xxx} & HE    \\ 
      \midrule
      %   \multirow{6}{*}{$\begin{array}{c}\rotatebox{90}{\text{Partially Shared}} \end{array}$}
      \multirow{6}{*}{$\begin{array}{@{\,}c@{\,}}{\text{Partial Param. Sharing}} \end{array}$}
        & 2-Agent Reacher [2$\times$1]     & \phantom{0}\underline{-34.6}\tiny{$\pm$3.2}\phantom{0}\phantom{0} & \phantom{0}-36.7\tiny{$\pm$4.1} \phantom{0}    & \phantom{0}\textbf{-33.4}\tiny{$\pm$4.8} \phantom{0}    & \phantom{0}-36.1\tiny{$\pm$1.9} \phantom{0} & HE    \\
        & 2-Agent Ant [2$\times$4]         & \textbf{3115.2}\tiny{$\pm$145.5}                                   & \underline{2809.5}\tiny{$\pm$61.9}\phantom{0} & {2306.7}\tiny{$\pm$255.4}                               & {2026.6}\tiny{$\pm$207.2} & HE    \\
        & 2-Agent Walker [2$\times$3]      & \textbf{3132.3}\tiny{$\pm$774.6}                                   & \underline{2717.4}\tiny{$\pm$581.7}           & 2748.7\tiny{$\pm$734.1}                                & \phantom{0}{651.7}\tiny{$\pm$266.0} & HE    \\
        & 3-Agent Hopper [3$\times$1]      & \textbf{2986.3}\tiny{$\pm$617.6}                                   & 1547.0\tiny{$\pm$321.7}                      & \underline{2678.1}\tiny{$\pm$997.8}                     & {2486.8}\tiny{$\pm$712.3} & HE    \\
        & 6-Agent HalfCheetah [6$\times$1] & \textbf{4879.3}\tiny{$\pm$845.1}                                    & 4413.7\tiny{$\pm$840.0}                     & \underline{4432.6}\tiny{$\pm$196.1}                     & {4336.1}\tiny{$\pm$768.7} & HE    \\
        & Manyagent Swimmer [8$\times$2]   & \phantom{0}\textbf{402.7}\tiny{$\pm$21.7}\phantom{0}                  & \phantom{0}353.9\tiny{$\pm$15.4}\phantom{0}  & \phantom{0}\underline{366.0}\tiny{$\pm$27.0}\phantom{0} & \phantom{0}\phantom{0}76.1\tiny{$\pm$2.8}\phantom{0}\phantom{0} & HE    \\
      %   & 17-Agent Humanoid [17$\times$1]  & xxx\tiny{$\pm$xxx}               & xxx\tiny{$\pm$xxx} & xxx\tiny{$\pm$xxx} & xxx\tiny{$\pm$xxx} & HE    \\ 
      \midrule
      % \multirow{6}{*}{$\begin{array}{c}\rotatebox{90}{\text{Separated}} \end{array}$}
      \multirow{6}{*}{$\begin{array}{@{\,}c@{\,}}{\text{Non-Param. Sharing}} \end{array}$}
        & 2-Agent Reacher [2$\times$1]     & \phantom{0}\textbf{-26.5}\tiny{$\pm$2.6}\phantom{0}\phantom{0}    & \phantom{0}-34.9\tiny{$\pm$4.3} \phantom{0}                  & \phantom{0}\underline{-31.4}\tiny{$\pm$6.4} \phantom{0} & \phantom{0}{-31.7}\tiny{$\pm$5.4} \phantom{0} & HE    \\
        & 2-Agent Ant [2$\times$4]         & \textbf{2793.8}\tiny{$\pm$385.1}                                  & \underline{2043.0}\tiny{$\pm$187.7}                  & 1825.4\tiny{$\pm$157.0}                                  & 1469.9\tiny{$\pm$36.1}\phantom{0} & HE    \\
        & 2-Agent Walker [2$\times$3]      & \textbf{3562.4}\tiny{$\pm$666.4}                                    & \underline{2461.8}\tiny{$\pm$368.3}                  & {2303.2}\tiny{$\pm$815.9}                                & \phantom{0}405.6\tiny{$\pm$107.9} & HE    \\
        & 3-Agent Hopper [3$\times$1]      & \underline{3560.2}\tiny{$\pm$107.7}                                    & \textbf{3597.1}\tiny{$\pm$87.8}\phantom{0}        & {{3390.1}}\tiny{$\pm$225.9}                            & 1851.4\tiny{$\pm$683.1} & HE    \\
        & 6-Agent HalfCheetah [6$\times$1] & \textbf{4792.7}\tiny{$\pm$263.9}                                    & \underline{4408.7}\tiny{$\pm$170.2}                  & 4103.8\tiny{$\pm$233.9}                                & 3831.3\tiny{$\pm$131.5} & HE    \\
        & Manyagent Swimmer [8$\times$2]   & \phantom{0}\textbf{379.1}\tiny{$\pm$48.3}\phantom{0}                 & \phantom{0}\underline{317.0}\tiny{$\pm$12.2}\phantom{0} & \phantom{0}303.3\tiny{$\pm$44.0}\phantom{0}            & \phantom{0}\phantom{0}97.2\tiny{$\pm$42.0}\phantom{0} & HE    \\
      %   & 17-Agent Humanoid [17$\times$1]  & xxx\tiny{$\pm$xxx}               & xxx\tiny{$\pm$xxx} & xxx\tiny{$\pm$xxx} & xxx\tiny{$\pm$xxx} & HE    \\ 
      \midrule
   \end{tabular}
   \caption{The average evaluation rewards and standard deviations on MAMuJoCo. We bold the best performances and underline the second-best performances. HE denotes heterogeneous agents.}
   \label{table_MAMuJoCo_results}
\end{table*}
\begin{figure}[t]
   \centering
   \includegraphics[width=0.9\linewidth]{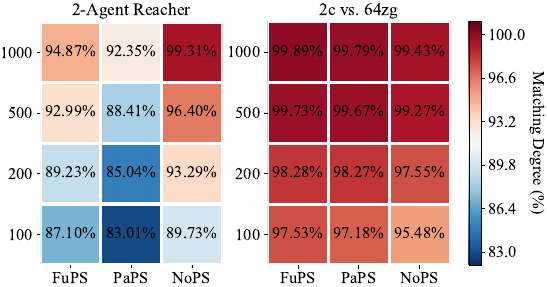}
   %\framebox[4.0in]{$\;$}
   \caption{Matching degree between the condition and constraint concerning rollout sizes and network types. FuPS, PaPS, and NoPS denote full, partial, and non-parameter sharing. Rollout size is set to 100, 200, 500, and 1000.}
   \label{fig_matching_degree}
\end{figure}
FP3O uses an approximation to substitute the dynamically changed policies, thus transforming the constraint into a more computationally efficient condition (line 7 in Algorithm~\ref{algorithm_policyiteration_fp3o}). 
To evaluate how well the condition aligns with the original constraint, we conduct the experiments on two distinct scenarios: 2-Agent Reacher (heterogeneous-agent task) and 2c vs. 64zg (homogeneous-agent task). We introduce the metric of \emph{matching degree} defined as $\frac{N_m}{N_t}\times 100\%$ to evaluate this alignment, where $N_m$ denotes the number of iterations in which the condition aligns with the original constraint, $N_t$ denotes the total number of iterations. Besides, given the rollout set $(s_0,\bm{a}_0,r_0),...,(s_N,\bm{a}_N,r_N)$, we define $N\times n$ as the \emph{rollout size}. As indicated in Figure~\ref{fig_matching_degree}, the 2c vs. 64zg task exhibits a higher overall matching degree compared to the 2-Agent Reacher task. This can be attributed to the homogeneity of agents in 2c vs. 64zg, where the policies $\bm{\pi}^{-i}$ of agent $-i$ and $A^i_{\bm{\pi}}$ of agent $i$ tend to be more harmonious, making it more probable for the condition and constraint to be simultaneously satisfied. Moreover, as parameter sharing is more proficient in representing homogeneity but less effective in portraying heterogeneity, as shown in Figure~\ref{fig_matching_degree}, the 2c vs. 64zg task displays better matching degree in parameter sharing cases while the 2-Agent Reacher task performs better with non-parameter sharing. Nevertheless, regardless of the agent and parameter-sharing types, an increased rollout size, leading to a more accurate Monte Carlo estimation for $\mu(\cdot)$, can improve the matching degree. This result holds the \emph{practical significance}, given that the rollout size in practice is typically on the order of $10^3$ or $10^4$, thus ensuring a reliable alignment between the condition and the original constraint.
\begin{table*}[t]
  \small
  \centering
  \begin{tabular}{ @{\,\,}c @{\,\,}| c | c | c c c c | c}
     \toprule

     Network & Scenario    & Difficulty   & FP3O                                      & HAPPO                                   & MAPPO                                     & IPPO               & Agent \\ \midrule
     % \multirow{6}{*}{$\begin{array}{c}\rotatebox{90}{\text{Shared}} \end{array}$}
     \multirow{6}{*}{$\begin{array}{@{\,}c@{\,}}{\text{Full Param. Sharing}} \end{array}$}
              % & MMM          & Easy       & xxx\tiny{$\pm$xxx}                        & xxx\tiny{$\pm$xxx}                      & xxx\tiny{$\pm$xxx}                            & xxx\tiny{$\pm$xxx} & HE    \\
              & bane vs. bane & Easy       & \textbf{100.0}\tiny{$\pm$1.4}\phantom{0}             & \textbf{100.0}\tiny{$\pm$1.9}\phantom{0}           & \textbf{100.0}\tiny{$\pm$3.5}\phantom{0}               & \textbf{100.0}\tiny{$\pm$2.3}\phantom{0} & HE    \\
              & 2c vs. 64zg   & Hard       & \textbf{100.0}\tiny{$\pm$2.5}\phantom{0}             & \phantom{0}96.9\tiny{$\pm$2.4}\phantom{0}          & \textbf{100.0}\tiny{$\pm$2.5}\phantom{0}               & \phantom{0}96.9\tiny{$\pm$2.3}\phantom{0} & HO    \\
              & 3s5z         & Hard       & \textbf{100.0}\tiny{$\pm$3.7}\phantom{0}               & \phantom{0}87.5\tiny{$\pm$24.7}                   & \phantom{0}\underline{93.8}\tiny{$\pm$3.6}\phantom{0}  & \phantom{0}87.5\tiny{$\pm$34.0} & HE \\
              & 5m vs. 6m     & Hard       & \phantom{0}\textbf{93.8}\tiny{$\pm$6.1}\phantom{0}    & \phantom{0}18.8\tiny{$\pm$12.8}                   & \phantom{0}\underline{81.3}\tiny{$\pm$10.0}            & \phantom{0}62.5\tiny{$\pm$19.1} & HO    \\
              & corridor     & Super Hard & \textbf{100.0}\tiny{$\pm$1.6}\phantom{0}               & \phantom{0}96.9\tiny{$\pm$35.3}                   & \textbf{100.0}\tiny{$\pm$2.9}\phantom{0}               & \phantom{0}96.9\tiny{$\pm$4.3}\phantom{0} & HO    \\
              & 6h vs. 8z     & Super Hard & \phantom{0}\textbf{90.6}\tiny{$\pm$6.6}\phantom{0}     & \phantom{0}53.1\tiny{$\pm$20.2}                  & \phantom{0}\underline{87.5}\tiny{$\pm$9.2}\phantom{0}  & \phantom{0}\underline{87.5}\tiny{$\pm$10.4} & HO    \\
              \midrule
     % \multirow{6}{*}{$\begin{array}{c}\rotatebox{90}{\text{Partially Shared}} \end{array}$}
     \multirow{6}{*}{$\begin{array}{@{\,}c@{\,}}{\text{Partial Param. Sharing}} \end{array}$}
              % & MMM          & Easy       & xxx\tiny{$\pm$xxx}                        & xxx\tiny{$\pm$xxx}                      & xxx\tiny{$\pm$xxx}                            & xxx\tiny{$\pm$xxx} & HE    \\
              & bane vs. bane & Easy       & \textbf{100.0}\tiny{$\pm$0.0}\phantom{0}             & \phantom{0}65.6\tiny{$\pm$26.4}               & \textbf{100.0}\tiny{$\pm$0.0}\phantom{0}              & \textbf{100.0}\tiny{$\pm$0.0}\phantom{0} & HE    \\
              & 2c vs. 64zg   & Hard       & \textbf{100.0}\tiny{$\pm$3.1}\phantom{0}             & \phantom{0}96.9\tiny{$\pm$3.3}\phantom{0}     & \textbf{100.0}\tiny{$\pm$3.1}\phantom{0}              & \textbf{100.0}\tiny{$\pm$4.3}\phantom{0} & HO    \\
              & 3s5z         & Hard       & \phantom{0}\textbf{96.9}\tiny{$\pm$3.9}\phantom{0}    & \phantom{0}\phantom{0}0.0\tiny{$\pm$31.0}     & \phantom{0}\underline{93.8}\tiny{$\pm$4.6}\phantom{0} & \phantom{0}90.6\tiny{$\pm$6.5}\phantom{0} & HE \\
              & 5m vs. 6m     & Hard       & \phantom{0}\textbf{93.8}\tiny{$\pm$4.2}\phantom{0}    & \phantom{0}31.3\tiny{$\pm$15.4}              & \phantom{0}{87.5}\tiny{$\pm$7.7}\phantom{0}            & \phantom{0}\underline{90.6}\tiny{$\pm$11.1} & HO    \\
              & corridor     & Super Hard & \phantom{0}\underline{90.6}\tiny{$\pm$9.2}\phantom{0}  & \phantom{0}78.1\tiny{$\pm$22.7}              & \phantom{0}\textbf{93.8}\tiny{$\pm$10.9}               & \phantom{0}\underline{90.6}\tiny{$\pm$3.6}\phantom{0} & HO    \\
              & 6h vs. 8z     & Super Hard & \phantom{0}\textbf{81.3}\tiny{$\pm$18.3}              & \phantom{0}\phantom{0}9.4\tiny{$\pm$18.8}    & \phantom{0}15.6\tiny{$\pm$21.1}                        & \phantom{0}\underline{46.9}\tiny{$\pm$23.0} & HO    \\
              \midrule
     % \multirow{6}{*}{$\begin{array}{c}\rotatebox{90}{\text{Separated}} \end{array}$}
     \multirow{6}{*}{$\begin{array}{@{\,}c@{\,}}{\text{Non-Param. Sharing}} \end{array}$}
              % & MMM          & Easy       & xxx\tiny{$\pm$xxx}                        & xxx\tiny{$\pm$xxx}                & xxx\tiny{$\pm$xxx}                            & xxx\tiny{$\pm$xxx} & HE    \\
              & bane vs. bane & Easy       & \textbf{100.0}\tiny{$\pm$0.0}\phantom{0}              & \textbf{100.0}\tiny{$\pm$0.0}\phantom{0}             & \textbf{100.0}\tiny{$\pm$0.0}\phantom{0}              & \textbf{100.0}\tiny{$\pm$0.0}\phantom{0} & HE    \\
              & 2c vs. 64zg   & Hard       & \textbf{100.0}\tiny{$\pm$2.8}\phantom{0}             & \textbf{100.0}\tiny{$\pm$1.9}\phantom{0}              & \textbf{100.0}\tiny{$\pm$1.9}\phantom{0}             & \textbf{100.0}\tiny{$\pm$1.3}\phantom{0} & HO    \\
              & 3s5z         & Hard       & \textbf{100.0}\tiny{$\pm$2.5}\phantom{0}              & \phantom{0}90.6\tiny{$\pm$4.8}\phantom{0}              & \phantom{0}93.8\tiny{$\pm$5.3}\phantom{0}           & \phantom{0}\underline{96.9}\tiny{$\pm$8.1}\phantom{0} & HE \\
              & 5m vs. 6m     & Hard       & \phantom{0}\textbf{90.6}\tiny{$\pm$6.4}\phantom{0}   & \phantom{0}\underline{68.8}\tiny{$\pm$16.8}              & \phantom{0}62.5\tiny{$\pm$10.0}                   & \phantom{0}62.5\tiny{$\pm$11.3} & HO    \\
              & corridor     & Super Hard & \phantom{0}\textbf{96.9}\tiny{$\pm$7.2}\phantom{0}     & \phantom{0}\textbf{96.9}\tiny{$\pm$3.0}\phantom{0}    & \phantom{0}93.8\tiny{$\pm$7.2}\phantom{0}            & \phantom{0}93.8\tiny{$\pm$27.2} & HO    \\
              & 6h vs. 8z     & Super Hard & \phantom{0}\textbf{87.5}\tiny{$\pm$15.3}               & \phantom{0}43.8\tiny{$\pm$10.2}                      & \phantom{0}43.8\tiny{$\pm$15.8}                      & \phantom{0}\underline{46.9}\tiny{$\pm$20.5} & HO    \\
              \midrule
  \end{tabular}
  \caption{The median evaluation win rates and standard deviations on SMAC. We bold the best performances and underline the second-best performances. HO and HE denote homogeneous and heterogeneous agents respectively.}
  \label{table_SMAC_results}
\end{table*}
\paragraph{Performances on MAMuJoCo.}
We select six representative tasks within MAMuJoCo that encompass a broad range of robotic control scenarios. We evaluate our algorithm on these tasks using full, partial, and non-parameter sharing and present the results in Table~\ref{table_MAMuJoCo_results}. Notably, we observe that the optimal performance of each algorithm is primarily attained through partial or non-parameter sharing, highlighting the heterogeneity of MAMuJoCo agents and the necessity of partial/non-parameter sharing for such tasks. IPPO fails to learn effective policies in several tasks (e.g., 2-agent Walker and Manyagent Swimmer), whereas MAPPO achieves more stable performances by leveraging the CTDE framework. Nevertheless, both algorithms fell short in handling these complex tasks, even using a more suitable non-parameter sharing scheme. HAPPO, utilizing non-parameter sharing, establishes the strongest baseline on these heterogeneous-agent tasks in MAMuJoCo, as indicated by the underlined values of non-parameter sharing in Table~\ref{table_MAMuJoCo_results}. However, it cannot maintain such advantages under full/partial parameter-sharing settings. This issue is further exacerbated in SMAC tasks, as we will discuss later. Our FP3O consistently outperforms the baselines on all scenarios and parameter-sharing types, demonstrating its advantages in versatility. 
\paragraph{Performances on SMAC.} \label{section_Performances_on_SMAC}
We subsequently evaluate our algorithm on various types of maps in SMAC with full, partial, and non-parameter sharing. Six different maps were chosen, incorporating three levels of difficulty, seven types of units, and both homogeneous-agent and heterogeneous-agent tasks. The results are shown in Table~\ref{table_SMAC_results}. We can notice that MAPPO with full parameter sharing achieves the strongest baseline on SMAC. This shows the homogeneity of SMAC agents and the high learning efficiency of parameter sharing while dealing with such tasks. However, MAPPO cannot maintain its superiority in other scenarios. Moreover, due to the problems of the broken theoretical support and excessive KL divergence, we can see that HAPPO with full/partial parameter sharing will suffer from serious performance degradation in most maps, especially in a large number of agent cases (e.g., HAPPO with partial parameter sharing has only 65.6\% win rate in \emph{easy} map bane vs. bane, and 0.0\% win rate in map 3s5z). Moreover, in the super hard task of 6h vs. 8z, both MAPPO and HAPPO with partial parameter sharing struggle to converge and occasionally reach a 0.0\% win rate on some seeds (see Appendix~\ref{appendix_training_monitoring}), resulting in poor overall performance. This suggests they may be falling into local optimal traps from which they are unable to escape. In contrast, our FP3O algorithm effectively overcomes local optima under all parameter-sharing configurations, consistently delivering outstanding performance. As indicated in Table~\ref{table_SMAC_results}, our FP3O algorithm again achieves excellent performances on all tasks, demonstrating its superior compatibility with all types of parameter sharing.

\section{Related Work}
% \paragraph{Parameter form.}
In MARL, full, partial, and non-parameter sharing are three prevalent network types utilized to address various multi-agent tasks. Full parameter sharing \citep{rashid2018qmix,wang2020qplex,christianos2020shared} has been extensively applied in the scenarios of homogeneous agents or a large number of agents, as it can considerably decrease the number of parameters and enhance learning efficiency. Non-parameter sharing \citep{lowe2017maddpg,kuba2022happo} is instrumental in managing heterogeneous agents by equipping each agent with an individual policy network to make diverse decisions. More recently, partial parameter sharing \citep{christianos2021scaling,kim2023parameter} has been used to mitigate the explosion of the parameter space while still permitting diverse decision-making.

Regarding PPO-based algorithms in cooperative MARL, IPPO \citep{de2020ippo} and MAPPO \citep{yu2021mappo} directly applied the PPO algorithm to each agent in the multi-agent systems with full parameter sharing. Although they can be technically extended to other parameter-sharing configurations, they overlook the interaction among agents and lack the theoretical guarantee. \citet{sun2023trust} demonstrated that enforcing the centralized trust region constraint by bounding independent ratios can ensure monotonic improvement for IPPO and MAPPO. However, the influence of network sharing among agents was not considered. \citet{wu2021coordinated} proposed theoretically-justified CoPPO, but it requires the impractical repetitive error backpropagation of the policy networks when enforcing partial/non-parameter sharing. Recently, HAPPO \citep{kuba2022happo} and A2PO \citep{wang2023order} adopted sequential updates but with the assumption that each agent's update does not impact the policies of others. Unfortunately, this assumption becomes invalid when the non-parameter sharing condition is deactivated, making them not generalizable to other scenarios.

\section{Conclusion}
We propose a versatile multi-agent PPO algorithm for cooperative MARL. 
It is achieved through a novel full-pipeline paradigm, which leverages various decompositions to establish multiple equivalent optimization pipelines. Its versatility relies on the pipelines' interconnections, which allows for a unified outcome across diverse parameter-sharing types. We provide a theoretical foundation and develop a practical FP3O algorithm by several approximations. Empirical evaluations on MAMuJoCo and SMAC demonstrate that FP3O outperforms other strong baselines, exhibiting remarkable compatibility across different parameter-sharing configurations. This inspires us to further explore its application on elaborately designed networks \citep{vaswani2017attention} in the future. Besides, the mechanism of the full-pipeline paradigm also provides valuable insight for improving the versatility of other algorithms, such as SAC-based \cite{haarnoja2018soft} and DDPG-based \cite{lillicrap2015continuous} algorithms.

% \onecolumn
\bibliography{aaai24}

\newpage
\onecolumn
\appendix
\numberwithin{equation}{section}
\numberwithin{figure}{section}
\numberwithin{table}{section}
\numberwithin{algorithm}{section}
\numberwithin{listing}{section}
\numberwithin{lemma}{section}
\numberwithin{theorem}{section}
% \chapter{附录}
% \addcontentsline{toc}{chapter}{附录}
\addcontentsline{toc}{section}{Appendix}

% Start the appendix part
\part{Appendix}

% Insert the appendix TOC
\parttoc

% \tableofcontents
\section{Limitations}
Similar to \cite{yu2021mappo,wu2021coordinated,kuba2022happo,wang2023order}, our method is also established on the cooperative multi-agent Markov game, where all agents have a shared reward and aim to maximize a shared objective: $\mathcal{J}(\bm{\pi})$. Therefore, its direct application to non-cooperative MARL settings may not yield optimal results, even though the non-cooperative task is an important setting. The main challenge in generalizing our framework to non-cooperative games is how to establish a private reward objective function for each agent. MADDPG~\cite{lowe2017maddpg} is working in both cooperative and competitive scenarios, and the latter is known to be a much harder problem. As the same CTDE-based algorithm, MADDPG can provide potential insight into this problem and we will leave it for future work. Other the other hand, the monotonic improvement guarantee of the full-pipeline paradigm holds given that Assumption~\ref{assumption_policy} is satisfied, which is typically valid in the CTDE framework or fully decentralized framework. However, in other frameworks (e.g., auto-regressive policy representation $\bm{\pi}(\bm{a}|s)=\prod_{m=1}^{n}\pi^{i_m}(a^{i_m}|s, a^{i_{1:m-1}})$), its monotonic improvement guarantee still need further examination.

\section{Reproducibility}\label{appendix_reproducibility}
In Section~\ref{section_fp3o}, we have developed full-pipeline optimization to a programming-friendly FP3O algorithm. The main implementation differences between MAPPO, HAPPO, and FP3O lie in the update process of the policy networks. We provide the pseudo code of three algorithms with parameter sharing in Listings~\ref{listing_mappo}, \ref{listing_happo}, and \ref{listing_fp3o}, which contains the main step of the update of policy networks. 
% We have provided the source code in the supplementary materials.
% Please note that the demo code is not identical to the actual source code, as they are designed for intuitive demonstration purposes. 
% For the complete source code, which is anonymously provided, you can access it at \url{https://github.com/an-anony/fp3o-neurips23}. 
Moreover, we specify all training details (e.g., hyperparameter settings), repetitive experiments (e.g., the number of random seeds), and the type of computing infrastructure (e.g. the type of GPUs) in Appendix~\ref{appendix_hyperparameters}.
\begin{listing}[H]
\caption{Actor update of MAPPO}
\label{listing_mappo}%
\begin{lstlisting}[]
    def mappo_update(self, samples):
       # Update all agents. The dimension of actions_batch is [episode_length * num_rollout_threads * num_agent, 1]
       share_obs_batch, obs_batch, rnn_states_batch, rnn_states_critic_batch, \
       actions_batch, value_preds_batch, return_batch, masks_batch, \
       active_masks_batch, old_action_log_probs_batch, adv_targ, \
       available_actions_batch = samples
       
       # Regenerate in a single forward process for all steps
       values, action_log_probs, dist_entropy = \
             self.policy.evaluate_actions(share_obs_batch, obs_batch, 
                                        rnn_states_batch, rnn_states_critic_batch, 
                                        actions_batch, active_masks_batch, 
                                        available_actions_batch, masks_batch)
    
       # Calculate the main PPO-clip objective
       imp_weights = torch.prod(torch.exp(action_log_probs - old_action_log_probs_batch),dim=-1,keepdim=True)
    
       surr1 = imp_weights * adv_targ
       surr2 = torch.clamp(imp_weights, 1.0 - self.clip_param, 1.0 + self.clip_param) * adv_targ
       
       policy_loss = (-torch.sum(torch.min(surr1, surr2), dim=-1, keepdim=True) * active_masks_batch).sum() / active_masks_batch.sum()
    
       # Actor update
       self.policy.actor_optimizer.zero_grad()
       (policy_loss - dist_entropy * self.entropy_coef).backward()
       self.policy.actor_optimizer.step()
\end{lstlisting}
\end{listing}
\newpage

\begin{listing}[H]
\caption{Actor update of HAPPO}
\label{listing_happo}%
\begin{lstlisting}[]
    def happo_update(self, samples):
       # Update one agent. The dimension of actions_batch is [episode_length * num_rollout_threads * 1, 1]
       share_obs_batch, obs_batch, rnn_states_batch, rnn_states_critic_batch, \
       actions_batch, value_preds_batch, return_batch, masks_batch, \
       active_masks_batch, old_action_log_probs_batch, adv_targ, \
       available_actions_batch, factor_1_to_m_batch = samples
       
       # Regenerate in a single forward process for the steps of one agent
       values, action_log_probs, dist_entropy = \
             self.policy.evaluate_actions(share_obs_batch, obs_batch, 
                                        rnn_states_batch, rnn_states_critic_batch, 
                                        actions_batch, active_masks_batch, 
                                        available_actions_batch, masks_batch)
    
       # Calculate the main PPO-clip objective
       imp_weights = torch.prod(torch.exp(action_log_probs - old_action_log_probs_batch),dim=-1,keepdim=True)
    
       surr1 = imp_weights * adv_targ * factor_1_to_m_batch
       surr2 = torch.clamp(imp_weights, 1.0 - self.clip_param, 1.0 + self.clip_param) * adv_targ * factor_1_to_m_batch
       
       policy_loss = (-torch.sum(torch.min(surr1, surr2), dim=-1, keepdim=True) * active_masks_batch).sum() / active_masks_batch.sum()
    
       # Actor update
       self.policy.actor_optimizer.zero_grad()
       (policy_loss - dist_entropy * self.entropy_coef).backward()
       self.policy.actor_optimizer.step()
\end{lstlisting}
\end{listing}

\begin{listing}[H]
\caption{Actor update of FP3O}
\label{listing_fp3o}%
    \begin{lstlisting}[]
    def fp3o_update(self, samples):
       # Check the condition
       if (self.is_dependent_step and self.is_satisfy_st()) or not self.is_dependent_step:
          # Update all agents. The dimension of actions_batch is [episode_length * num_rollout_threads * num_agent, 1]
          share_obs_batch, obs_batch, rnn_states_batch, rnn_states_critic_batch, \
          actions_batch, value_preds_batch, return_batch, masks_batch, \
          active_masks_batch, old_action_log_probs_batch, adv_targ, \
          available_actions_batch, factor_j_batch, factor_ij_batch = samples
          
          # Regenerate in a single forward process for all steps
          values, action_log_probs, dist_entropy = \
                self.policy.evaluate_actions(share_obs_batch, obs_batch, 
                                           rnn_states_batch, rnn_states_critic_batch, 
                                           actions_batch, active_masks_batch, 
                                           available_actions_batch, masks_batch)
    
          # Calculate the main PPO-clip objective
          imp_weights = torch.prod(torch.exp(action_log_probs - old_action_log_probs_batch),dim=-1,keepdim=True)
    
          surr1 = (imp_weights * factor_ij_batch - 1) * factor_j_batch * adv_targ
          surr2 = (torch.clamp(imp_weights, 1.0 - self.clip_param, 1.0 + self.clip_param) * factor_ij_batch - 1) * factor_j_batch * adv_targ
          
          policy_loss = (-torch.sum(torch.min(surr1, surr2), dim=-1, keepdim=True) * active_masks_batch).sum() / active_masks_batch.sum()
    
          # Actor update
          self.policy.actor_optimizer.zero_grad()
          (policy_loss - dist_entropy * self.entropy_coef).backward()
          self.policy.actor_optimizer.step()
\end{lstlisting}
\end{listing}

\newpage
\section{Potential Limitation of Sequential Update Scheme}\label{appendix_sequential_update_scheme_bad}
When the policy networks of agents in HAPPO or A2PO are interconnected (for instance, through full parameter sharing, partial parameter sharing, or other associative designs), a significant issue can emerge: each agent's update can lead to changes across the policy networks of all other agents. As the sequential update scheme advances, these changes accumulate, resulting in excessive KL divergence. To demonstrate this, we train HAPPO with full parameter sharing on the 3s5z map of the StarCraftII Multi-Agent Challenge. Upon completion of each agent's update round during a policy iteration, we record the distribution of KL divergence between its updated policy and the old one, as shown in Figure \ref{fig_happo_kl}. It can be observed that the KL divergence between the updated and old policy increases as the sequential update scheme progresses, particularly for agents 6 and 7. This excessive KL divergence could negatively impact performance in cases of full or partial parameter sharing.
\begin{figure}[H]
   \centering
   \includegraphics[width=0.7\linewidth]{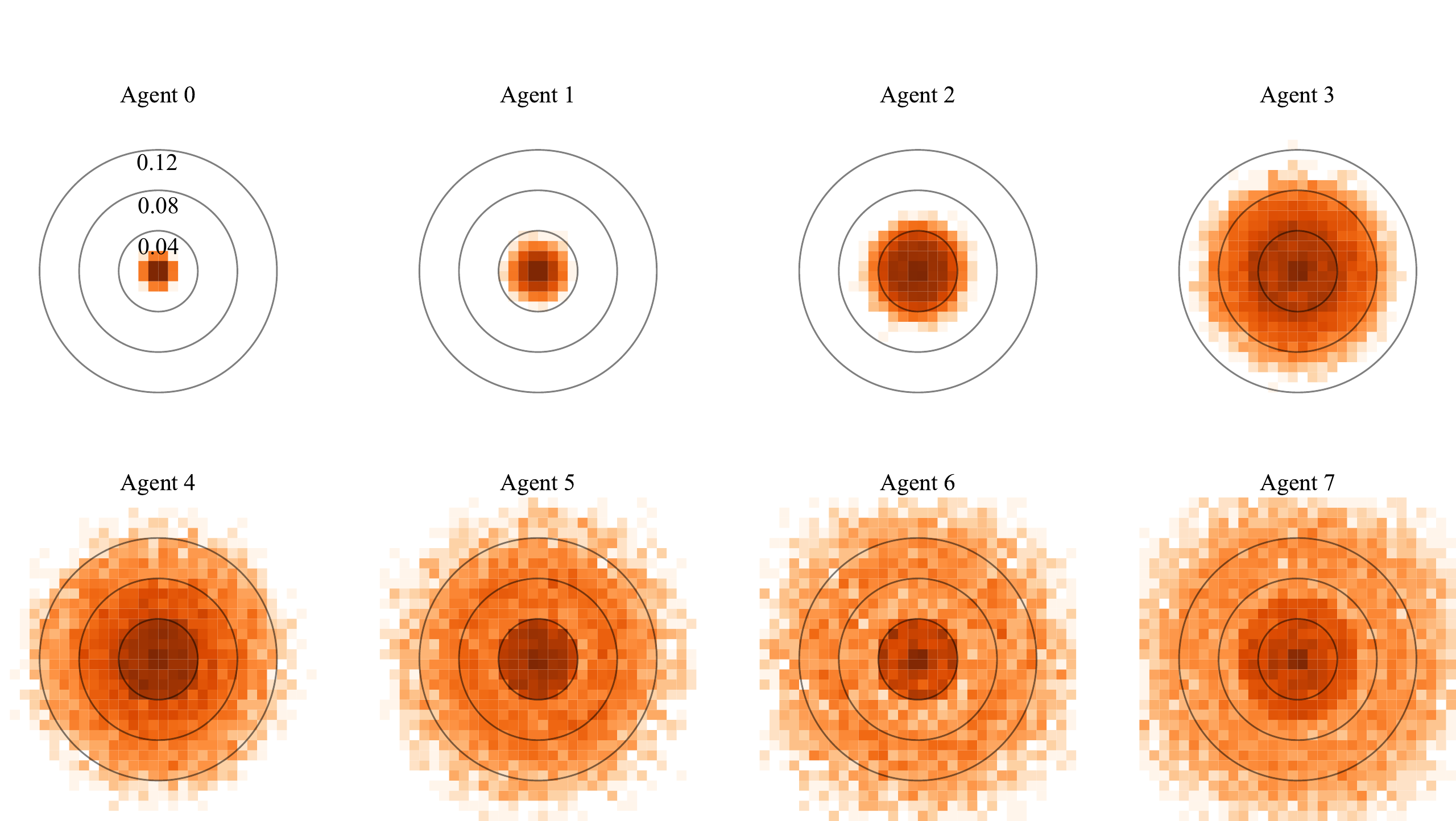}
   %\framebox[4.0in]{$\;$}
   \caption{The distribution of KL divergence between the updated policy and the old policy for each agent in HAPPO. Darker colors indicate a higher data distribution. The center represents that KL divergence is equal to 0, and the outward direction represents that KL divergence gets greater. The locations where KL divergence equals 0.04, 0.08, and 0.12 are marked by black circles}
   \label{fig_happo_kl}
\end{figure}

\section{Proofs, Diagram and Analysis}
\subsection{Proof of Equation~(\ref{equation_multibound})}\label{appendix_proof_equation5}
We separate one expected agent's contribution independently of other agents from the team through Equation~(\ref{equation_multibound}). Now we give its derivation.
\begin{lemma}
   \label{lemma_fp3o_decomposition}
   Let $\varnothing$ denote the empty set. In the cooperative multi-agent tasks, given the complete set $i_{1:n}=\mathcal{N}$ of agents and $\forall i_{p} \in i_{1:n}$, the following equation holds:
   \begin{equation}
      \sum_{m=1}^n A_{\bm{\pi}}^{i_{m}}(s, \bm{a}^{i_{1:m-1}}, \bm{a}^{i_{m}})=A_{\bm{\pi}}^{i_p}(s,a^{\varnothing },a^{i_p})+A_{\bm{\pi}}^{-i_{p}}(s,a^{i_p},\bm{a}^{-i_{p}}). \nonumber
   \end{equation}
\end{lemma}
\begin{proof}
   Based on Equations (\ref{equation_state_action_value_function}), (\ref{equation_advantage_function}) and (\ref{equation_multi_agent_advantage_decomposition}), we have
   \begin{align}
      &\sum_{m=1}^n A_{\bm{\pi}}^{i_{m}}(s, \bm{a}^{i_{1:m-1}}, \bm{a}^{i_{m}})=A_{\bm{\pi}}(s,\bm{a})\nonumber   \\
      =&Q_{\bm{\pi}}(s,\bm{a})-V_{\bm{\pi}}(s)=Q_{\bm{\pi}}^{i_{p},-i_p}(s,\bm{a}^{i_p,-i_p})-V_{\bm{\pi}}(s)\nonumber                                 \\ 
      =&\left(Q_{\bm{\pi}}^{i_p}(s,\bm{a}^{i_{p}})-V_{\bm{\pi}}(s)\right)
      -\left(Q_{\bm{\pi}}^{i_{p},-i_p}(s,\bm{a}^{i_p,-i_p})-Q_{\bm{\pi}}^{i_p}(s,\bm{a}^{i_{p}})\right)\nonumber   \\
      =&A_{\bm{\pi}}^{i_p}(s,a^{\varnothing },a^{i_p})+A_{\bm{\pi}}^{-i_{p}}(s,a^{i_p},\bm{a}^{-i_{p}}).\nonumber 
  \end{align}
\end{proof}

With Equation~(\ref{equation_happo_bound}), we have
\begin{align}
   \mathcal{J}(\tilde{\bm{\pi}})\geq&\mathcal{J}(\bm{\pi})+\sum_{m=1}^n{M_{\bm{\pi}}^{i_{m}}(\tilde{\bm{\pi}}^{i_{1:m-1}}, \tilde{\pi}^{i_m})} \nonumber \\
   =&\mathcal{J}(\bm{\pi})+\sum_{m=1}^n\left[\mathbb{E} _{s\sim \rho_{\bm{\pi}},\bm{a}\sim \tilde{\bm{\pi}}}\left[A_{\bm{\pi}}^{i_{m}}(s, \bm{a}^{i_{1:m-1}}, \bm{a}^{i_{m}})\right]-\mathcal{C}D_{\rm KL}^{\max}({\pi}^{i_m}, {\tilde{\pi}}^{i_m})\right] \nonumber \\
   =&\mathcal{J}(\bm{\pi})+\mathbb{E} _{s\sim \rho_{\bm{\pi}},\bm{a}\sim \tilde{\bm{\pi}}}\left[\sum_{m=1}^n A_{\bm{\pi}}^{i_{m}}(s, \bm{a}^{i_{1:m-1}}, \bm{a}^{i_{m}})\right]-\sum_{m=1}^n \mathcal{C}D_{\rm KL}^{\max}({\pi}^{i_m}, {\tilde{\pi}}^{i_m}), \nonumber \\
    &\text{which, by introducing Lemma~\ref{lemma_fp3o_decomposition}, equals to}     \nonumber  \\                                       
    =&\mathcal{J}(\bm{\pi})+\mathbb{E}_{s\sim\rho_{\bm{\pi}}, \bm{a}\sim\bm{\tilde{\pi}}}\left[A_{\bm{\pi}}^{i_p}(s,a^{i_p})+A_{\bm{\pi}}^{-i_{p}}(s,a^{i_p},\bm{a}^{-i_{p}})\right]-\sum_{m=1}^n \mathcal{C}D_{\rm KL}^{\max}({\pi}^{i_m}, {\tilde{\pi}}^{i_m}) \nonumber \\
   = &\mathcal{J}(\bm{\pi})    +\mathbb{E}_{s\sim\rho_{\bm{\pi}}, \bm{a}\sim\bm{\tilde{\pi}}}[A_{\bm{\pi}}^{i_p}(s,a^{i_p})]-\mathcal{C}{D_{\rm KL}^{\max}{(\pi^{i_p}, \tilde{\pi}^{i_p})}} \nonumber                                                                \\
     & \qquad +\mathbb{E}_{s\sim\rho_{\bm{\pi}}, \bm{a}\sim\bm{\tilde{\pi}}}[A_{\bm{\pi}}^{-i_{p}}(s,a^{i_p},\bm{a}^{-i_{p}})]-\sum_{i\in -i_p}\mathcal{C}{D_{\rm KL}^{\max}{(\pi^{i}, \tilde{\pi}^{i})}}, \nonumber \\
     &\text{and by Definition \ref{definition_surrogate}, this is}\nonumber \\
   = &\mathcal{J}(\bm{\pi})+M_{\bm{\pi}}^{i_p}(\tilde{\pi}^\varnothing,{\tilde{\pi}}^{i_p})+M_{\bm{\pi}}^{-i_p}(\tilde{\pi}^{i_p}, \bm{\tilde{\pi}}^{-i_p}). \nonumber
\end{align}

\subsection{Proof of Equations~(\ref{equation_impsampling1}) and (\ref{equation_impsampling2})}\label{appendix_proof_equation67}
In the single pipeline optimization, we first execute $\mathop{\rm maximize}\nolimits_{\tilde{\pi}^{i_p}}M_{\bm{\pi}}^{i_p}(\tilde{\pi}^\varnothing,{\tilde{\pi}}^{i_p})$. We then execute $\mathop{\rm maximize}\nolimits_{\tilde{\bm{\pi}}^{-i_p}}M_{\bm{\pi}}^{-i_p}(\tilde{\pi}^{i_p}, \tilde{\bm{\pi}}^{-i_p})$ based on the result $\tilde{\pi}^{i_p}$ of first step. This is similar to the sequential update scheme of HAPPO. However, the dynamic action input of $A_{\pi}^{i_p}(s,a^{\varnothing }, a^{i_p})$ and $ A_{\bm{\pi}}^{-i_p}(s,a^{i_p},\bm{a}^{-i_p})$ leads to high cost for agents to maintain the advantage functions. Thus, we introduce the importance-sampling objectives, as shown in Equations~(\ref{equation_impsampling1}) and (\ref{equation_impsampling2}), to address this issue. The proof is as follows.

For the term $\mathbb{E}_{s\sim \rho_{\bm{\pi}},\bm{a}\sim \tilde{\bm{\pi}}} \left[A_{\bm{\pi}}^{i_{p}}(s, a^{\varnothing}, {a}^{i_{p}})\right] $ in surrogate objective $M_{\bm{\pi}}^{i_p}(\tilde{\pi}^\varnothing,{\tilde{\pi}}^{i_p})$, we adopt the importance sampling as follows.
\begin{align}
    & \mathbb{E}_{s\sim \rho_{\bm{\pi}},\bm{a}\sim \tilde{\bm{\pi}}} \left[A_{\bm{\pi}}^{i_{p}}(s, a^{\varnothing}, {a}^{i_{p}})\right] \nonumber                                                                      \\
    &\text{which, by Equations (\ref{equation_state_action_value_function}), (\ref{equation_advantage_function}), equals} \nonumber     \\
    =& \mathbb{E}_{s\sim \rho_{\bm{\pi}},\bm{a}\sim \tilde{\bm{\pi}}} \left[Q_{\bm{\pi}}^{i_{p}}(s, {a}^{i_{p}})-V_{\bm{\pi}}(s)\right] = \mathbb{E}_{s\sim \rho_{\bm{\pi}},\bm{a}\sim \tilde{\bm{\pi}}} \left[\mathbb{E}_{\bm{a}^{-i_p}\sim \bm{\pi}^{-i_p}}[Q_{\bm{\pi}}(s, \bm{a})]-V_{\bm{\pi}}(s)\right] \nonumber \\
     =& \mathbb{E}_{s\sim \rho_{\bm{\pi}},a^{i_p}\sim \tilde{\pi}^{i_p},\bm{a}^{-i_p}\sim \bm{\pi}^{-i_p}} \left[Q_{\bm{\pi}}(s, \bm{a})-V_{\bm{\pi}}(s)\right] \nonumber \\
     =&\mathbb{E}_{s\sim \rho_{\bm{\pi}},a^{i_p}\sim \tilde{\pi}^{i_p},\bm{a}^{-i_p}\sim \bm{\pi}^{-i_p}} \left[A_{\bm{\pi}}(s, \bm{a})\right], \nonumber \\
    &\text{then with importance sampling, this is} \nonumber     \\
     =& \mathbb{E}_{s\sim \rho_{\bm{\pi}},\bm{a}\sim \bm{\pi}} \left[\frac{\tilde{\pi}^{i_p}(a^{i_p}|s)}{\pi^{i_p}(a^{i_p}|s)}A_{\bm{\pi}}(s,\bm{a})\right]. \label{equation_appendix_sampling1}
\end{align}
Hence, at $k$-th iteration, given the old policy $\pi_{k}^{i_p}$, the optimization for $M_{\bm{\pi}_k}^{i_p}$ of agent $\pi^{i_p}$ is 
\begin{align}
   \label{equation_appendix_single_step1}
   \pi_{k+1}^{i_p} =& \arg\max\limits_{\pi^{i_p}}M_{\bm{\pi}_k}^{i_p}(\pi_{k+1}^\varnothing,{\pi}^{i_p})\nonumber \\
   =&\arg\max\limits_{\pi^{i_p}}\left[\mathbb{E}_{s\sim \rho_{\bm{\pi}_k},\bm{a}\sim \bm{\pi}_k} \left[\frac{{\pi}^{i_p}(a^{i_p}|s)}{\pi_k^{i_p}(a^{i_p}|s)}A_{\bm{\pi}_k}(s,\bm{a})\right]-\mathcal{C}D_{\rm KL}^{\max}({\pi}^{i_p}_{k}, {\pi}^{i_p})\right]\nonumber \\
   =&\arg\max\limits_{\pi^{i_p}}\left[\mathbb{E}_{s\sim \rho_{\bm{\pi}_k},\bm{a}\sim \bm{\pi}_k} \left[r(\pi^{i_p})\cdot A_{\bm{\pi}_k}(s,\bm{a})\right]-\mathcal{C}D_{\rm KL}^{\max}({\pi}^{i_p}_{k}, {\pi}^{i_p})\right]. 
\end{align}  

For the advantage function term $\mathbb{E}_{s\sim \rho_{\bm{\pi}}, \bm{a}\sim \bm{\tilde{\pi}}} \left[A_{\bm{\pi}}^{-i_p}(s, a^{i_p}, \bm{a}^{-i_p})\right]$ in surrogate objective $M_{\bm{\pi}}^{-i_p}(\tilde{\pi}^{i_p}, \bm{\tilde{\pi}}^{-i_p})$, we have
% For the advantage function in $M_{\bm{\pi}}^{-i_p}(\tilde{\pi}^{i_p}, \bm{\tilde{\pi}}^{-i_p})$
\begin{align}
    & \mathbb{E}_{s\sim \rho_{\bm{\pi}},\bm{a}\sim \bm{\tilde{\pi}}} \left[A_{\bm{\pi}}^{-i_p}(s, a^{i_p}, \bm{a}^{-i_p})\right]\nonumber                                                                                                                                                                                          \\
    &\text{which, by Equations (\ref{equation_state_action_value_function}), (\ref{equation_advantage_function}), equals} \nonumber     \\
     =&\mathbb{E}_{s\sim \rho_{\bm{\pi}},\bm{a}\sim \bm{\tilde{\pi}}} \left[Q_{\bm{\pi}}^{i_p,-i_p}(s, \bm{a}^{i_p,-i_p})-Q_{\bm{\pi}}^{i_p}(s, \bm{a}^{i_p})\right]\nonumber                                                                                                                                                      \\
     =&\mathbb{E}_{s\sim \rho_{\bm{\pi}},\bm{a}\sim \bm{\tilde{\pi}}} \left[Q_{\bm{\pi}}(s, \bm{a})-V_{\bm{\pi}}(s)\right]-\mathbb{E}_{\bm{a}\sim \bm{\tilde{\pi}}}\left[Q_{\bm{\pi}}^{i_p}(s, \bm{a}^{i_p})-V_{\bm{\pi}}(s)\right]\nonumber                                                                                       \\
     =&\mathbb{E}_{s\sim \rho_{\bm{\pi}},\bm{a}\sim \bm{\tilde{\pi}}} \left[A_{\bm{\pi}}(s, \bm{a})\right]-\mathbb{E}_{\bm{a}\sim \bm{\tilde{\pi}}}\left[\mathbb{E}_{\bm{a}^{-i_p}\sim \bm{\pi}^{-i_p}}[Q_{\bm{\pi}}(s, \bm{a})]-V_{\bm{\pi}}(s)\right]\nonumber                                                                   \\
     =&\mathbb{E}_{s\sim \rho_{\bm{\pi}},\bm{a}\sim \bm{\tilde{\pi}}} \left[A_{\bm{\pi}}(s, \bm{a})\right]-\mathbb{E}_{{a}^{i_p}\sim \tilde{\pi}^{i_p},\bm{a}^{-i_p}\sim \bm{\pi}^{-i_p}}\left[Q_{\bm{\pi}}(s, \bm{a})-V_{\bm{\pi}}(s)\right]\nonumber                                                                             \\
     =&\mathbb{E}_{s\sim \rho_{\bm{\pi}},\bm{a}\sim \bm{\tilde{\pi}}} \left[A_{\bm{\pi}}(s, \bm{a})\right]-\mathbb{E}_{{a}^{i_p}\sim \tilde{\pi}^{i_p},\bm{a}^{-i_p}\sim \bm{\pi}^{-i_p}}\left[A_{\bm{\pi}}(s, \bm{a})\right]\nonumber                                                                                             \\
     =&\mathbb{E}_{s\sim \rho_{\bm{\pi}},\bm{a}\sim \bm{\pi}} \left[\frac{\bm{\tilde{\pi}}(\bm{a}|s)}{\bm{\pi}(\bm{a}|s)}A_{\bm{\pi}}(s, \bm{a})\right] \nonumber 
   -\mathbb{E}_{s\sim \rho_{\bm{\pi}},\bm{a}\sim \bm{\pi}} \left[\frac{\tilde{\pi}^{i_p}(a^{i_p}|s)}{\pi^{i_p}(a^{i_p}|s)}A_{\bm{\pi}}(s, \bm{a})\right]                                                                                                                                                                 \\
     =&\mathbb{E}_{s\sim \rho_{\bm{\pi}},\bm{a}\sim \bm{\pi}} \left[\left(\frac{\bm{\tilde{\pi}}^{-i_p}(\bm{a}^{-i_p}|s)}{\bm{\pi}^{-i_p}(\bm{a}^{-i_p}|s)}-1\right)\cdot \frac{\tilde{\pi}^{i_p}(a^{i_p}|s)}{\pi^{i_p}(a^{i_p}|s)}\cdot A_{\bm{\pi}}(s,\bm{a})\right].\label{equation_appendix_sampling2}
\end{align}
Hence, at $k$-th iteration, given the old policy $\bm{\pi}_{k}^{-i_p}$ and the new policy $\pi^{i_p}_{k+1}$ generated by Equation~(\ref{equation_appendix_single_step1}), the optimization for $M_{\bm{\pi}_k}^{-i_p}$ of agents $\bm{\pi}^{-i_p}$ is 
\begin{align}
   \bm{\pi}_{k+1}^{-i_p} &= \arg\max\limits_{\bm{\pi}^{-i_p}} M_{\bm{\pi}_k}^{-i_p}(\pi_{k+1}^{i_p}, \bm{\pi}^{-i_p}) \nonumber \\
   &=\arg\max\limits_{\bm{\pi}^{-i_p}}\left[\mathbb{E}_{s\sim \rho_{\bm{\pi}_k},\bm{a}\sim \bm{\pi}_k} \left[\left(\frac{\bm{{\pi}}^{-i_p}(\bm{a}^{-i_p}|s)}{\bm{\pi}_k^{-i_p}(\bm{a}^{-i_p}|s)} - 1\right)\cdot \frac{{\pi}_{k+1}^{i_p}(a^{i_p}|s)}{\pi_k^{i_p}(a^{i_p}|s)} \cdot A_{\bm{\pi}_k}(s,\bm{a})\right] -\sum_{i \in -i_p} \mathcal{C}D_{\rm KL}^{\max}({\pi}_k^{i}, \pi^{i})\right] \nonumber \\
   &=\arg\max\limits_{\bm{\pi}^{-i_p}}\left[\mathbb{E}_{s\sim \rho_{\bm{\pi}_k},\bm{a}\sim \bm{\pi}_k} \big[\big(r(\bm{\pi}^{-i_p}) - 1\big)\cdot r(\pi_{k+1}^{i_p}) \cdot A_{\bm{\pi}_k}(s,\bm{a})\big] -\sum\nolimits_{i \in -i_p} \mathcal{C}D_{\rm KL}^{\max}({\pi}_k^{i}, \pi^{i})\right]. \nonumber
\end{align}
End of the proof. 

\subsection{Proof of Lemma~\ref{lemma_shared_lower_bound}}\label{appendix_shared_lower_bound}
\textbf{Lemma~\ref{lemma_shared_lower_bound}} (Shared Lower Bound)\textbf{.} \emph{Let $\mathcal{L}_{\bm{\pi}}(\tilde{\pi}^{i_p},\tilde{\pi}^{j_p})=\mathcal{J}(\bm{\pi})+M_{\bm{\pi}}^{i_p}(\tilde{\pi}^\varnothing,{\tilde{\pi}}^{i_p})+M_{\bm{\pi}}^{-i_p}(\tilde{\bm{\pi}}^{-j_p}, \tilde{\pi}^{j_p})$ be the lower bound of pipeline $p$. For all pipelines $p=1,...,n$, their lower bounds satisfies
\begin{equation}
   \mathcal{L}_{\bm{\pi}}(\tilde{\pi}^{i_1},\tilde{\pi}^{j_1}) = ,\dots, = \mathcal{L}_{\bm{\pi}}(\tilde{\pi}^{i_n},\tilde{\pi}^{j_n}).\nonumber
\end{equation}}

\begin{proof}
Considering an \emph{arbitrary} pipeline $p$, its lower bound $\mathcal{L}_{\bm{\pi}}(\tilde{\pi}^{i_p},\tilde{\pi}^{j_p})$ satisfies
\begin{align}
   &\mathcal{L}_{\bm{\pi}}(\tilde{\pi}^{i_p},\tilde{\pi}^{j_p})=\mathcal{J}(\bm{\pi})+M_{\bm{\pi}}^{i_p}(\tilde{\pi}^\varnothing,{\tilde{\pi}}^{i_p})+M_{\bm{\pi}}^{-i_p}(\tilde{\bm{\pi}}^{-j_p}, \tilde{\pi}^{j_p}) \nonumber \\
   =&\mathcal{J}(\bm{\pi})+\mathbb{E}_{s\sim\rho_{\bm{\pi}}, \bm{a}\sim\bm{\tilde{\pi}}}\left[A_{\bm{\pi}}^{i_p}(s,a^\varnothing,a^{i_p})\right]+\mathbb{E}_{s\sim\rho_{\bm{\pi}}, \bm{a}\sim\bm{\tilde{\pi}}}\left[A_{\bm{\pi}}^{-i_{p}}(s,a^{i_p},\bm{a}^{-i_{p}})\right]-\sum_{i=1}^n \mathcal{C}D_{\rm KL}^{\max}({\pi}^{i}, {\tilde{\pi}}^{i}), \nonumber \\
   & \text{which, by Equations~(\ref{equation_appendix_sampling1}) and (\ref{equation_appendix_sampling2}), equals}\nonumber \\
    =& \mathcal{J}(\bm{\pi}) + \mathbb{E}_{s\sim\rho_{\bm{\pi}}, \bm{a}\sim\bm{{\pi}}}\left[\frac{\tilde{\pi}^{i_p}(a^{i_p}|s)}{\pi^{i_p}(a^{i_p}|s)}A_{\bm{\pi}}(s,\bm{a})\right]  \nonumber \\
    & \qquad \ +\mathbb{E}_{s\sim \rho_{\bm{\pi}},\bm{a}\sim \bm{\pi}} \left[\left(\frac{\tilde{\bm{\pi}}^{-i_p}(\bm{a}^{-i_p}|s)}{\bm{\pi}^{-i_p}(\bm{a}^{-i_p}|s)}-1\right)\cdot \frac{\tilde{\pi}^{i_p}(a^{i_p}|s)}{\pi^{i_p}(a^{i_p}|s)}\cdot A_{\bm{\pi}}(s,\bm{a})\right]-\sum_{i=1}^n{\mathcal{C}D_{\rm KL}^{\max}{(\pi^{i}, \tilde{\pi}^{i})}}\nonumber\\
    =&\mathcal{J}(\bm{\pi}) + \mathbb{E}_{s\sim\rho_{\bm{\pi}}, \bm{a}\sim\bm{{\pi}}}\left[\prod_{i=1}^{n}\frac{\tilde{\pi}^{i}(a^{i}|s)}{\pi^{i}(a^{i}|s)}A_{\bm{\pi}}(s,\bm{a})\right]-\sum_{i=1}^n{\mathcal{C} D_{\rm KL}^{\max}{(\pi^{i}, \tilde{\pi}^{i})}}.\label{equation_appendix_shared_lower_bound}
\end{align}
Note that Equation~(\ref{equation_appendix_shared_lower_bound}) does not depend on the variable $p$. Thus, by generalizing the above equation to cover the cases of $p=1,\dots,n$, we can obtain a series of equalities: $\mathcal{L}_{\bm{\pi}}(\tilde{\pi}^{i_1},\tilde{\pi}^{j_1})=,\dots,=\mathcal{L}_{\bm{\pi}}(\tilde{\pi}^{i_n},\tilde{\pi}^{j_n})$. To capture this unified perspective, we introduce a shared lower bound $\mathcal{L}^{S}_{\bm{\pi}}(\tilde{\pi}^{1},\dots,\tilde{\pi}^{n})$ that is equal to Equation~(\ref{equation_appendix_shared_lower_bound}) for all pipelines.
\end{proof}

\subsection{Proof of Lemma~\ref{lemma_KL_expectation}}\label{appendix_KL_expectation}
\textbf{Lemma~\ref{lemma_KL_expectation}.} \emph{At $k$-the policy iteration, the intermediate policies $\{\pi_{k+0.5}^{i_p}\}_{p=1}^n$ and the new policies $\{\pi_{k+1}^{j_p}\}_{p=1}^n$ are produced by Equations~(\ref{equation_full_step1}) and (\ref{equation_full_step2}) respectively. Regardless of whether parameter sharing is used or not, the following inequality holds:
\begin{equation}
 \sum_{i=1}^n \mathcal{C}D_{\rm KL}^{\max}({\pi}^{i}_{k}, {\pi}_{k+1}^{i}) \leq \sum_{i=1}^n \mathbb{E}_{s,\bm{a}} \Big[\big(r(\bm{\pi}_{k+1})-r(\pi_{k+0.5}^{i})\cdot r(\bm{\pi}_{k+1}^{-i})+r(\pi_{k+0.5}^{i})-1\big)\cdot A^{i} \Big].\nonumber
\end{equation}}

% \begin{lemma}[3.2]
%    At $k$-the policy iteration, given the intermediate policies $\{\pi_{k+0.5}^{i_p}\}_{p=1}^n$ produced by \cref{equation_full_step1} and the new policies $\{\pi_{k+1}^{j_p}\}_{p=1}^n$ produced by \cref{equation_full_step2}, the following inequalities hold across all pipelines, $p=1,\dots,n$:
%    \begin{align}
%       \bigg\{\mathcal{C}D_{\rm KL}^{\max}({\pi}^{j_p}_{k}, {\pi}_{k+1}^{j_p}) \leq \mathbb{E}_{s,\bm{a}} \big[ (r(\bm{\pi}_{k+1})-r(\pi_{k+0.5}^{j_p})\cdot r(\bm{\pi}_{k+1}^{-j_p})+r(\pi_{k+0.5}^{j_p}))A^{j_p} \big]\bigg\}_{p=1}^{n}.\nonumber
%    \end{align}
% \end{lemma}
\begin{proof}
\textbf{Non-parameter sharing case:}

At the independent step of $k$-th iteration, we perform the maximization of the modified objective $\widetilde{M}_{\bm{\pi}_k}^{i_p}(\pi_{k+1}^\varnothing,\pi^{i_p})$ w.r.t $\pi^{i_p}$ of all pipelines, as shown in Equation~(\ref{equation_full_step1}). For an arbitrary pipeline $p$, this step leads to
\begin{align}
& \mathbb{E}_{s,\bm{a}} \left[r(\pi_{k+0.5}^{i_p})\cdot A^{i_p}\right]-\mathcal{C}D_{\rm KL}^{\max}\left({\pi}^{i_p}_{k}, {\pi}_{k+0.5}^{i_p}\right)\geq \mathbb{E}_{s,\bm{a}} \left[r(\pi^{i_p}_{k})\cdot A^{i_p}\right]-\mathcal{C}D_{\rm KL}^{\max}({\pi}^{i_p}_{k}, {\pi}_k^{i_p}) \nonumber\\
\Leftrightarrow \quad & \mathbb{E}_{s\sim \rho_{\bm{\pi}_k},\bm{a}\sim \bm{\pi}_k} \left[r(\pi_{k+0.5}^{i_p})\cdot A^{i_p}\right]-\mathbb{E}_{s\sim \rho_{\bm{\pi}_k},\bm{a}\sim \bm{\pi}_k}\left[r(\pi_{k}^{i_p})\cdot A^{i_p}\right]\geq \mathcal{C}D_{\rm KL}^{\max}({\pi}^{i_p}_{k}, {\pi}_{k+0.5}^{i_p})-0 \nonumber\\
\Leftrightarrow \quad & \mathbb{E}_{s\sim \rho_{\bm{\pi}_k},\bm{a}\sim \bm{\pi}_k} \left[\left(r(\pi_{k+0.5}^{i_p})-1\right)\cdot A^{i_p}\right]\geq \mathcal{C}D_{\rm KL}^{\max}({\pi}^{i_p}_{k}, {\pi}_{k+0.5}^{i_p}) \nonumber
\end{align}
Therefore, by generalizing the above inequality to cover the cases of $p=1,\dots,n$ at the \emph{independent step}, we have 
\begin{align}
\label{equation_step1_results_proof}
\bigg\{\mathcal{C}D_{\rm KL}^{\max}({\pi}^{i_p}_{k}, {\pi}_{k+0.5}^{i_p})\leq \mathbb{E}_{s\sim \rho_{\bm{\pi}_k},\bm{a}\sim \bm{\pi}_k} \left[\left(r(\pi_{k+0.5}^{i_p})\cdot-1\right)\cdot A^{i_p}\right]\bigg\}_{p=1}^n
\end{align}
In full-pipeline paradigm, the result $\pi_{k+0.5}^{i_p}$ generated by pipeline $p$ at the independent step will be further optimized by one of the complementary pipelines $-p$ at the dependent step due to the non-overlapping selection. These interconnections among pipelines are essential for making full-pipeline optimization a high general update scheme. Thus, when initiating the \emph{dependent step} and executing Equation~(\ref{equation_full_step2}), it is important to take these interconnections into account. Specifically, for an arbitrary pipeline $p$, the dependent step will perform $\pi_{k+1}^{j_p} = \arg\max_{\pi^{j_p}}\widetilde{M}_{\bm{\pi}_k}^{-i_p}(\bm{\pi}_{k+1}^{-j_p}, \pi^{j_p})$ based on the update result $\pi_{k+0.5}^{j_p}$ obtained from another pipeline at the independent step:
\begin{align}
 & \mathbb{E}_{s,\bm{a}} \left[\left(r(\pi^{j_p}_{k+1})\cdot r(\bm{\pi}_{k+1}^{-\{i_p,j_p\}}) - 1\right)\cdot r(\pi_{k+1}^{i_p}) \cdot A^{j_p}\right] - \mathcal{C}D_{\rm KL}^{\max}({\pi}_k^{j_p}, \pi^{j_p}_{k+1})\nonumber \\
 &\qquad \qquad\qquad\qquad\geq \mathbb{E}_{s,\bm{a}} \left[\left(r(\pi^{j_p}_{k+0.5})\cdot r(\bm{\pi}_{k+1}^{-\{i_p,j_p\}}) - 1\right)\cdot r(\pi_{k+1}^{i_p}) \cdot A^{j_p}\right] - \mathcal{C}D_{\rm KL}^{\max}({\pi}_k^{j_p}, \pi_{k+0.5}^{j_p}) \nonumber\\
   \Leftrightarrow \qquad & \mathbb{E}_{s\sim \rho_{\bm{\pi}_k},\bm{a}\sim \bm{\pi}_k} \left[\left(r(\pi^{j_p}_{k+1})\cdot r(\bm{\pi}_{k+1}^{-j_p})-r(\pi^{j_p}_{k+0.5})\cdot r(\bm{\pi}_{k+1}^{-j_p})\right)A^{j_p}\right]\geq \mathcal{C}D_{\rm KL}^{\max}({\pi}^{j_p}_{k}, {\pi}_{k+1}^{j_p})-\mathcal{C}D_{\rm KL}^{\max}({\pi}^{j_p}_{k}, {\pi}_{k+0.5}^{j_p}) \nonumber
\end{align}
Therefore, by generalizing the above inequality to cover the cases of $p=1,\dots,n$ at the \emph{dependent step}, we have 
\begin{align}
   \label{equation_step2_results_proof}
   \bigg\{\mathbb{E}_{s\sim \rho_{\bm{\pi}_k},\bm{a}\sim \bm{\pi}_k} \left[\left(r(\bm{\pi}_{k+1})-r(\pi^{j_p}_{k+0.5})\cdot r(\bm{\pi}_{k+1}^{-j_p})\right)A^{j_p}\right]\geq \mathcal{C}D_{\rm KL}^{\max}({\pi}^{j_p}_{k}, {\pi}_{k+1}^{j_p})-\mathcal{C}D_{\rm KL}^{\max}({\pi}^{j_p}_{k}, {\pi}_{k+0.5}^{j_p})\bigg\}_{p=1}^n
\end{align}
For each pipeline $p$, by replacing $\mathcal{C}D_{\rm KL}^{\max}({\pi}^{j_p}_{k}, {\pi}_{k+0.5}^{j_p})$ term in Equation~(\ref{equation_step2_results_proof}) with the inequality in Equation~(\ref{equation_step1_results_proof}), we have
\begin{align}
   &\mathbb{E}_{s\sim \rho_{\bm{\pi}_k},\bm{a}\sim \bm{\pi}_k} \left[\left(r(\bm{\pi}_{k+1})-r(\pi^{j_p}_{k+0.5})\cdot r(\bm{\pi}_{k+1}^{-j_p})\right)A^{j_p}\right]\geq \mathcal{C}D_{\rm KL}^{\max}({\pi}^{j_p}_{k}, {\pi}_{k+1}^{j_p})-\mathbb{E}_{s\sim \rho_{\bm{\pi}_k},\bm{a}\sim \bm{\pi}_k} \left[\left(r(\pi_{k+0.5}^{j_p})-1\right)A^{j_p}\right]  \nonumber \\
   \Leftrightarrow \qquad &\mathcal{C}D_{\rm KL}^{\max}({\pi}^{j_p}_{k}, {\pi}_{k+1}^{j_p})\leq \mathbb{E}_{s\sim \rho_{\bm{\pi}_k},\bm{a}\sim \bm{\pi}_k} \left[\left(r(\bm{\pi}_{k+1})-r(\pi^{j_p}_{k+0.5})\cdot r(\bm{\pi}_{k+1}^{-j_p})\right)A^{j_p}\right]+\mathbb{E}_{s\sim \rho_{\bm{\pi}_k},\bm{a}\sim \bm{\pi}_k} \left[\left(r(\pi_{k+0.5}^{j_p})-1\right)A^{j_p}\right]  \nonumber \\
   \Leftrightarrow \qquad &\mathcal{C}D_{\rm KL}^{\max}({\pi}^{j_p}_{k}, {\pi}_{k+1}^{j_p})\leq \mathbb{E}_{s\sim \rho_{\bm{\pi}_k},\bm{a}\sim \bm{\pi}_k} \left[\left(r(\bm{\pi}_{k+1})-r(\pi^{j_p}_{k+0.5})\cdot r(\bm{\pi}_{k+1}^{-j_p})+r(\pi_{k+0.5}^{j_p})-1\right)A^{j_p}\right]. \nonumber
\end{align}
By generalizing the above inequality to cover the cases of $p=1,\dots,n$, we have
\begin{equation}
   \bigg\{\mathcal{C}D_{\rm KL}^{\max}({\pi}^{j_p}_{k}, {\pi}_{k+1}^{j_p}) \leq \mathbb{E}_{s,\bm{a}} \left[\left(r(\bm{\pi}_{k+1})-r(\pi_{k+0.5}^{j_p})\cdot r(\bm{\pi}_{k+1}^{-j_p})+r(\pi_{k+0.5}^{j_p})-1\right)A^{j_p} \right]\bigg\}_{p=1}^{n}. \nonumber
\end{equation}
Note that each inequality is not purely derived from the optimization of the individual pipeline. Specifically, for each pipeline $p$, $\mathcal{C}D_{\rm KL}^{\max}({\pi}^{j_p}_{k}, {\pi}_{k+0.5}^{j_p})$ term in Equation~(\ref{equation_step2_results_proof}) cannot utilize its own inequality in Equation~(\ref{equation_step1_results_proof}) produced at independent step since $i_p\neq j_p$. Nevertheless, it can be substituted with the inequality produced by another pipeline, highlighting the advantageous interconnections among pipelines. Therefore, we finally arrive at 
\begin{equation}
   \sum_{i=1}^n \mathcal{C}D_{\rm KL}^{\max}({\pi}^{i}_{k}, {\pi}_{k+1}^{i}) \leq \sum_{i=1}^n \mathbb{E}_{s,\bm{a}} \big[\big(r(\bm{\pi}_{k+1})-r(\pi_{k+0.5}^{i})\cdot r(\bm{\pi}_{k+1}^{-i})+r(\pi_{k+0.5}^{i})-1\big)A^{i} \big]. \nonumber
\end{equation}

\textbf{Parameter sharing case:}

Now, we briefly present the parameter sharing case where Equations~(\ref{equation_full_step1}) and (\ref{equation_full_step2}) are implemented by maximizing the expectation along pipelines. At the \emph{independent step}, we have
\begin{align}
& \frac{1}{n}\sum_{p=1}^{n}\left[\mathbb{E}_{s,\bm{a}} \left[r(\pi_{k+0.5}^{i_p})\cdot A^{i_p}\right]-\mathcal{C}D_{\rm KL}^{\max}\left({\pi}^{i_p}_{k}, {\pi}_{k+0.5}^{i_p}\right)\right] \geq \frac{1}{n}\sum_{p=1}^{n}\left[\mathbb{E}_{s,\bm{a}} \left[r(\pi^{i_p}_{k})\cdot A^{i_p}\right]-\mathcal{C}D_{\rm KL}^{\max}({\pi}^{i_p}_{k}, {\pi}_k^{i_p})\right] \nonumber\\
\Leftrightarrow \qquad & \sum_{p=1}^{n}\left[\mathbb{E}_{s\sim \rho_{\bm{\pi}_k},\bm{a}\sim \bm{\pi}_k} \left[\left(r(\pi_{k+0.5}^{i_p})-1\right)\cdot A^{i_p}\right]\right]\geq \sum_{p=1}^{n}\left[\mathcal{C}D_{\rm KL}^{\max}({\pi}^{i_p}_{k}, {\pi}_{k+0.5}^{i_p})\right]. \label{equation_step1_results_proof_sha}
\end{align}
At the \emph{dependent step}, we have
\begin{align}
& \frac{1}{n}\sum_{p=1}^{n}\left[\mathbb{E}_{s,\bm{a}} \left[\left(r(\pi^{j_p}_{k+1})\cdot r(\bm{\pi}_{k+1}^{-\{i_p,j_p\}}) - 1\right)\cdot r(\pi_{k+1}^{i_p}) \cdot A^{j_p}\right] - \mathcal{C}D_{\rm KL}^{\max}({\pi}_k^{j_p}, \pi^{j_p}_{k+1})\right]\nonumber \\
\geq & \frac{1}{n}\sum_{p=1}^{n}\left[\mathbb{E}_{s,\bm{a}} \left[\left(r(\pi^{j_p}_{k+0.5})\cdot r(\bm{\pi}_{k+1}^{-\{i_p,j_p\}}) - 1\right)\cdot r(\pi_{k+1}^{i_p}) \cdot A^{j_p}\right] - \mathcal{C}D_{\rm KL}^{\max}({\pi}_k^{j_p}, \pi_{k+0.5}^{j_p})\right] \nonumber\\
   \Leftrightarrow \qquad & \sum_{p=1}^{n}\left[\mathbb{E}_{s\sim \rho_{\bm{\pi}_k},\bm{a}\sim \bm{\pi}_k} \left[\left(r(\pi^{j_p}_{k+1})\cdot r(\bm{\pi}_{k+1}^{-j_p})-r(\pi^{j_p}_{k+0.5})\cdot r(\bm{\pi}_{k+1}^{-j_p})\right)A^{j_p}\right]\right]\nonumber \\
   \geq & \sum_{p=1}^{n}\left[\mathcal{C}D_{\rm KL}^{\max}({\pi}^{j_p}_{k}, {\pi}_{k+1}^{j_p})-\mathcal{C}D_{\rm KL}^{\max}({\pi}^{j_p}_{k}, {\pi}_{k+0.5}^{j_p})\right]. \label{equation_step2_results_proof_sha}
\end{align}
Combining Equations~(\ref{equation_step1_results_proof_sha}) and (\ref{equation_step2_results_proof_sha}) by $\sum_{i=1}^{n}\left[\mathcal{C}D_{\rm KL}^{\max}({\pi}^{i}_{k}, {\pi}_{k+0.5}^{i})\right]$ term, we arrive at 
\begin{equation}
   \sum_{i=1}^n \mathcal{C}D_{\rm KL}^{\max}({\pi}^{i}_{k}, {\pi}_{k+1}^{i}) \leq \sum_{i=1}^n \mathbb{E}_{s,\bm{a}} \big[\big(r(\bm{\pi}_{k+1})-r(\pi_{k+0.5}^{i})\cdot r(\bm{\pi}_{k+1}^{-i})+r(\pi_{k+0.5}^{i})-1\big)\cdot A^{i} \big]. \nonumber
\end{equation}
Consequently, Lemma~\ref{lemma_KL_expectation} helps to reach a unified outcome for different parameter-sharing types. 

\end{proof}

\subsection{Proof of Theorem~\ref{theorem_imporved_lower_bound_fp3o}}\label{appendix_imporved_lower_bound_fp3o}
\textbf{Theorem~\ref{theorem_imporved_lower_bound_fp3o}.} \emph{Whether parameter sharing is used or not, given Assumption~\ref{assumption_policy}, the new policies $\{\pi_{k+1}^{j_p}\}_{p=1}^n$ of full-pipeline optimization enable the shared lower bound to be monotonically improved:
$\mathcal{L}^{S}_{\bm{\pi}_k}(\pi_{k+1}^{1},\dots,\pi_{k+1}^{n}) \geq \mathcal{L}^{S}_{\bm{\pi}_k}(\pi_{k}^{1},\dots,\pi_{k}^{n}),$
which can directly guarantee $\mathcal{J}(\bm{\pi}_{k+1})\geq \mathcal{J}(\bm{\pi}_{k})$.}

\begin{proof}
   With Lemma~\ref{lemma_KL_expectation}, we have
   \begin{align}
      &\sum_{p=1}^n\mathcal{C}D_{\rm KL}^{\max}({\pi}^{j_p}_{k}, {\pi}_{k+1}^{j_p}) \nonumber \\
      \leq&\sum_{p=1}^n \mathbb{E}_{s,\bm{a}}\left[\left(r(\bm{\pi}_{k+1})-r(\pi_{k+0.5}^{j_p})\cdot r(\bm{\pi}_{k+1}^{-j_p})+r(\pi_{k+0.5}^{j_p})-1\right)A^{j_p} \right]\nonumber \\
      =&\mathbb{E}_{s,\bm{a}}\left[\sum_{p=1}^nr(\bm{\pi}_{k+1})A^{j_p}\right]-\mathbb{E}_{s,\bm{a}}\left[\sum_{p=1}^nA^{j_p}\right]-\sum_{p=1}^n\mathbb{E}_{s,\bm{a}}\left[\left(r(\pi_{k+0.5}^{j_p})\cdot r(\bm{\pi}_{k+1}^{-j_p})\right)A^{j_p}\right]+\sum_{p=1}^n\mathbb{E}_{s,\bm{a}}\left[r(\pi_{k+0.5}^{j_p})A^{j_p}\right]\nonumber \\
      =&\mathbb{E}_{s\sim \rho_{\bm{\pi}_k},\bm{a}\sim \bm{\pi}_k}\left[\frac{\bm{\pi}_{k+1}(\bm{a}|s)}{\bm{\pi}_{k}(\bm{a}|s)}\sum_{p=1}^nA^{j_p}\right]-\mathbb{E}_{s\sim \rho_{\bm{\pi}_k},\bm{a}\sim \bm{\pi}_k}\left[\sum_{p=1}^nA^{j_p}\right] \nonumber \\
      &\qquad -\sum_{p=1}^n\mathbb{E}_{s\sim \rho_{\bm{\pi}_k},\bm{a}\sim \bm{\pi}_k}\left[\left(\frac{\pi_{k+0.5}^{j_p}(a^{j_p}|s)}{\pi_{k}^{j_p}(a^{j_p}|s)}\cdot \frac{\bm{\pi}_{k+1}^{-j_p}(\bm{a}^{-j_p}|s)}{\bm{\pi}_{k}^{-j_p}(\bm{a}^{-j_p}|s)}\right)A^{j_p}\right]+\sum_{p=1}^n\mathbb{E}_{s\sim \rho_{\bm{\pi}_k},\bm{a}\sim \bm{\pi}_k}\left[\frac{\pi_{k+0.5}^{j_p}(a^{j_p}|s)}{\pi_{k}^{j_p}(a^{j_p}|s)}A^{j_p}\right]\nonumber \\
      =&\mathbb{E}_{s\sim \rho_{\bm{\pi}_k},\bm{a}\sim \bm{\pi}_k}\left[\prod_{i=1}^{n}\frac{\pi^i_{k+1}(a^i|s)}{\pi^i_{k}(a^i|s)}A_{\bm{\pi}_k}(s, \bm{a})\right]-\mathbb{E}_{s\sim \rho_{\bm{\pi}_k},\bm{a}\sim \bm{\pi}_k}\left[A_{\bm{\pi}_k}(s,\bm{a})\right]\nonumber \\
      &\qquad -\sum_{p=1}^n\mathbb{E}_{s\sim \rho_{\bm{\pi}_k},a^{j_p}\sim \pi^{j_p}_{k+0.5},\bm{a}^{-j_p}\sim \bm{\pi}^{-j_p}_{k+1}}\left[A^{j_p}\right]+\sum_{p=1}^n\mathbb{E}_{s\sim \rho_{\bm{\pi}_k},a^{j_p}\sim \pi^{j_p}_{k+0.5},\bm{a}^{-j_p}\sim \bm{\pi}^{-j_p}_k}\left[A^{j_p}\right] \nonumber 
      \end{align}
      \begin{align}
      =&\mathbb{E}_{s\sim \rho_{\bm{\pi}_k},\bm{a}\sim \bm{\pi}_k}\left[\prod_{i=1}^{n}\frac{\pi^i_{k+1}(a^i|s)}{\pi^i_{k}(a^i|s)}A_{\bm{\pi}_k}(s, \bm{a})\right]-\sum_{p=1}^n{\mu(\bm{\pi}^{-j_p}_{k+1})} +\sum_{p=1}^n{\mu(\bm{\pi}^{-j_p}_{k})}\qquad\qquad\qquad\qquad\qquad\qquad\qquad\qquad\nonumber \\
      &\text{which, by introducing the constraint in Equation~(\ref{equation_full_step2}), is}\nonumber \\
      \leq&\mathbb{E}_{s\sim \rho_{\bm{\pi}_k},\bm{a}\sim \bm{\pi}_k}\left[\prod_{i=1}^{n}\frac{\pi^i_{k+1}(a^i|s)}{\pi^i_{k}(a^i|s)}A_{\bm{\pi}_k}(s, \bm{a})\right]\label{equation_appendix_stacking_result}
   \end{align}
   % From the above derivation, we can also find that as long as $\sum_{p=1}^nA^{j_p}=A_{\bm{\pi}_{k}}(s,\bm{a})$ hold, we can obtain the outcome of Equation~\ref{equation_appendix_stacking_result}. Therefore, when revisiting the modified objectives in Section~\ref{section_full_pipeline_optimization}, we only need to split $A_{\bm{\pi}}(s, \bm{a})$ into $n$ arbitrary scalar values $A^1,...,A^n$, i.e., $\sum_{i=1}^n{A^{i}}=A_{\bm{\pi}}(s, \bm{a})$. There are no strict standards for specific spilt methods (e.g., we can execute average split). More Importantly, the \emph{scalar values} $A^1,...,A^n$ do not mean the individual advantage \emph{functions} and do not involve the credit assignment issue like~\cite{sunehag2017value}.
   Based on the above derivation, it becomes evident that Equation~(\ref{equation_appendix_stacking_result}) can be obtained as long as the condition $\sum_{p=1}^nA^{j_p}=A_{\bm{\pi}_{k}}(s,\bm{a})$ holds. Therefore, when we revisit the modified objectives in Section~\ref{section_full_pipeline_optimization}, it is understandable to split $A_{\bm{\pi}}(s, \bm{a})$ into $n$ arbitrary scalar values $A^1,...,A^n$, i.e., $\sum_{i=1}^n{A^{i}}=A_{\bm{\pi}}(s, \bm{a})$. How to split $A_{\bm{\pi}}(s, \bm{a})$ will not affect the theoretical outcome; hence, there is no strict criterion for the splitting way (e.g., we can use an average split). Based on the above analysis, it is important to note that the \emph{scalar values} $A^1,...,A^n$ should not be interpreted as individual advantage \emph{functions} and they do not involve the credit assignment issue like~\cite{sunehag2017value}.

   Previously, we have established that the shared lower bound $\mathcal{L}^{S}_{\bm{\pi}}(\tilde{\pi}^{1},\dots,\tilde{\pi}^{n})$ is equal to Equation~(\ref{equation_appendix_shared_lower_bound}). Thus, at the $k$-th iteration, we evaluate the difference in shared lower bounds before and after the full-pipeline optimization as follows:
   \begin{align}
      &\mathcal{L}^{S}_{\bm{\pi}_k}({\pi}_{k+1}^{1},...,{\pi}_{k+1}^{n})-\mathcal{L}^{S}_{\bm{\pi}_k}({\pi}_k^{1},...,{\pi}_k^{n})\nonumber \\
      =&\left(\mathbb{E}_{s\sim \rho_{\bm{\pi}_k},\bm{a}\sim \bm{\pi}_k}\left[\prod_{i=1}^{n}\frac{\pi^i_{k+1}(a^i|s)}{\pi^i_{k}(a^i|s)}A_{\bm{\pi}_k}(s, \bm{a})\right]-\sum_{i=1}^n\mathcal{C}D_{\rm KL}^{\max}({\pi}^{i}_{k}, {\pi}_{k+1}^{i})\right)\nonumber-\left(\mathbb{E}_{s\sim \rho_{\bm{\pi}_k},\bm{a}\sim \bm{\pi}_k}\left[A_{\bm{\pi}_k}(s, \bm{a})\right]-0\right)\nonumber \\
      &\text{which, by Equation~(\ref{equation_appendix_stacking_result}), satisfies}\nonumber \\
      \geq& 0\nonumber
   \end{align}
   With the improved shared lower bound, we thus have
   \begin{align}
      &\mathcal{J}(\bm{\pi}_{k+1}) \nonumber \\
      \geq&\mathcal{J}(\bm{\pi}_k)+M_{\bm{\pi}_k}^{i_p}({\pi}_{k+1}^\varnothing,\pi_{k+1}^{i_p})+M_{\bm{\pi}_k}^{-i_p}(\pi_{k+1}^{i_p}, \bm{\pi}_{k+1}^{-i_p})\nonumber \\
      =&\mathcal{L}^{S}_{\bm{\pi}_k}({\pi}_{k+1}^{1},...,{\pi}_{k+1}^{n}) \nonumber \\
      \geq&\mathcal{L}^{S}_{\bm{\pi}_k}({\pi}_k^{1},...,{\pi}_k^{n})\nonumber \\
      =&\mathcal{J}(\bm{\pi}_{k}). \nonumber
   \end{align}
\end{proof}

\subsection{Proof of the Main Objectives at the Independent Step}\label{appendix_main_objective_fp3o}
For arbitrary pipeline $p$, we simplify the main objective of the independent step as follows:
{\small
\begin{align}
   &L^{\rm FP3O}(\theta^{i_p}|\bm{\theta}_{k})\nonumber \\
   =&\mathbb{E}_{s,\bm{a}}\bigg[\min\bigg\{\left(r(\pi_{\theta^{i_p}}^{i_p})\cdot r\big(\bm{\pi}^{-\{f^{-1}(i_p),i_p\}}_{{\bm{\theta}_k}^{-\{f^{-1}(i_p),i_p\}}}\big)-1\right) r\big(\pi_{{\theta}_k^{f^{-1}(i_p)}}^{f^{-1}(i_p)}\big) \cdot\hat{A}^{i_p}, \Big({\rm clip}\big(r(\pi_{\theta^{i_p}}^{i_p}),1\pm\epsilon\big)\cdot r\big(\bm{\pi}^{-\{f^{-1}(i_p),i_p\}}_{{\bm{\theta}}_k^{-\{f^{-1}(i_p),i_p\}}}\big)-1\Big) r\big(\pi_{{\theta}_k^{f^{-1}(i_p)}}^{f^{-1}(i_p)}\big)\cdot \hat{A}^{i_p}\bigg\}\bigg]\nonumber \\
   &\text{which, with $r\big(\bm{\pi}^{-\{f^{-1}(i_p),i_p\}}_{{\bm{\theta}_k}^{-\{f^{-1}(i_p),i_p\}}}\big)=r\big(\pi_{{\theta}_k^{f^{-1}(i_p)}}^{f^{-1}(i_p)}\big)=1$, is}\nonumber \\
   =&\mathbb{E}_{s,\bm{a}}\bigg[\min\bigg(\Big(r(\pi_{\theta^{i_p}}^{i_p})-1\Big)\cdot \hat{A}^{i_p},\Big({\rm clip}\big(r(\pi_{\theta^{i_p}}^{i_p}),1\pm\epsilon\big)-1\Big)\cdot \hat{A}^{i_p}\bigg)\bigg]\nonumber \\
   =&\mathbb{E}_{s,\bm{a}}\left[\min\left(r(\pi_{\theta^{i_p}}^{i_p})\cdot \hat{A}^{i_p},{\rm clip}\big(r(\pi_{\theta^{i_p}}^{i_p}),1\pm\epsilon\big) \cdot \hat{A}^{i_p}\right)\right]-\mathbb{E}_{s,\bm{a}}\left[\hat{A}^{i_p}\right].\nonumber
\end{align}}
Clearly, it is the clipping objective of Equation~(\ref{equation_full_step1}) with an addition term $\mathbb{E}_{s,\bm{a}}\left[\hat{A}^{i_p}\right]$. However, it is important to note that $\mathbb{E}_{s,\bm{a}}\left[\hat{A}^{i_p}\right]$ has no contribution to the gradient $\frac{\partial L^{\rm FP3O}(\theta^{i_p}|\bm{\theta}_{k})}{\partial \theta^{i_p}}$ when optimizing $\theta^{i_p}$. Thus, $L^{\rm FP3O}(\theta^{i_p}|\bm{\theta}_{k})$ is completely equivalent to the clipping objective of Equation~(\ref{equation_full_step1}).

\subsection{Intuitive Diagram}
To be intuitive, we present Figure~\ref{fig_appendix_agent_demo} to illustrate the procedure of our FP3O algorithm. As shown in the Figure~\ref{fig_appendix_agent_demo}, HAPPO adopts a sequential update scheme where 5 agents are allocated one by one to train on a single pipeline. Therefore, this process entails a total of 5 update steps. During each step, the training agent can observe the update results of its preceding agents, while the remaining agents stay idle. In contrast, FP3O aims to establish multiple parallel pipelines, enabling these idle agents to be trained concurrently on these parallel pipelines. During the independent step, each agent is trained independently, while during the dependent step, agents have the opportunity to utilize the updated results of the independent step across all pipelines.

\begin{figure}[H]
    \centering
   \includegraphics[width=0.8\linewidth]{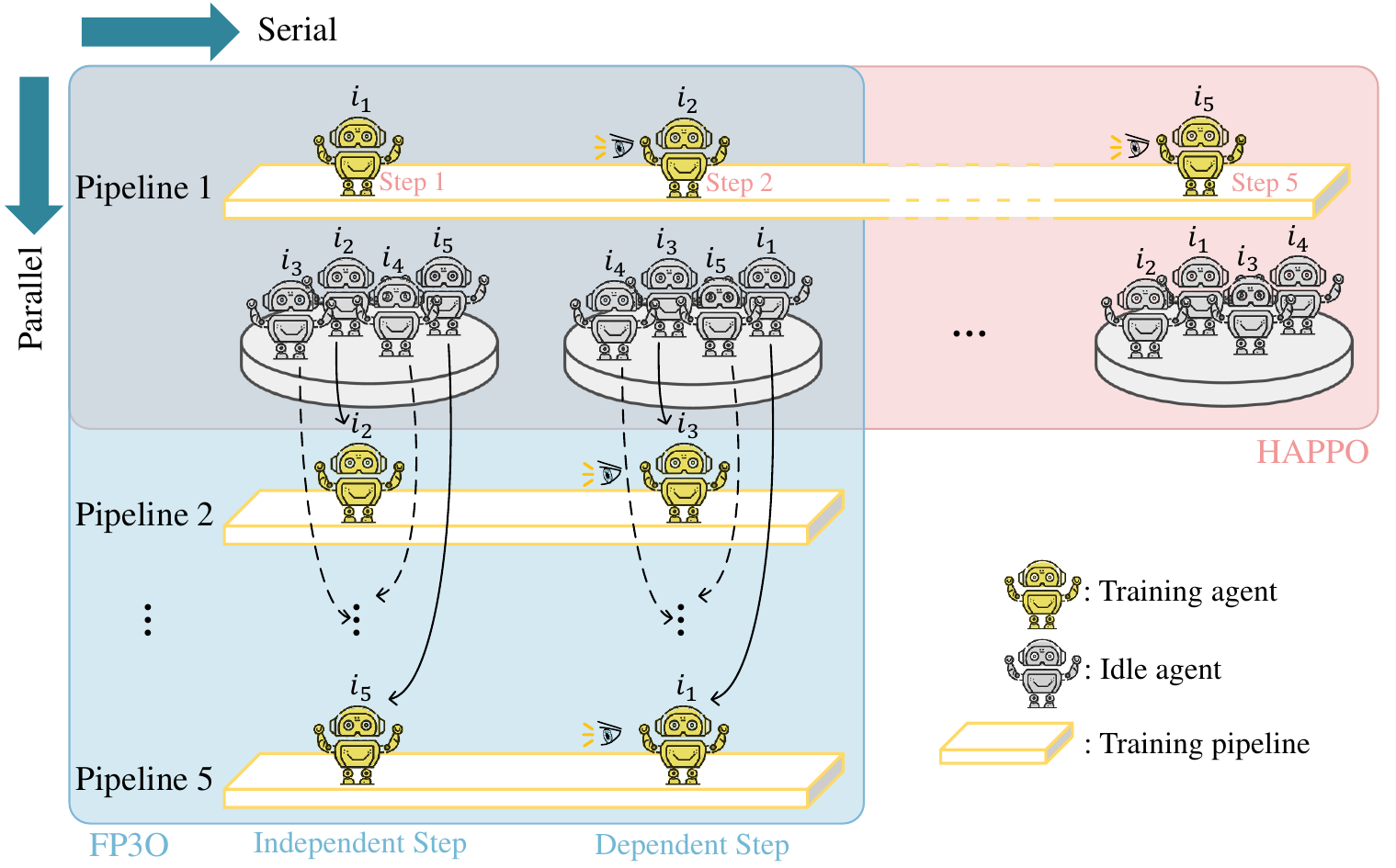}
   %\framebox[4.0in]{$\;$}
   \caption{Illustrative diagram of FP3O and HAPPO algorithms in the 5-agent scenario. The blue box and text encapsulate the main process of FP3O, while the red box and text encapsulate the main process of HAPPO. There is an overlapping part between two boxes in the upper left corner.}
   \label{fig_appendix_agent_demo}
\end{figure}

\subsection{Instability Analysis}
The instability arises from the cumulative product of ratios, e.g., $\prod_{i=1}^n(\frac{\pi^i}{\pi^i_{\rm old}})$. As the order $n$ increases, the instability intensifies. To provide a concise and unified representation, we present the main objectives of PPO-based algorithms in the following table:

\begin{table}[H]
\centering
\begin{tabular}{@{\,\,}c @{\,\,}|@{\,\,} c c c c c@{\,\,}}
   \toprule
   Method & IPPO/MAPPO & CoPPO & HAPPO & A2PO & Ours \\ \midrule
   Main Objective & \multirow{2}{*}{$\frac{\pi^i}{\pi^i_{\rm old}}A^i$} & \multirow{2}{*}{$\frac{\pi^i}{\pi^i_{\rm old}}\cdot\frac{\boldsymbol{\pi}^{-i}}{\boldsymbol{\pi}^{-i}_{\rm old}}A^i$} & \multirow{2}{*}{$\left(\frac{\pi^i}{\pi^i_{\rm old}}-1\right)\frac{\boldsymbol{\pi}^{<i}}{\boldsymbol{\pi}^{<i}_{\rm old}}A^i$} & \multirow{2}{*}{$\frac{\pi^i}{\pi^i_{\rm old}}\cdot\frac{\boldsymbol{\pi}^{<i}}{\boldsymbol{\pi}^{<i}_{\rm old}}A^i$} & \multirow{2}{*}{$\left(\frac{\pi^i}{\pi^i_{\rm old}}\cdot\frac{\boldsymbol{\pi}^{-\{j,i\}}}{\boldsymbol{\pi}^{-\{j,i\}}_{\rm old}}-1\right)\frac{\pi^j}{\pi^j_{\rm old}}A^i$} \\
   (Before Clipping) & & & & & \\ \midrule
   Order & $1$ & $n$ & $1,...,n$ & $1,...,n$ & $1, n$ \\
   Min Order & $1$ & $n$ & $1$ (first step) & $1$ (first step) & $1$ (independent step) \\
   Max Order & $1$ & $n$ & $n$ (last step) & $n$ (last step) & $n$ (dependent step) \\
   \bottomrule
\end{tabular}
\caption{Comparison of the main objectives and instability of various PPO-based algorithms including IPPO/MAPPO, CoPPO, HAPPO, A2PO, and our FP3O.}
\label{table_appendix_instability_analysis}
\end{table}
As indicated by the table, the instability of CoPPO remains constant at $n$ and the instability of HAPPO and A2PO progressively ranges from 1 to $n$. In contrast, our approach exhibits instability at 1 and $n$. Our approach does not introduce additional instability issues and may even reduce the instability challenge compared to its predecessors (CoPPO, HAPPO, A2PO), all of which are theoretically proven algorithms. Thus, the enhancement and efficiency of FP3O do not arise from compromising on instability. Nevertheless, when compared with IPPO/MAPPO, P3O still displays a higher degree of instability.

To mitigate the instability, ~\citep{wu2021coordinated} introduced the double-clip trick for CoPPO (an additional inner clip limits the cumulative product: $\frac{\boldsymbol{\pi}^{-i}}{\boldsymbol{\pi}^{-i}_{\rm old}}\rightarrow \text{clip}(\frac{\boldsymbol{\pi}^{-i}}{\boldsymbol{\pi}^{-i}_{\rm old}}, 1\pm \epsilon_2)$), and A2PO~\cite{wang2023order} follows this as well ($\frac{\boldsymbol{\pi}^{<i}}{\boldsymbol{\pi}^{<i}_{\rm old}}\rightarrow \text{clip}(\frac{\boldsymbol{\pi}^{<i}}{\boldsymbol{\pi}^{<i}_{\rm old}}, 1\pm \epsilon_2)$). Therefore, the double-clip technique is deemed a viable approach to alleviate the instability in estimating agent $i$'s policy gradient if necessary. Specifically, it is achieved by first clipping the joint policy ratio of other agents, followed by clipping the policy ratio of agent $i$. With the double-clip technique, the main objective $L^{\rm FP3O}(\theta^i|\bar{\bm{\theta}})$ of FP3O can be formulated as 
\begin{equation}
    \label{eq_double_clip}
   \mathbb{E}_{s,\bm{a}}\bigg[\min\bigg(\Big(r(\pi_{\theta^{i}}^{i})\cdot g(-i)-r(\pi_{\bar{\theta}^{j}}^{j})\Big) \cdot\hat{A}^{i}, \Big({\rm clip}\big(r(\pi_{\theta^{i}}^{i}),1\pm\epsilon_1\big)\cdot g(-i)-r(\pi_{\bar{\theta}^{j}}^{j})\Big) \cdot \hat{A}^{i}\bigg)\bigg], 
\end{equation}
where $g(-i)={\rm clip}\Big(r\big(\bm{\pi}^{-i}_{\bar{\bm{\theta}}^{-i}}\big),1\pm\epsilon_2\Big)$, $\epsilon_1$ is the clipping range of the policy ratio of agent $i$, and $\epsilon_2$ is the clipping range of the joint policy ratio of other agents. To validate the above analysis, we conduct experiments in Appendix~\ref{appendix_stability_test}.

\section{Experiments Settings}
\subsection{Details of Environments}\label{appendix_environments}
\paragraph{MAMuJoCo.}
MAMuJoCo benchmark \citep{de2020mamujoco} is a continuous and partially observable task~\footnote{Note that we employ a partial observability setting for MAMuJoCo tasks, which differs from the full observability setting in~\cite{wang2023order}. For instance, in the case of the 2-agent Walker task, the observation is an 11-dimensional vector under partial observability (ours), while it is a 19-dimensional vector under full observability~\cite{wang2023order}. The full observability setting provides more comprehensive information compared to our partial observability setting.}. MAMuJoCo groups different joints of a robot in the MuJoCo simulator \citep{todorov2012mujoco} and models them as different agents. For instance, as shown in Figure \ref{fig_mujoco} (a), MAMuJoCo regards six joints of HalfCheetah as six different agents and then lets each agent control its joints based on its local observation. Because of the diversification of body part control, the agents are regarded as heterogeneous. 
\begin{figure}[H]
    \centering
   \includegraphics[width=0.7\linewidth]{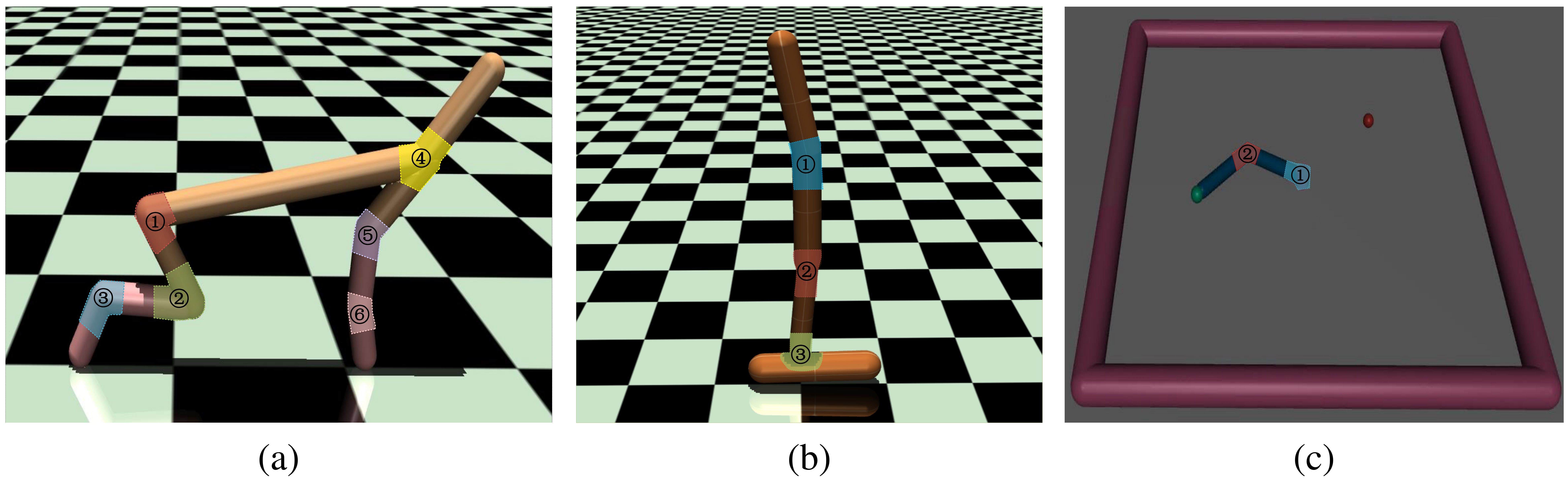}
   %\framebox[4.0in]{$\;$}
   \caption{Examples of MAMuJoCo. Different agents are depicted in different colors. (a) 6-Agent HalfCheetah [6$\times$1]. (b) 3-Agent Hopper [3$\times$1]. (c) 2-Agent Reacher [2$\times$1].}
   \label{fig_mujoco}
\end{figure}

\paragraph{SMAC.}
SMAC benchmark \citep{samvelyan2019starcraft} is a discrete and partially observable task. There are a variety of maps involving three difficulties: Easy, Hard, and Super-Hard. The goal is to train a team of ally units to defeat an opponent team of enemy units. Table \ref{table_smac_unit_type} lists the types of units in different scenarios. According to whether the types of ally units are the same or not, the scenarios be classified into homogeneous-agent or heterogeneous-agent tasks. For example, agents are regarded to be heterogeneous in scenarios bane vs. bane and 3s5z because the ally contains two different types of units. Agents are regarded to be homogeneous in the remaining scenarios (i.e., 2c vs. 64zg, 5m vs. 6m, corridor, 6h vs. 8z) because of the single type of units.
\begin{table}[H]
   \centering
   \begin{tabular}{@{\,} c @{\,\,\,\,} c  @{\,\,\,\,} c @{\,\,\,} c @{\,}}
      \toprule
      Scenarios    & Ally Units                          & Enemy Units                         & Type \\ \midrule
      bane vs bane & 20 Zerglings \& 4 Banelings         & 20 Zerglings \& 4 Banelings         & HE   \\
      2c vs 64zg   & 2 Colossi                           & 64 Zerglings                        & HO   \\
      3s5z         & 3 Stalkers \& 5 Zealots             & 3 Stalkers \& 5 Zealots             & HE   \\
      5m vs 6m     & 5 Marines                           & 6 Marines                           & HO   \\
      corridor     & 6 Zealots                           & 24 Zerglings                        & HO   \\
      6h vs 8z     & 6 Hydralisks                        & 8 Zealots                           & HO   \\
      \bottomrule
   \end{tabular}
   \caption{Unit types in SMAC benchmark.}
   \label{table_smac_unit_type}
\end{table}
\subsection{Baseline Selection}
Our baselines includes HAPPO~\citep{kuba2022happo}, MAPPO~\citep{yu2021mappo}, and IPPO~\citep{de2020ippo}. To ensure a comprehensive evaluation, we have extended HAPPO, MAPPO, and IPPO to encompass all three parameter-sharing types, i.e., full, partial, and non-parameter sharing. 

Additionally, we have included CoPPO~\citep{wu2021coordinated} with full parameter sharing in Appendix~\ref{appendix_CoPPO_Baseline} due to its limited adaptability to partial/non-parameter sharing. Given that A2PO~\cite{wang2023order} employs the same sequential update scheme as HAPPO and shares similar limitations in full/partial parameter sharing cases as demonstrated in Appendix~\ref{appendix_sequential_update_scheme_bad}, we have selected HAPPO as a representative. Nevertheless, we also add experiments of A2PO on Multi-Agent Particle Environment (MPE)~\cite{mordatch2018emergence} in Appendix~\ref{appendix_experiments_on_MPE}.

To maintain the consistency of our study, we have consciously excluded downstream works of these PPO-based algorithms from our baselines, as their objectives and contributions diverge from ours. For instance, \citet{li2022ace} proposed bidirectional action-dependent Q-learning to address the non-stationarity problem, which was later incorporated into MAPPO. Similarly, \citet{wen2022multiatt} integrated the Transformer model into HAPPO to enhance its performance.

Considering the following two reasons, we haven't run the comparisons to HATRPO~\cite{kuba2022happo}.
(1) HATRPO and HAPPO share the same theoretical foundation and sequential update scheme, both of which exhibit the same limitations in full or partial parameter sharing as demonstrated in Appendix~\ref{appendix_sequential_update_scheme_bad}. By evaluating HAPPO, we can sufficiently demonstrate our contribution. Second, TRPO-based algorithms, such as HATRPO, differ significantly from PPO-based algorithms in their implementation, including aspects like the utilization of the conjugate gradient, the Hessian of the average KL-divergence, and notably, distinct actor update methods (TRPO-based algorithms use an estimated maximal step size, while PPO-based algorithms use a learning rate). Thus, conducting a comparison and drawing meaningful observations would be challenging due to the presence of these unrelated factors.

\begin{figure}[H]
    \centering
   \includegraphics[width=0.8\linewidth]{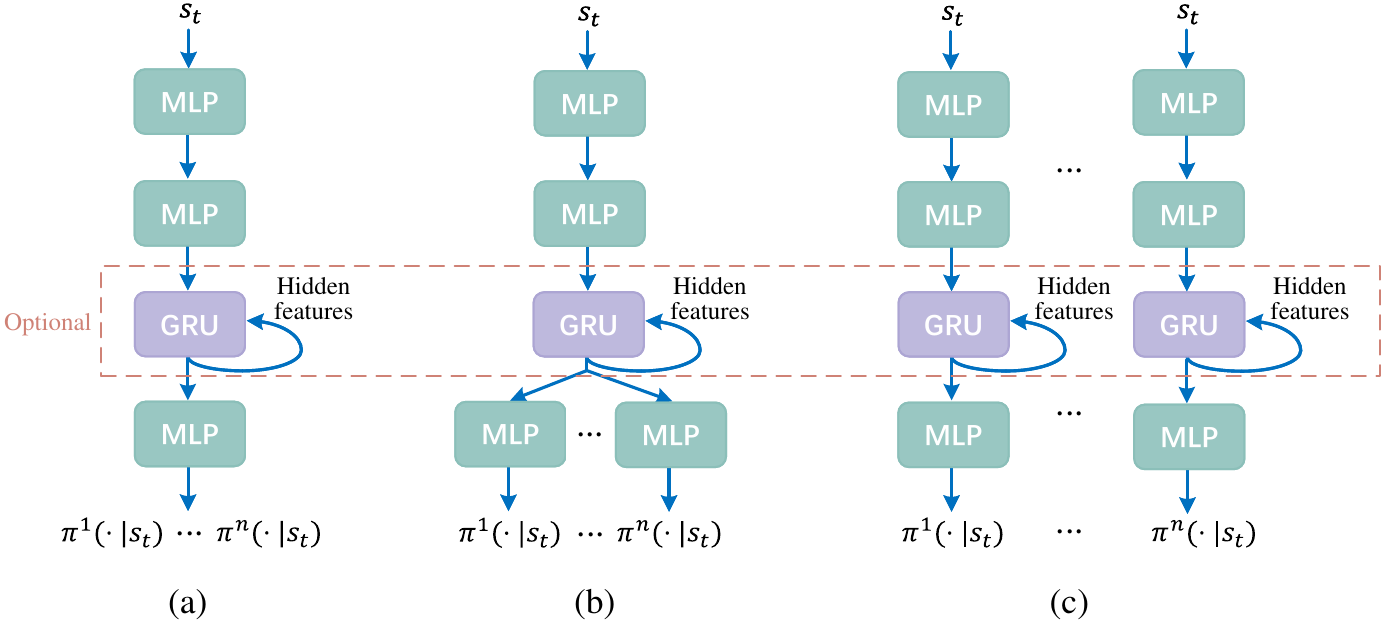}
   %\framebox[4.0in]{$\;$}
   \caption{Illustration of different network types. The GRU layer is optional in some tasks. (a) Full parameter sharing. (b) Partial parameter sharing. (c) Non-parameter sharing.}
   \label{fig_FuPS_PaPS_NoPS}
\end{figure}

\subsection{Details of Training}\label{appendix_hyperparameters}
\paragraph{Network type.}
Figure~\ref{fig_FuPS_PaPS_NoPS} depicts the different network configurations used in our experiments. Full parameter sharing: different agents' policy networks share the same set of parameters. Partial parameter sharing: the last layers of different agents' policy networks have individual parameters, while other layers share the same set of parameters. Non-parameter sharing: each agent's policy network has an individual set of parameters.

\paragraph{Non-overlapping selection setting.} In our experiment, we perform non-overlapping selection with an easily-implemented method, i.e., cyclic shift. For example, in a 5-agent scenario, $i_{1:n}=\{2, 4, 1, 3, 5\}$ executes cyclic left shift, then yielding $j_{1:n}=\{4, 1, 3, 5, 2\}$.

\paragraph{Repetitive experiments.} In our experiments on MAMuJoCo, the score is averaged over 32 test episodes after each training iteration. The final evaluation score is further averaged over 5 different seeds. In SMAC, We follow the median win rate evaluation metric adopted in \cite{wang2020rode} and \cite{yu2021mappo}. The win rate is computed over 32 test episodes after each training iteration and the final evaluation result is averaged over 5 different seeds.

\paragraph{Computing infrastructure.} We implement our algorithm with Python 3.7 and PyTorch 1.9.0 using GPU acceleration. All of the experiments are trained on TITAN X GPU, running 64-bit Linux 4.4.0. All dependencies and their version have been listed in \verb+requirements.txt+ of our source code.

% \subsection{Hyperparameter Settings in MAMuJoCo Domain}\label{appendix_hyperparameters_mamujoco}
\paragraph{Hyperparameter settings.} Most of the hyperparameters in our MAMuJoCo and SMAC experiments are kept the same as the original hyperparameters in the HAPPO paper~\cite{kuba2022happo} and the MAPPO paper~\cite{yu2021mappo}. However, in the MAMuJoCo experiments, we increased the learning rate to 1e-5 from the original 5e-6 used in HAPPO to prevent too-slow convergence. Additionally, we introduced partial observability to MAMuJoCo experiments by setting the range of observation to 5, meaning that each agent observes the nearest 5 agents or joints. This is different from HAPPO, which assumes full observability. Furthermore, we used a 0.05 ppo clip for 2-Agent Walker and 3-Agent Hopper because of the high instability and large variance observed in their training process. The detailed hyperparameters in MAMuJoCo are listed in  Table \ref{table_hyperparameter_mujoco1} and Table \ref{table_hyperparameter_mujoco2}. The detailed hyperparameters in SMAC are listed in Table \ref{table_hyperparameter_smac1} and Table \ref{table_hyperparameter_smac2}.
% The hyperparameter settings in MAMuJoCo are shown in Table \ref{table_hyperparameter_mujoco1} and Table \ref{table_hyperparameter_mujoco2}.
\begin{table}[H]
   \centering
   \begin{tabular}{c c || c c}
      \toprule
      Hyperparameter         & Value      & Hyperparameter     & Value      \\ \midrule
      actor lr               & 1e-5       & gain               & 0.01       \\
      critic lr              & 3e-4       & episode length     & 1000       \\
      gamma                  & 0.99       & rollout threads    & 4          \\
      gae lamda              & 0.95       & num mini-batch     & 32       \\
      optimizer              & Adam       & gradient clip norm & 10         \\
      optimizer epsilon      & 1e-5       & fc layer dim   & 64         \\
      weight decay           & 0          & num fc   & 2         \\
      activation             & ReLU       & num layer after    & 1          \\
      network initialization & Orthogonal & value loss         & huber loss          \\
      training threads       & 8          & huber delta        & 10.0 \\
      ppo epochs             & 5          & entropy coef       & 0.001       \\
      stacked frames         & 1          & std x coef             & 1       \\
     agent observation       & 5          & std y coef         & 0.5     \\
      \bottomrule
   \end{tabular}
   \caption{Common hyperparameters in MAMuJoCo.}
   \label{table_hyperparameter_mujoco1}
\end{table}

\begin{table}[H]
   \centering
   \begin{tabular}{c c c c c c}
      \toprule
      \multirow{2}{*}{Task}  & \multirow{2}{*}{PPO Clip} & \multirow{2}{*}{Actor Network} & \multicolumn{3}{c}{Steps} \\ \cmidrule{4-6}
      ~                     &           &        & FuPS & PaPS & NoPS \\ \midrule
      2-Agent Reacher [2$\times$1]     & 0.2      & mlp     & 10e6  & 10e6 & 10e6 \\
      2-Agent Ant [2$\times$4]         & 0.2      & mlp     & 10e6  & 10e6 & 10e6 \\
      2-Agent Walker [2$\times$3]      & 0.05      & mlp     & 10e6  & 10e6 & 10e6 \\
      3-Agent Hopper [3$\times$1]      & 0.05      & mlp     & 10e6  & 10e6 & 10e6 \\
      6-Agent HalfCheetah [6$\times$1] & 0.2      & mlp     & 10e6  & 10e6 & 10e6\\
      Manyagent Swimmer [8$\times$2]   & 0.2      & mlp     & 10e6 & 10e6 & 10e6 \\
      \bottomrule
   \end{tabular}
   \caption{Adopted hyperparameters for different tasks in MAMuJoCo. FuPS, PaPS, and NoPS denote full parameter sharing, partial parameter sharing, and non-parameter sharing respectively.}
   \label{table_hyperparameter_mujoco2}
\end{table}

\begin{table}[H]
    \centering
    \begin{tabular}{c c || c c}
       \toprule
       Hyperparameters        & Value      & Hyperparameters    & Value      \\ \midrule
       actor lr               & 5e-4       & gain               & 0.01       \\
       critic lr              & 5e-4       & episode length     & 400        \\
       gamma                  & 0.99       & rollout threads    & 8          \\
       gae lamda              & 0.95       & batch size         & 3200       \\
       optimizer              & Adam       & num mini-batch     & 1          \\
       optimizer epsilon      & 1e-5       & gradient clip norm & 10         \\
       weight decay           & 0          & num fc   & 2         \\
       activation             & ReLU       & fc layer dim   & 64          \\
       network initialization & Orthogonal & num GRU layers    & 1          \\
       training threads       & 32         & RNN hidden dim    & 64 \\
       ppo epochs             & 5          & num layer after    & 1       \\
       stacked frames         & 1          & value       & huber loss      \\
       entropy coef           & 0.01       & huber delta        & 10.0       \\

       \bottomrule
    \end{tabular}
    \caption{Common hyperparameters in SMAC.}
    \label{table_hyperparameter_smac1}
\end{table}

\begin{table}[H]
\centering
\begin{tabular}{c c c c c c}
   \toprule
   \multirow{2}{*}{Map}  & \multirow{2}{*}{PPO Clip} & \multirow{2}{*}{Actor Network} & \multicolumn{3}{c}{Steps} \\ \cmidrule{4-6}
   ~                     &           &        & FuPS & PaPS & NoPS \\ \midrule
   bane vs. bane          & 0.2      & rnn     & 2e6  & 2e6 & 2e6 \\
   2c vs. 64zg            & 0.2      & rnn     & 5e6  & 5e6 & 5e6 \\
   3s5z                  & 0.2      & rnn     & 5e6  & 5e6 & 5e6 \\
   5m vs. 6m               & 0.05     & rnn     & 8e6  & 10e6 & 10e6 \\
   corridor               & 0.2      & mlp     & 8e6  & 10e6 & 10e6\\
   6h vs. 8z               & 0.2      & mlp     & 10e6 & 15e6 & 20e6 \\
   10m vs. 11m       & 0.2      & rnn     & 8e6 & - & - \\
   3s5z vs. 3s6z     & 0.05      & rnn     & 12e6 & - & - \\
   \bottomrule
\end{tabular}
\caption{Adopted hyperparameters for different maps in SMAC. FuPS, PaPS, and NoPS denote full parameter sharing, partial parameter sharing, and non-parameter sharing respectively.}
\label{table_hyperparameter_smac2}
\end{table}

% \newpage
\section{Additional Experiments and Analysis}\label{appendix_additional_exp}
\subsection{Comparison with CoPPO Baseline}\label{appendix_CoPPO_Baseline}
We evaluate the FP3O algorithm against CoPPO baseline \citep{wu2021coordinated} on more SMAC scenarios with full parameter sharing. As shown in Figure~\ref{fig_coppo}, we can observe that the training process of CoPPO is less stable, even in relatively easy scenarios (i.e., bane vs. bane and 2c vs. 64zg). This instability is a result of CoPPO's direct update to the joint policy during optimization, causing the optimization objective of each agent to constantly change after every mini-batch update. On the contrary, FP3O demonstrates a more stable performance across all tasks and achieves better overall results.
\begin{figure}[h]
\centering
\includegraphics[width=0.8\linewidth]{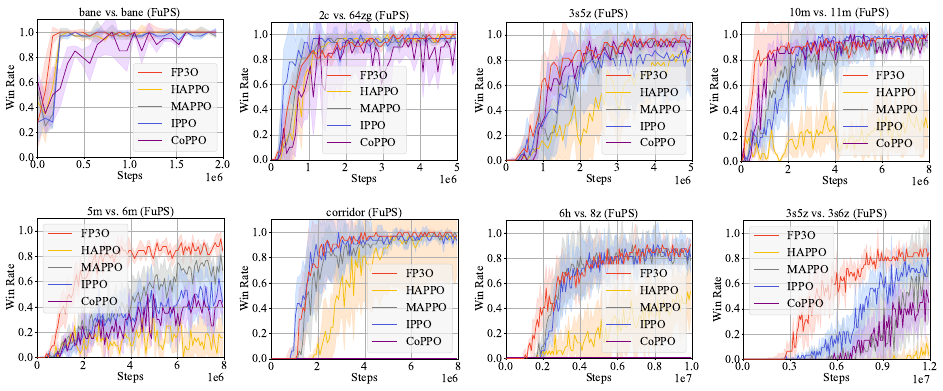}
%\framebox[4.0in]{$\;$}
\caption{Median evaluation win rate in the SMAC domain. FuPS denotes full parameter sharing.}
\label{fig_coppo}
\end{figure}

\subsection{Convergence Monitoring in 6h vs. 8z (PaPS)}\label{appendix_training_monitoring}
We depict the learning curves of MAPPO, HAPPO, and FP3O on 5 random seeds to monitor their convergence in the 6h vs. 8z (PaPS) task. As shown in Figure~\ref{fig_appendix_convergence_monitoring}, both MAPPO and HAPPO struggle to converge and occasionally reach a 0.0\% win rate on some seeds. HAPPO's struggle can be attributed to its limitation in the parameter-sharing case. Regarding MAPPO, it can be observed that seed 3 and seed 4 demonstrated convergence trends after only $2.5\times10^6$ steps, whereas the training of the 0\% seeds did not show any convergence trend even after 6$\times$ sampling efficiency ($1.5\times10^7$ steps). 
\begin{figure}[h]
\centering
\includegraphics[width=0.6\linewidth]{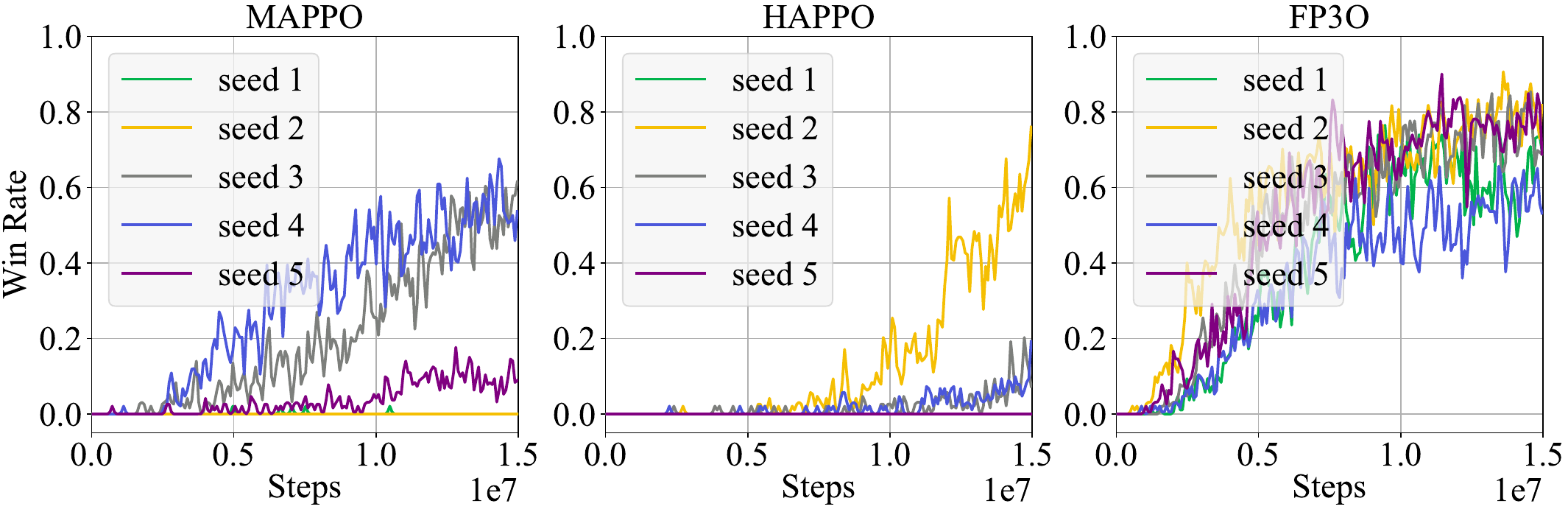}
%\framebox[4.0in]{$\;$}
\caption{The learning curves of MAPPO, HAPPO, and FP3O on 5 random seeds in the 6h vs. 8z (PaPS) task.}
\label{fig_appendix_convergence_monitoring}
\end{figure}
This suggests that MAPPO may have encountered some local optima from which it struggled to escape. 
However, our FP3O algorithm effectively overcomes local optima on all seeds, consistently delivering outstanding performance. Therefore, convergence is a valuable metric to evaluate an algorithm's capability to escape local optima, particularly under identical hyperparameter settings (the same amount of sampled data at each iteration).

\subsection{Experiments on More Algorithms and Domains}\label{appendix_experiments_on_MPE}
We conduct experiments on Multi-Agent Particle Environment (MPE)~\cite{mordatch2018emergence} (all hyperparameters are set the same as MAPPO paper~\cite{yu2021mappo}). We also incorporate the performance of the A2PO baseline~\cite{wang2023order}\footnote{We adopt the original hyperparameters of A2PO from~\cite{wang2023order}.} in the 3s5z task (FuPS). The results of these experiments are presented in Table~\ref{table_appendix_MPE} and Figure~\ref{fig_appendix_a2po_3s5z}. As can be discerned from the results, A2PO underperforms when using full parameter sharing. This shortfall is similar to HAPPO and can be attributed to the sequential updates, as explained in Appendix~\ref{appendix_sequential_update_scheme_bad}. Consequently, the sequential update scheme of A2PO also exhibits a lack of versatility across different parameter-sharing configurations. Although block updates\footnote{A2PO selectively partitions the agents into blocks for tasks involving numerous agents. For instance, in the 3s5z task, 8 agents are divided into 3 blocks, each containing 3, 3, and 2 agents, respectively. Consequently, the number of sequential updates decreases from 8 to 3. We denote the use of block updates as ``A2PO w/ Block'' and without as ``A2PO w/o Block''.} (A2PO w/ Block) can reduce the number of sequential updates to alleviate performance degradation, they do not address the fundamental issues of broken improvement guarantee and excessive KL divergence.
\begin{table}[h]
\begin{minipage}{0.65\textwidth}
\centering
\begin{tabular}{@{\,}c@{\,}|@{\,}c@{\,} c@{\,} c@{\,} c@{\,}}
\toprule
Task & FP3O & HAPPO & MAPPO & A2PO \\ \midrule
Simple reference (FuPS) & $-4.9_{\pm4.0}$ & $-9.8_{\pm6.4}$ & $-7.2_{\pm5.6}$ & $-9.2_{\pm7.1}$ \\
Cooperative communication (FuPS) & $-7.3_{\pm8.6}$ & $-13.7_{\pm9.5}$ & $-12.1_{\pm9.0}$ & $-12.9_{\pm10.6}$\\
Simple reference (NoPS) & $-4.6_{\pm2.1}$ & $-6.8_{\pm3.5}$ & $-6.5_{\pm4.8}$  & $-5.1_{\pm2.9}$ \\
Cooperative communication (NoPS) & $-4.8_{\pm5.0}$ & $-6.7_{\pm5.2}$ & $-9.1_{\pm6.8}$& $-6.9_{\pm5.7}$ \\
\bottomrule
\end{tabular}
\caption{The average evaluation rewards and standard deviations on MPE. FuPS and NoPS denote full parameter sharing and non-parameter sharing respectively.}
\label{table_appendix_MPE}
\end{minipage}%
\hspace{0.2cm}
\begin{minipage}{0.35\textwidth}
\centering
\includegraphics[width=1\textwidth]{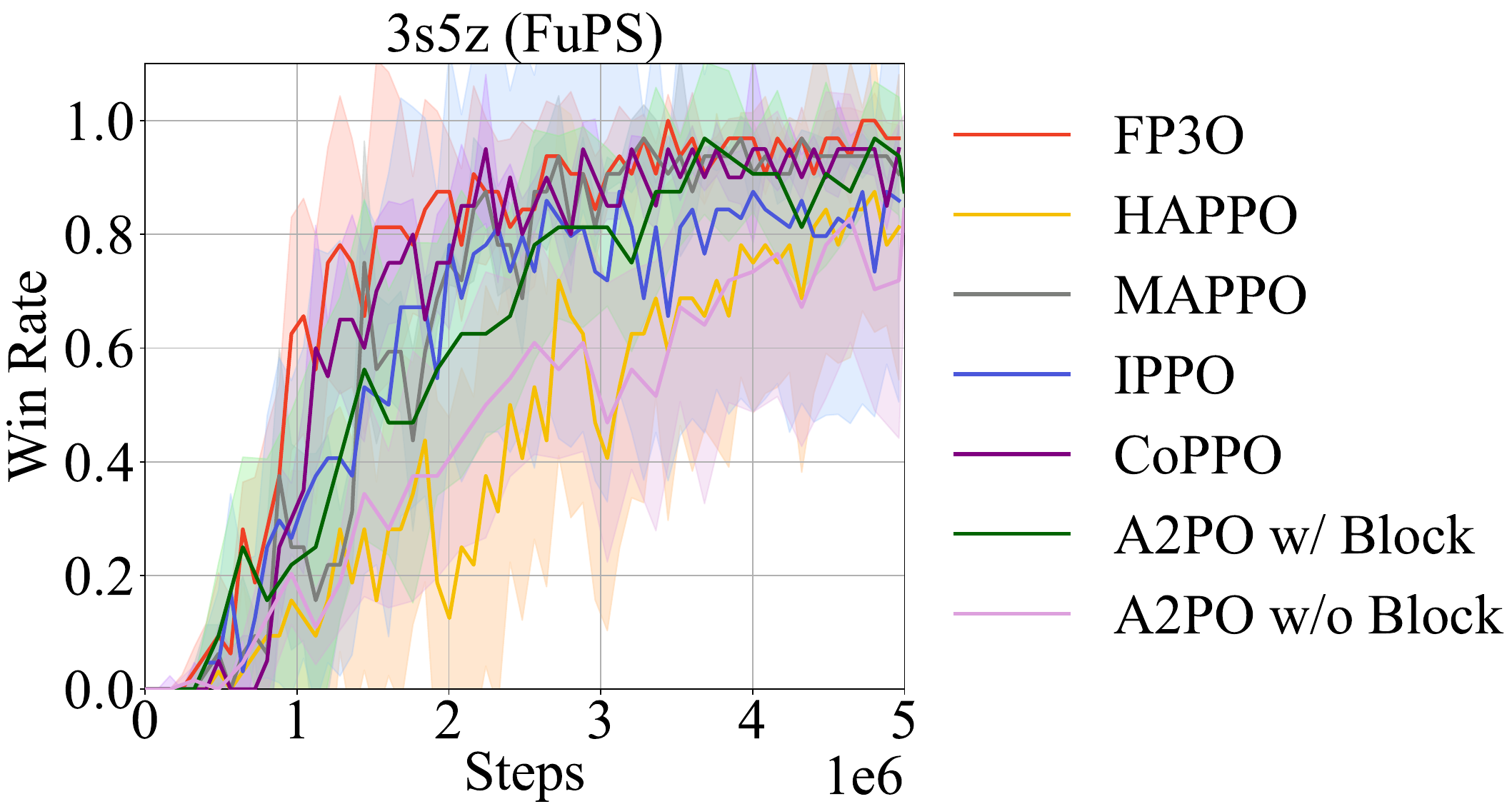}
%\framebox[4.0in]{$\;$}
\caption{Median evaluation win rate on 3s5z. FuPS denotes full parameter sharing.}
\label{fig_appendix_a2po_3s5z}
\end{minipage}
\end{table}

\subsection{Experiments on $\{A^i\}_{i=1}^n$}
In Section~\ref{section_full_pipeline_optimization}, we split $A_{\bm{\pi}}(s, \bm{a})$ into $n$ arbitrary scalar values $A^1,...,A^n$, i.e., $\sum_{i=1}^n{A^{i}}=A_{\bm{\pi}}(s, \bm{a})$. In Appendix~\ref{appendix_imporved_lower_bound_fp3o}, we prove that how to split $A_{\bm{\pi}}(s, \bm{a})$ does not impact the theoretical outcome; hence, there is no strict criterion for the splitting way. These insights are derived at a theoretical level, and we are currently investigating the practical implications of different splitting ways. In Figure~\ref{fig_appendix_random_average}, we compare the performances of FP3O using two splitting ways: average split (i.e., $\hat{A}^{i}=\frac{1}{n}\hat{A}_{\bm{\pi}}$), and random split (i.e., randomly generating $n$ scalar values $\{\hat{A}^{i}\}_{i=1}^n$ that satisfies $\sum_{i=1}^n{\hat{A}^{i}}=\hat{A}_{\bm{\pi}}$), where $\hat{A}_{\bm{\pi}}$ is an estimator of the joint advantage function. Here, we adopt the Generalized Advantage Estimation \citep{schulman2015gae} for $\hat{A}_{\bm{\pi}}$. The results reveal that there is no significant difference in performance between the average split and the random split.
\begin{figure}[h]
\centering
\includegraphics[width=0.8\linewidth]{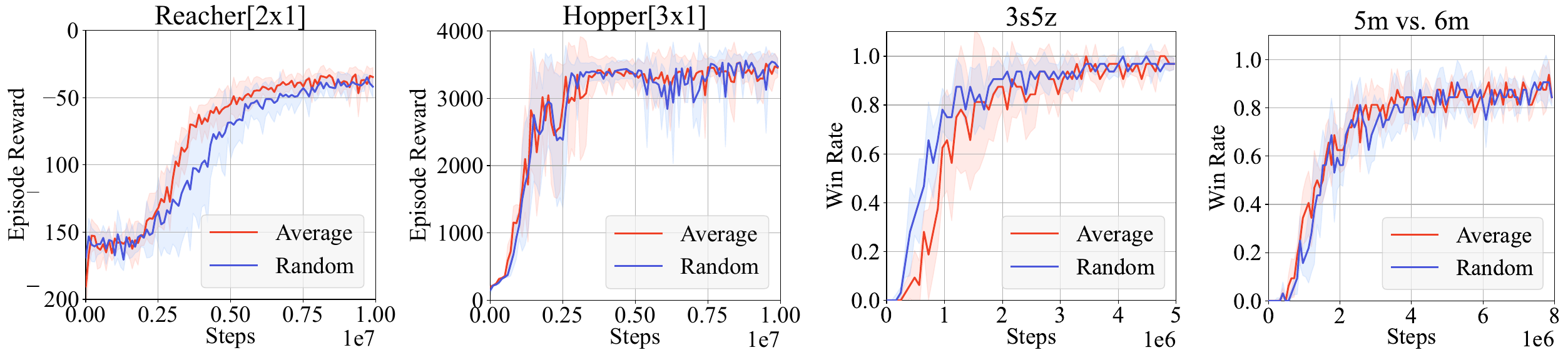}
%\framebox[4.0in]{$\;$}
\caption{Performance comparison of FP3O with FuPS: average split vs. random split.}
\label{fig_appendix_random_average}
\end{figure}

\subsection{Ablation Study on Dependent Step}
To demonstrate how the dependent step affects the performance and what the performance of the independent step alone looks like, we perform an ablation study on the dependent step. We compare the result of the independent step (InStep) against that of the independent step + dependent step (InStep+DeStep) on 3-Agent Hopper (heterogeneous-agent task) and 5m vs. 6m (homogeneous-agent task). As shown in Figure~\ref{fig_appendix_ablation}, it can be observed that InStep+DeStep outperforms InStep on all tasks. To answer the above results, we first revisit Equation~(\ref{equation_full_step1}). When replacing the KL divergence term with PPO-clip in Equation~(\ref{equation_full_step1}), it can be observed that the objective of the independent step will closely resemble that of MAPPO/IPPO, differing only in the reward part $A^{i_p}$. Thus, the independent step suffers from the same limitation as MAPPO/IPPO, i.e., ignores the impact of other agents' updates on an agent's objective, and is insufficient to guarantee the monotonic improvement of the performance.
\begin{figure}[H]
\centering
\includegraphics[width=1\linewidth]{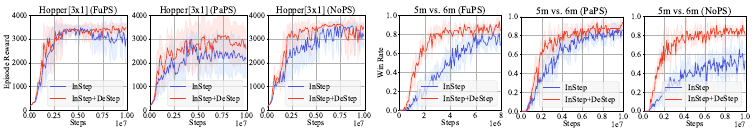}
%\framebox[4.0in]{$\;$}
\caption{Performance comparison between InStep and InStep+DeStep in different tasks.}
\label{fig_appendix_ablation}
\end{figure}

\subsection{Computational Analysis}
% We provide detailed wall-clock time and GPU memory usage in Tables~\ref{table_appendix_computational1} and \ref{table_appendix_computational2} respectively. It is worthing to notice that the wall-clock time is mainly comprised of two parts: 1) \emph{the rollout time} when agents interact with the environment to generate trajectories. For this part, FP3O, HAPPO, MAPPO and IPPO have the same rollout time. 2) \emph{the update time} when training the neural networks. The former accounts for the majority of the wall-clock time (e.g., In bane vs. bane of MAPPO, rollout time occupies approximately 85\% while update time occupies approximately 15\%).
In this part, we provide a comprehensive computational analysis through the wall-clock time and GPU memory usage, which are presented in Tables~\ref{table_appendix_computational1} and \ref{table_appendix_computational2} respectively. It is worth noting that the wall-clock time is mainly composed of two parts (1) \emph{the rollout time} and (2) \emph{the update time}. The former refers to the time agents spend interacting with the environment to generate trajectories. Importantly, FP3O, HAPPO, MAPPO, and IPPO have the same rollout time. On the other hand, the latter refers to the time taken to train neural networks, which accounts for a small portion of the total wall-clock time. For example, in bane vs. bane of MAPPO, rollout time occupies approximately 85\% while update time occupies approximately 15\%.
\begin{table}[H]
\centering
\begin{tabular}{c | c | c c c c}
   \toprule
   Network & Tasks & FP3O & HAPPO & MAPPO & IPPO \\ \midrule
   FuPS & bane vs. bane & 710 & 764 & 697 & 692 \\ 
   FuPS & 6-AgentHalfCheetah & 1108 & 1517 & 1032 & 1022 \\ \midrule
   PaPS & bane vs. bane & 754 & 794 & 714 & 714 \\ 
   PaPS & 6-Agent HalfCheetah & 1722 & 1825 & 1594 & 1547 \\ \midrule
   NoPS & bane vs. bane & 818 & 813 & 769 & 766 \\
   NoPS & 6-AgentHalfCheetah & 2041 & 2037 & 1750 & 1743 \\ 
   \bottomrule
\end{tabular}
\caption{The wall-clock time (min). Wall-clock time is associated with serial computing overhead. FuPS, PaPS, and NoPS denote full, partial, and non-parameter sharing respectively.}
\label{table_appendix_computational1}
\end{table}
\begin{table}[H]
\centering
\begin{tabular}{c | c | c c c c}
   \toprule
   Network & Tasks & FP3O & HAPPO & MAPPO & IPPO \\ \midrule
   FuPS & bane vs. bane & 3679 & 1009 & 3687 & 3241 \\ 
   FuPS & 6-AgentHalfCheetah & 969 & 879 & 967& 967 \\ \midrule
   PaPS & bane vs. bane & 3013 & 1001 & 3011 & 2901 \\ 
   PaPS & 6-Agent HalfCheetah & 948 & 879 & 947 & 945 \\ \midrule
   NoPS & bane vs. bane & 1082 & 1081 & 1081 & 1059 \\
   NoPS & 6-AgentHalfCheetah & 879&879&879&877   \\ 
   \bottomrule
\end{tabular}
\caption{The GPU memory usage (MiB). GPU memory usage is associated with parallel computing overhead. FuPS, PaPS, and NoPS denote full, partial, and non-parameter sharing respectively.}
\label{table_appendix_computational2}
\end{table}

Regarding serial computing (wall-clock time), we can draw the following conclusion from Table~\ref{table_appendix_computational1}.
Full parameter sharing: HAPPO $>$ FP3O $>$ MAPPO $\approx$ IPPO; 
Partial parameter sharing: HAPPO $>$ FP3O $>$ MAPPO $\approx$ IPPO; 
Non-parameter sharing: HAPPO $\approx$ FP3O $>$ MAPPO $\approx$ IPPO.

Regarding parallel computing (GPU memory usage), we can draw the following conclusion from Table~\ref{table_appendix_computational2}.
Full parameter sharing: FP3O $\approx$ MAPPO $>$ IPPO $>$ HAPPO;
Partial parameter sharing: FP3O $\approx$ MAPPO $>$ IPPO $>$ HAPPO;
Non-parameter sharing: FP3O $\approx$ MAPPO $\approx$ HAPPO $>$ IPPO.

In summary, the parallel computing overheads of FP3O demonstrate the same trend as MAPPO due to the parallel manner. However, FP3O has higher serial computing overheads because of the necessity to calculate intermediate policies and the condition. On the other hand, HAPPO has higher serial consumption but lower parallel consumption in the parameter-sharing cases as opposed to FP3O, MAPPO, and IPPO. In the NoPS case, all algorithms exhibit similar GPU memory usage, and FP3O and HAPPO exhibit similar wall-clock times. This is because our codes for all algorithms are forked and developed from MAPPO repository\footnote{\url{https://github.com/marlbenchmark/on-policy}} and HAPPO repository\footnote{\url{https://github.com/cyanrain7/TRPO-in-MARL}}, inheriting the original serial implementation (using a FOR loop) for non-parameter sharing case. However, similar to MAPPO, FP3O can be further implemented in parallel to enhance the performance on running time in the NoPS case, for example, through multi-threaded training.

\subsection{Analysis of Approximation}
FP3O employs two approximations: 1. clipping, and 2. intermediate policy approximation. The clip (PPO-based) or KL constraint (TRPO-based) is \emph{indispensable} for approximating $D_{KL}^{max}$ term. The necessity of this arises from practical considerations as explained in the TRPO paper~\cite{schulman2015trpo}. The intermediate policy is an \emph{optional} approximation, selected to improve computational efficiency. In Figure~\ref{fig_matching_degree}, we validated this approximation and demonstrated its practical significance.

\begin{table}[H]
\centering
\begin{tabular}{c | c c}
   \toprule
   Approximation & Clipping & Clipping+Intermediate Policy  \\ \midrule
   2-Agent Reacher [2$\times$1] (NoPS) & $-25.9_{\pm5.4}$ & $-26.5_{\pm2.6}$ \\ 
   3-Agent Hopper [3$\times$1] (NoPS) & $3623.4_{\pm196.8}$ & $3560.2_{\pm107.7}$ \\ 
   3s5z (FuPS) & $100_{\pm4.6}$ & $100.0_{\pm3.7}$ \\ 
   6h vs. 8z (FuPS) & $90.6_{\pm6.4}$ & $93.8_{\pm6.1}$ \\ 
   \bottomrule
\end{tabular}
\caption{Comparative performance evaluation of the clipping alone and the clipping+intermediate policy (FP3O) approximation on different tasks.}
\label{table_appendix_approximation_analysis}
\end{table}
To further scrutinize the latter approximation, we conducted experiments using clipping alone, which is a one-step approximation similar to what PPO/TRPO has done. The results are provided in Table~\ref{table_appendix_approximation_analysis}.
As observed, FP3O performs comparably, consistent with the analysis in Figure~\ref{fig_matching_degree}. In addition, FP3O reduces the running time by approximately 20\%. 

\subsection{Stability Test}\label{appendix_stability_test}
The stability of each algorithm is examined through performance fluctuation experiments. Under identical environmental conditions and parameter settings, each algorithm is executed multiple times, with the standard deviations of performance serving as the main focus. An algorithm is deemed unstable if it exhibits significant discrepancies in output across multiple runs. We conduct these experiments on the 3s5z (FuPS) task, with the results\footnote{The instability of CoPPO has been demonstrated and discussed in Appendix~\ref{appendix_CoPPO_Baseline}.} depicted in Figure~\ref{fig_appendix_stability_test}.

As we can see, HAPPO and A2PO display substantial deviations due to their incompatibility with full-parameter sharing. In contrast, MAPPO maintains more stable performance. FP3O experiences some fluctuations during the initial phases of training, but it stabilizes upon reaching convergence. The use of the double-clip technique can mitigate these early-stage fluctuations.

\begin{figure}[H]
\centering
\includegraphics[width=0.9\linewidth]{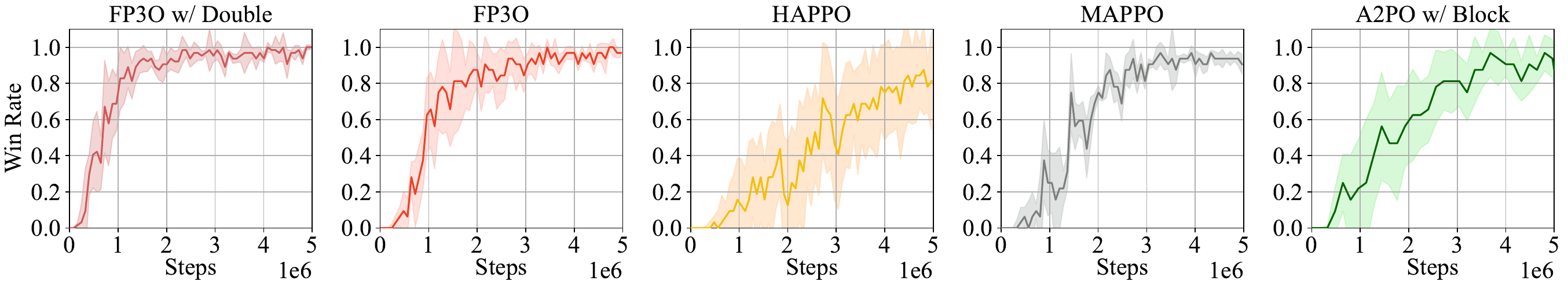}
%\framebox[4.0in]{$\;$}
\caption{Performance stability comparison across various algorithms on the 3s5z (FuPS) task. FP3O w/ Double denotes FP3O with double-clip objective, as shown in Equation~\ref{eq_double_clip}.}
\label{fig_appendix_stability_test}
\end{figure}

\newpage
\subsection{Extended Analysis of Tables~\ref{table_MAMuJoCo_results} and \ref{table_SMAC_results}}

\begin{figure}[H] 
\begin{minipage}{0.48\textwidth} 
\vspace{0.3cm}
  \centering 
  \includegraphics[width=1\linewidth,height=5.8in]{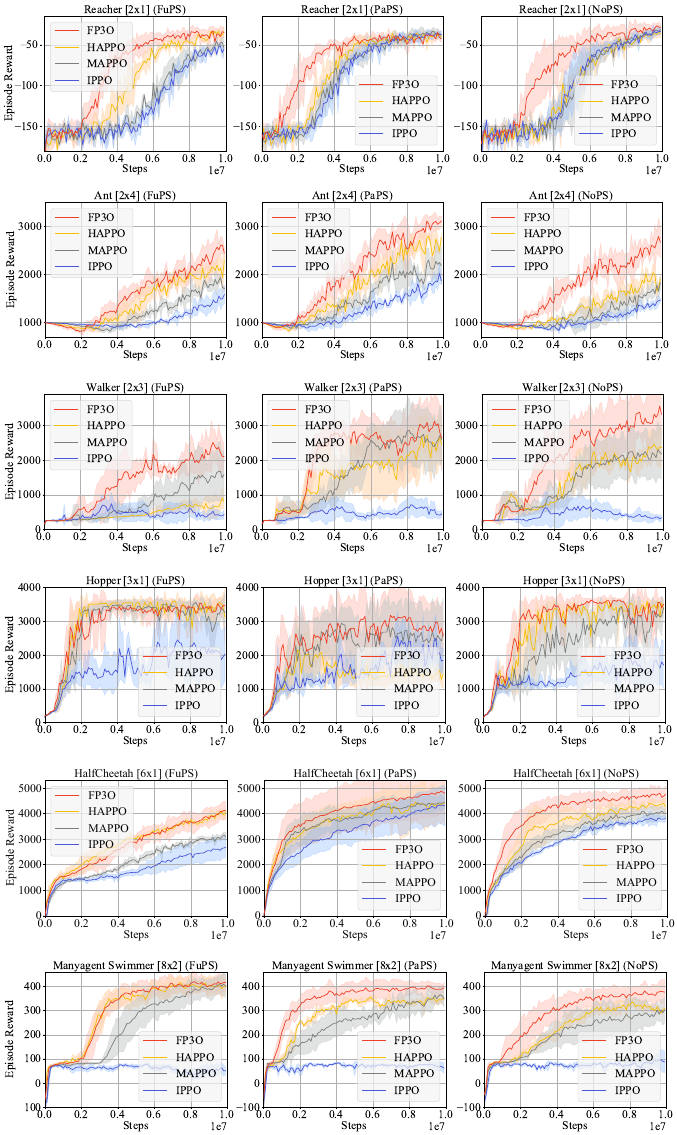} 
  \caption{Average evaluation reward in the MAMuJoCo domain. FuPS, PaPS, and NoPS denote full parameter sharing, partial parameter sharing, and non-parameter sharing respectively.} 
  \label{fig_result_mujoco} 
\end{minipage} 
\hspace{0.3cm}
\begin{minipage}{0.48\textwidth} 
  \centering 
  \includegraphics[width=1\linewidth,height=5.8in]{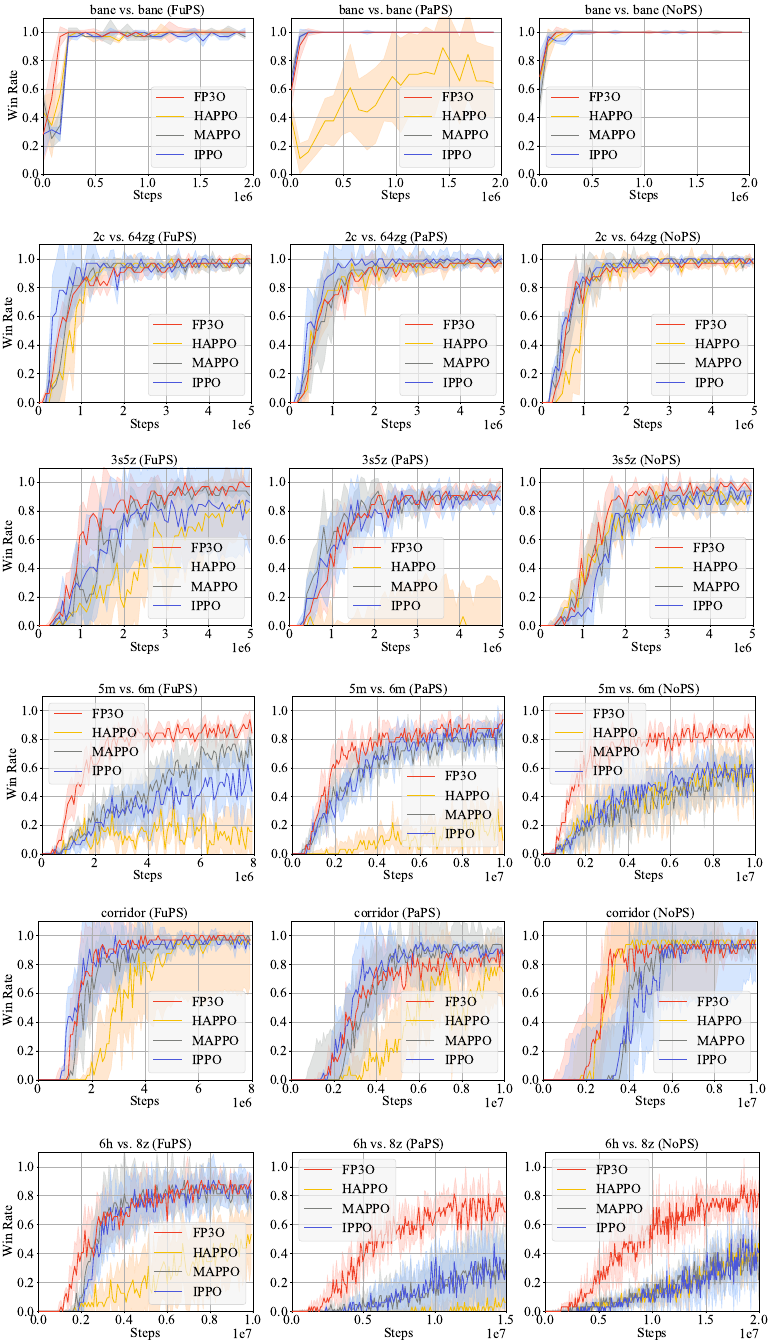} 
  \caption{Median evaluation win rate in the SMAC domain. FuPS, PaPS, and NoPS denote full parameter sharing, partial parameter sharing, and non-parameter sharing respectively.} 
  \label{fig_result_smac} 
\end{minipage}
\end{figure}

We present the training process in the MAMuJoCo and SMAC in Figures~\ref{fig_result_mujoco} and \ref{fig_result_smac}. In the MAMuJoCo domain, we observe that FP3O does not achieve the best performances in certain tasks, such as 3-Agent Hopper (FuPS), 2-Agent Reacher (PaPS), and 3-Agent Hopper (NoPS). This discrepancy could potentially be attributed to the randomness of 5 seeds in our repetitive experiments, as we can see from Figure~\ref{fig_result_mujoco} that there is no significant difference between the result of FP3O and the best performance in these tasks.

In the SMAC domain, we observe that FP3O exhibits weaker performance in the super hard corridor (PaPS) task, as compared to MAPPO. This could be due to the great sensitivity of the corridor map to hyperparameter, as thoroughly investigated and demonstrated in Figure 17-24 of \cite{yu2021mappo}. In our experiment, the hyperparameters in SMAC are taken from the original hyperparameters (MAPPO-friendly hyperparameters) of MAPPO paper without further fine-tuning. This could have led to the results obtained by FP3O and MAPPO in the corridor task (PaPS).

As demonstrated in Tables~\ref{table_MAMuJoCo_results}, \ref{table_SMAC_results} and Figures~\ref{fig_result_mujoco}, \ref{fig_result_smac}, it is interesting to observe that partial parameter sharing seems to have a more detrimental effect on HAPPO than full parameter sharing. To explain this phenomenon, we suppose that agents have a shared backbone with parameter vector $\theta^B$, and each agent $i$ has an individual last layer $\theta^i$. In the sequential update scheme of HAPPO, agents $1,...,n$ are updated sequentially as follows:
\begin{align}
\text{Agent $1$ turn: $\theta^B$}&\text{ and $\theta^1$ are changed.} \nonumber\\
\text{Agent $2$ turn: $\theta^B$}&\text{ and $\theta^2$ are changed.}\nonumber\\
&\vdots \nonumber \\
\text{Agent $n$ turn: $\theta^B$}&\text{ and $\theta^n$ are changed.}\nonumber
\end{align}
Let us focus on agent $1$. $\theta^1$ is \emph{fixed} after the turn of agent $1$, while $\theta^B$ will be \emph{continuously changed} during the turns of agents $2,...,n$.
By the end of the sequential update process, the final $\theta^B$ deviates significantly from the beginning. As a result, it can no longer cooperate efficiently with $\theta^1$, causing the output action of agent $1$ chaotic. All agents of HAPPO suffer from the above problem in partial parameter sharing. However, it does not arise in the full parameter-sharing case as the parameters of the shared backbone and shared last layer are bundled together and trained jointly.
\end{document}